\documentclass[11pt]{article}
\usepackage{fullpage}

\usepackage{titlesec}
\titleformat{\subsubsection}[runin]{\normalfont\bfseries}{\thesubsubsection}{0.7em}{}[.]

\usepackage{graphicx}
\usepackage{subfigure}
\usepackage{booktabs} 
\usepackage{fancyhdr}

\usepackage{amssymb,amsfonts,amsmath,amssymb,mathtools,mathrsfs,sansmath}
\usepackage[table]{xcolor}

\graphicspath{{./figs/}}

\usepackage{tikz}
\usetikzlibrary{calc}
\usetikzlibrary{decorations.pathreplacing}

\newcommand{\norm}[1]{\|#1\|}
\newcommand{\nor}{\norm{\cdot}}

\newcommand{\abs}[1]{|#1|}
\newcommand{\op}[1]{\operatorname{#1}}

\newcommand{\one}{\boldsymbol{1}}
\newcommand{\zero}{\boldsymbol{0}}

\newcommand{\argmin}{\mathop{\op{argmin}}}
\newcommand{\Argmin}{\mathop{\op{Argmin}}}

\newcommand{\Argmax}{\mathop{\op{Argmax}}}

\newcommand{\proj}{\Pi}
\newcommand{\prox}{\op{prox}}
\newcommand{\conv}{\op{conv}}
\newcommand{\cone}{\op{cone}}
\newcommand{\ext}{\op{ext}}

\newcommand{\diag}{\op{diag}}
\newcommand{\sign}{\op{sign}}
\newcommand{\range}{\op{range}}
\newcommand{\dist}{\op{dist}}
\newcommand{\supp}{\op{Supp}}

\newcommand{\val}{\op{val}}
\newcommand{\card}{\op{card}}

\newcommand{\<}{\langle}
\renewcommand{\>}{\rangle}

\renewcommand{\int}{\mathcal{I}}
\newcommand{\ptn}{\texttt{P}}
\newcommand{\ptnn}{\bar{\texttt{P}}}
\newcommand{\BD}{\texttt{BD}}
\newcommand{\BDbar}{\overline{\texttt{BD}}}
\newcommand{\T}{\texttt{T}}
\newcommand{\simplex}{\mathbf{\Delta}}

\newcommand{\Mmat}{M}
\newcommand{\Mset}{\mathcal{M}}

\newcommand{\reals}{\mathbb{R}}
\newcommand{\hbeta}{\widehat{\beta}}
\newcommand{\sT}{{\sf T}}
\newcommand{\hSigma}{\widehat{\Sigma}}
\newcommand{\Tr}{{\sf Tr}}

\newcommand{\Sc}{\mathcal{S}}
\newcommand{\Bc}{\mathcal{B}}
\newcommand{\Ac}{\mathcal{A}}
\newcommand{\cG}{\mathcal{G}}
\newcommand{\cB}{\mathcal{B}}

\newcommand{\cC}{\mathcal{C}}
\newcommand{\cF}{\mathcal{F}}
\newcommand{\x}{\beta}
\newcommand{\y}{\theta}
\newcommand{\z}{z}
\newcommand{\g}{g}

\newcommand{\owl}{w}
\newcommand{\ksp}{{k-\op{sp}}}
\newcommand{\errL}{{\sf err}^{\rm Lasso}}
\newcommand{\errDS}{{\sf err}^{\rm DS}}
\def\errkk{{\sf err}^{k\square k}}
\def\errk1{{\sf err}^{k\square 1}}
\newcommand{\agg}{\Lambda}
\newcommand{\pran}{\beta;\nor}
\newcommand{\prans}{\beta^\star;\nor}

\definecolor{ygcolor}{RGB}{0,0,0}
\definecolor{darkblue}{RGB}{0,0,200}
\newcommand{\ErrSet}{{\color{ygcolor}\Xi}}

\newcommand{\vpsi}{{\color{ygcolor}{\psi}}}
\newcommand{\REc}{{\color{ygcolor}{\alpha}}}
\newcommand{\rat}{{\color{ygcolor}{\gamma}}}

\newcommand{\qc}{{\color{ygcolor}q}}
\newcommand{\ql}{{\color{ygcolor}\widetilde{q}}}

\usepackage{scalerel}
\def\sq{{\mathord{\scalerel*{\square}{o}}}} 

\usepackage{amsthm,thmtools}
\theoremstyle{plain}
\newtheorem{proposition}{Proposition}[section]
\newtheorem{theorem}[proposition]{Theorem}
\newtheorem{lemma}[proposition]{Lemma}
\newtheorem{corollary}[proposition]{Corollary} 
\newtheorem{definition}[proposition]{Definition}
\newtheorem{remark}[proposition]{Remark}

\newtheorem{example}[proposition]{Example}

\usepackage{hyperref}
\hypersetup{
    colorlinks=true,
    linkcolor=blue!50!black,
    filecolor=magenta,      
    urlcolor=cyan,
    citecolor=green!50!black,
}

\numberwithin{equation}{section}

\title{New Computational and Statistical Aspects of Regularized Regression with Application to Rare Feature Selection and Aggregation}

\author{Amin Jalali\thanks{Technicolor AI Lab.
Email: \texttt{amin.jalali@technicolor.com}.} \and 
Adel Javanmard\thanks{Marshall School of Business, University of Southern California. Email: \texttt{ajavanma@usc.edu}.} \and 
Maryam Fazel\thanks{Department of Electrical Engineering, University of Washington. Email: \texttt{mfazel@uw.edu}.}}

\date{April 2019}

\usepackage{natbib}

\begin{document}
\maketitle

\begin{abstract}
Prior knowledge on properties of a target model often come as discrete or combinatorial descriptions. This work provides a {\em unified computational framework} for defining norms that promote such structures. More specifically, we develop associated tools for optimization involving such norms given only the orthogonal projection oracle onto the non-convex set of desired models. As an example, we study a norm, which we term the doubly-sparse norm, for promoting vectors with few nonzero entries taking only a few distinct values. We further discuss how the K-means algorithm can serve as the underlying projection oracle in this case and how it can be efficiently represented as a quadratically constrained quadratic program. Our motivation for the study of this norm is regularized regression in the presence of rare features which poses a challenge to various methods within high-dimensional statistics, and in machine learning in general. The proposed estimation procedure is designed to perform automatic feature selection and aggregation for which we develop statistical bounds. The bounds are general and offer a statistical framework for norm-based regularization. The bounds rely on novel geometric quantities on which we attempt to elaborate as well. 
\end{abstract}


\tableofcontents
\section{Introduction}
A large portion of estimation procedures in high-dimensional statistics and machine learning have been designed based on principles and methods in continuous optimization. In this pursuit, incorporating prior knowledge on the target model, often presented as discrete and combinatorial descriptions, has been of interest in the past decade. Aside from many individual cases that have been studied in the literature, a number of general frameworks have been proposed. For example, \citep{bach2013learning,obozinski2016unified} define sparsity-related norms and their associated optimization tools from support-based monotone set functions. On the other hand, several unifications have been proposed for the purpose of providing estimation and recovery guarantees. A well-known example is the work of \citep{chandrasekaran2012convex} which connects the success of norm minimization in model recovery given random linear measurements to the notion of Gaussian width \citep{Gordon88}. 
However, many of the final results of these frameworks (excluding discrete approaches such as \citep{bach2013learning}) are quantities that are hard to compute; even evaluating the norm. 
Therefore, many a time computational aspects of these norms and their associated quantities are treated on a case by case basis. In fact, a {\em unified} framework for turning discrete descriptions into continuous tools for estimation, that 1) provides a {\em computational} suite of optimization tools, and 2) is amenable to {\em statistical analysis}, is largely underdeveloped. 

Consider a measurement model $y = X\beta^\star + \epsilon$, where $X\in\mathbb{R}^{n\times p}$ is the {\em design} matrix and $\epsilon\in\mathbb{R}^n$ is the {\em noise} vector. Given combinatorial descriptions of the underlying model, say $\beta^\star\in\Sc \subset\mathbb{R}^p$, in addition to $X$ and $y$, much effort and attention has been dedicated to understanding {\em constrained estimators} for recovery. For example, only assuming access to the (non-convex) projection onto the set of desired models~$\Sc$ enables devising a certain class of recovery algorithms constrained to~$\Sc$; Iterative Hard Thresholding (IHT) algorithms, \citep[Section 3]{blumensath2008iterative} \citep{blumensath2011sampling} (projects onto the set of $k$-sparse vectors), \citep[Section 2]{jain2010guaranteed} (projects onto the set of rank-$r$ matrices), \citep{roulet2017iterative} (does 1-dimensional K-means which is projection onto the set of models with $K$ distinct values), belong to this class. However, a major subset of estimation procedures focus on norms, designed based on the non-convex structure sets, for estimation. Working with convex functions, such as norms, for promoting a structure is a prominent approach due to its flexibility and robustness. Namely, the proposed norms can be used along with different loss functions and constraints\footnote{This is in contrast to the specific constrained loss minimization setups required in IHT.}. In addition, the continuity property of these functions allows the optimization problems to take into account points that are {\em near} (but not necessarily inside) the structure set; a {\em soft} approach to specifying the model class. The seminal work of \citep{chandrasekaran2012convex} provides guarantees for norm minimization estimation, constrained with $X\beta=y$ or $\norm{X\beta-y}_2\leq \delta$, using the notion of Gaussian width. Dantzig selector is another popular category of constrained estimators studied in the literature (e.g., \citep{chatterjee2014generalized}) but other variations also exist (\autoref{sec:varphi-insight} provides a list). In analyzing all of these constrained estimators, {\em the tangent cone}, at the target model with respect to the norm ball, is the determining factor for recoverability. Then, the notion of Gaussian width of such cone \citep{chandrasekaran2012convex,Gordon88} allows for establishing high probability bounds for recovery from many random ensembles of design. In a way, the Gaussian width, or a related quantity known as the statistical dimension \citep{amelunxen2014living}, are local quantities that can be understood as an operational method for {\em gauging the model complexity with respect to the norm} and determining the minimal acquisition requirements for recovery from random linear measurements.

However, regularized estimators pose further challenges for analysis. More specifically, consider
\begin{align}\label{eq:estimator}
\hbeta ~\equiv~ \argmin_\beta ~~ \frac{1}{2n}\|y - X\beta\|_2^2 + \lambda \norm{\beta}
\end{align}
where $\lambda$ is the regularization parameter. 
From an optimization theory perspective, for a fixed design and noise, \eqref{eq:estimator} and a norm minimization problem constrained with $\norm{X\beta-y}\leq \delta$ (see \eqref{eq:estimator-constrained} and \eqref{eq:estimator-tube}) are equivalent if a certain value of $\delta$, corresponding to $\lambda$, is being used; meaning that $\hbeta$ for these estimators will be equal. However, the mapping between theses problem parameters is in general complicated (e.g., see \citep{aravkin2016level}) which renders the aforementioned equivalence useless when studying error bounds that are expressed in terms of these problem parameters (e.g., see bounds in \autoref{thm:estimation} and their dependence on $\lambda$). Furthermore, in the study of {\em expected} error bounds for a family of noise vectors (or design matrices), such equivalence is in general irrelevant (e.g., fixing $\lambda$, each realization of noise will imply a different $\delta$ corresponding to the given value of $\lambda$). 
Nonetheless, a good understanding of regularized estimators with {\em decomposable norms} have been developed; see \citep{negahban2012unified,candes2013simple,foygel2014corrupted,wainwright2014structured,vaiter2015model} for slightly different definitions. These are norms with a special geometric structure and only a handful of examples are known (including the $\ell_1$ norm and the nuclear norm). 
In regularization with general norms, it is possible to provide a high-level analysis, inspired by the analysis for decomposable norms, and provide error bounds; e.g., see \citep{banerjee2014estimation} and follow up works. However, the proposed bounds are in a way {\em conceptual} and no general computational guidelines for evaluating these bounds exist. 
In this work, we introduce a geometric quantity for gauging model complexity with respect to a norm in {\em regularized estimation.} Such quantity, accompanied by a few computational guidelines and connections to the rich literature on convex geometry, then allows for principled approach towards evaluating the previous conceptual error bounds leading to our final statistical characterizations for \eqref{eq:estimator} that are sensitive to 1) norm-induced properties of design, and 2) non-local properties of the model with respect to the norm. 

A motivation behind our pursuit of a computational and statistical framework for regularization is to handle the {\em presence of many rare features} in real datasets, which has been a challenging proving ground for various methods within high-dimensional statistics, and in machine learning in general; see \autoref{sec:intro} for further motivation. In this work, we study an efficient estimator, namely a regularized least-squares problem, for {\em automatic feature selection and aggregation} and develop statistical bounds. The regularization, an atomic norm proposed by \citep{jalali2013convex}, poses new challenges for computation (even norm evaluation) and statistical analysis (e.g., non-decomposability). We extend the computational framework provided in \citep{jalali2013convex} for this norm, in \autoref{sec:the-norm}, and provide statistical error bounds in \autoref{sec:kd-final}. We also establish advantages over Lasso (\autoref{sec:examples-kd}). Moreover, our estimation and prediction error bounds, rely on simple geometric notions to gauge condition numbers and model complexity with respect to the norm. These bounds are quite general and go beyond regularization for feature selection and aggregation.

\subsection{Summary of Technical Contributions}
In this work, we consider regularized regression in the presence of rare features (presented in \autoref{sec:intro}) as our main case study. In our attempt to address this problem, we develop several general results for defining norms from given combinatorial descriptions and for statistical analysis of norm-regularized least-squares, as summarized in the following: 
\begin{enumerate}
\item We adopt an approach to defining norms from given descriptions of desired models, and provide a unified machinery to derive useful quantities associated to a norm for optimization (e.g., the dual norm, the subdifferential of the norm and its dual, the proximal mapping, projection onto the dual norm ball, etc); see \autoref{sec:structure-norms}. Our approach relies on {\em the non-convex orthogonal projection onto the set of desired models.} In \autoref{sec:the-norm}, we discuss how a discrete algorithm such as K-means clustering can be used to define a norm, namely the doubly-sparse norm, for promoting vectors with few nonzero entries taking only a few distinct values. Our results extend those of \citep{jalali2013convex} to any structure.
 
\end{enumerate}
Complementing the existing statistical analysis approaches, for least-squares regularized with any norm, we take a variational approach, through quadratic functions, to understanding norms and provide alternative error bounds that can be easier to interpret, compute, and generalize: 
\begin{enumerate}\setcounter{enumi}{1}

\item 
We provide a prediction error bound 
in terms of {\em norm-induced aggregation measures} of the design matrix for when the noise satisfies the convex concentration property or is a subgaussian random vector. We do this by making a novel use of the Hanson-Wright inequality for when the dual to the norm has a concise variational representation; \autoref{sec:concise-var}. The new bounds are deterministic with respect to the design matrix, are interpretable, and allow for taking detailed information on the design matrix into account, going beyond results on well-known random ensembles which might be unrealistic in real applications. 
\end{enumerate}
Most of the existing estimation bounds for norm-regularized regression can be unified under the notion of {\em decomposability}; see \citep{negahban2012unified,candes2013simple,foygel2014corrupted,vaiter2015model} for slightly different definitions. Our results, in contrast, do not rely on such assumption:

\begin{enumerate}\setcounter{enumi}{2}
\item In gauging model complexity with respect to the regularizer, we introduce a novel geometric measure, termed as {\em the relative diameter}, which then allows for simplified derivations for restricted eigenvalue constants and prediction error bounds. More specifically, we go beyond decomposability {\em and} we provide techniques to compute such complexity measure (\autoref{sec:varphi}). We provide calculations for a variety of norms (e.g., ordered weighted $\ell_1$ norms) used in the high-dimensional statistics literature; \autoref{sec:varphi} and \autoref{app:varphi}. In \autoref{sec:varphi-insight}, we provide further insight into the notion of relative diameter and compare with existing quantities in the literature. Through illustrative examples, we showcase the sensitivity of the relative diameter to the properties of the model and the norm. 
\end{enumerate}
Finally, we use the aforementioned developments to design and analyze a regularized least-squares estimator for regression in the presence of rare features: 
\begin{enumerate}\setcounter{enumi}{3}
\item We propose to use doubly-sparse regularization for regression in the presence of rare features (\autoref{sec:reg-est}). We discuss how such choice allows for {\em automatic feature aggregation}. 
We use the insights and tools we develop in the paper for regression with rare features and establish the advantage of regularizing the least-squares regression with the doubly-sparse norm, given in \eqref{eq:estimator-kd}, over Lasso, in \autoref{sec:kd-final}. 
\end{enumerate}
Last but not least, we provide various characterizations related to a number of norms common in the high-dimensional statistics literature such as the ordered weighted $\ell_1$ norms (commonly used for simultaneous feature selection and aggregation; e.g., see \citep{figueiredo2016ordered}.) which could be of independent interest. See \autoref{sec:varphi-owl} and \autoref{app:varphi}. Proof of technical lemmas are deferred to Appendices.

\paragraph{Notations.}	
Denote by $\norm{A}$ and $\norm{A}_F$ the operator norm and the Frobenius norm of a matrix $A$. We also represent its smallest and largest singular values by $\sigma_{\min}(A)$ and $\sigma_{\max}(A)$. 
For a positive integer $p$, we denote by $[p]$ the set $\{1,2,\ldots, p\}$.
For a compact set $\mathcal{M}\subset \mathbb{R}^p$, the polar set is denoted by $\mathcal{M}^\star = \{x:~ \langle x,y\rangle \leq 1, ~ \forall y\in \mathcal{M} \}$. For a positive integer $p$, we denote by $\mathbb{S}^{p-1}$ the $(p-1)$-dimensional unit sphere, $\mathbb{S}^{p-1} \equiv \{ x\in \reals^p:\, \|x\|_2 = 1\}$. Given a set $\mathcal{M} \subset \reals^p$, we denote by $\conv(\mathcal{M})$ the convex hull of $\mathcal{M}$, i.e., $\conv(\mathcal{M})\equiv \{\sum_{i=1}^k w_i x_i:\, \sum_{i=1}^k w_i = 1,\, w_i\ge0,\, x_i\in \mathcal{M},\,k\in\mathbb{N}\}$. Moreover, define $\cone (\mathcal{M}) \equiv \{\alpha a:~ \alpha \in\mathbb{R}_+,~a\in \mathcal{M} \}$. In addition, given a compact set $\mathcal{M}\subset \reals^p$, a point $a\in \mathcal{M}$ is an {\em extreme point} of $\mathcal{M}$ if $a=(b+c)/2$ for $b,c\in \mathcal{M}$ implies $a=b=c$. Denote by $\one_p$ and $\zero_p$ the vectors of all ones and all zeros in~$\mathbb{R}^p$, respectively. We may drop the subscripts when clear from the context. 
For two vectors $\beta,\theta\in\mathbb{R}^p$, their Hadamard (entry-wise) product is denoted by $\beta \circ \theta$ where $(\beta \circ \theta)_i = \beta_i\theta_i$ for $i\in[p]$. 
The unit simplex in $\mathbb{R}^p$ is denoted by $\simplex_p\equiv \{u\in\mathbb{R}^p:~ u\geq \zero_p,~ \one^\sT u = 1\}$. The full unit simplex is denoted by $\widetilde\simplex_p\equiv \{u\in\mathbb{R}^p:~ u\geq \zero_p,~ \one^\sT u \leq 1\}$. 
In all of this work, we assume the model ($\beta$ or $\beta^\star$) is nonzero.

\section{Motivation: Regularized Regression for Rare Features}\label{sec:intro}
Data sparsity has been a challenging proving ground for various methods. Sparse sensing matrices in the established field of compressive sensing \citep{berinde2008combining}, the inherent sparsity of document-term matrices in text data analysis \citep{wang2010latent}, the ubiquitous sparsity of biological data, from gut microbiota to gene sequencing data, and the sparse interaction matrices in recommendation systems, have been challenging the established methods that otherwise have provable guarantees when certain well-conditioning properties (e.g., the restricted isometry property in compressive sensing) hold. See \citep{yan2018rare} for further motivations.

A common approach when lots of rare features are present is to remove the very rare features in a pre-processing step (e.g., Treelets by \citep{lee2008}). This is not efficient as it may discard large amount of information and better approaches are needed to make use of the rare features to boost estimation and predictive power. 
On the other hand, there have been efforts for establishing success of $\ell_1$ minimization in case of {\em certain sparse sensing matrices} (e.g., see \citep{berinde2008combining,berinde2008sparse, gilbert2010sparse}) where gaps between their statistical requirements and information-theoretical limits exist. 
Combinatorial approaches for subset selection, through integer programming, have also been restricted to certain sparse design matrices to achieve polynomial-time recovery \citep{del2018subset}. 
Instead, a variety of ad-hoc methods, based on solving different optimization programs, have been proposed for going beyond sparse models and making use of rare features \citep{bondell2008simultaneous, zeng2014ordered} and 
there has been a recent interest in this problem within the high-dimensional statistics community. 
While some of these estimators come with a statistical theory, they may require extensive prior knowledge \citep{li2018graph,yan2018rare} which could be expensive or difficult to gather in real applications.

\subsection{Our Approach: Doubly-Sparse Regularization}\label{sec:reg-est}
We approach this problem through feature aggregation, but unlike previous works, we do so in an automatic fashion at the same time as estimation. More specifically, in learning a linear model from noisy measurements, we use the model proposed by \citep{jalali2013convex}: we are interested in vectors that are not only sparse (to be able to ignore unnecessary features) but also have only a few distinct values, which induces a grouping among features and allows for automatic aggregation. We refer to this prior as {\em double-sparsity} and elaborate on it in the sequel. Considering the {\em structure norm} (see \citep{jalali2013convex}, \autoref{sec:Snorms-def}, or \autoref{sec:the-norm}) corresponding to this prior, we study a regularized least-squares optimization program in \eqref{eq:estimator-kd}. Since the existing machinery of atomic norms \citep{chandrasekaran2012convex} does not come with tools for optimization, we develop new tools in \autoref{sec:structure-norms} that can be used to efficiently compute and analyze the proposed estimator. Superior performance over the use of $\ell_1$ regularization (Lasso) in the presence of rare features is showcased in \autoref{sec:kd-final}. 

\subsubsection{The Prior and the Regularization}
A $k$-sparse vector $\beta\in\mathbb{R}^n$ can be expressed as a linear combination of $k$ indicator functions for singletons in $\{1,\ldots,n\}$; i.e., $\beta = \sum_{t=1}^k \beta_t \one(\{i_t\})$ where $\supp(\beta) = \{i_1,\ldots,i_k\}$. In contrast, we are interested in vectors that can be expressed as a linear combination of a few indicator functions using a coarse partitioning of $\{1,\ldots,n\}$; i.e., $\beta = \sum_{t=1}^d \beta_t \one(S_t)$ where $S_1,\ldots,S_d$ partition $\{1,\ldots,n\}$ and $d$ is small. Here, $\beta_t$'s can be zero; i.e., we are allowing $0$ to be one of the $d$ distinct values. 
To combine the two priors, for two fixed values $1\leq d \leq k \leq p$, one can consider vectors $\beta = \sum_{t=1}^d \beta_t \one(S_t)$ where $S_1,\ldots,S_d\subset\{1,\ldots,n\}$ are non-empty and disjoint and $\abs{S_1\cup\ldots\cup S_d}=k$. Those are the vectors with at most $k$ nonzero values where the top $k$ entries have at most $d$ distinct values. 
Finally, to make the prior more suitable for our regression setting, we allow for arbitrary sign patterns within each part. 

Given a vector $\beta$ denote by $\bar{\beta}$ the sorted version of $\abs{\beta}$ in descending order; i.e., $\bar{\beta}_1\geq \bar{\beta}_2\geq \cdots \geq \bar{\beta}_p \geq 0$. Then, we consider
\begin{equation}\label{eq:def-Skd}
\Sc_{k,d} \equiv \bigl\{\beta:~ \card(\beta) \leq k \,,~ \abs{\{\bar{\beta}_1,\ldots,\bar{\beta}_k\}} \leq d \bigr\}; 
\end{equation}
{\em the vectors with at most $k$ nonzero values whose top $k$ absolute values take at most $d$ distinct values}. \autoref{piece_const} illustrates an example. See \citep{jalali2013convex} for further detail and existing works around this idea. 
With the aid of the machinery presented in \autoref{sec:structure-norms}, we can define a norm, referred to as the \emph{doubly-sparse norm}, that can help in recovery of models from $\Sc_{k,d}$ in a sense characterized by our statistical error bounds. For two fixed values $1\leq d \leq k \leq p$, we refer to this norm as the $(k\square d)$-norm, denoted by $\norm{\cdot}_{k\square d}$.

\newcommand{\smalldot}{\circle{0.01}}
\newcommand{\meddot}{\circle{0.18}}
\newcommand{\bigdot}{\circle*{0.3}}
\begin{figure}[h]
\vskip 1in
\hskip 1.9in
\begin{picture}(0,0)
	\setlength{\unitlength}{.35cm}
	\linethickness{0.25mm}
	\put(1.5,1){\vector(1,0){20}}
	\put(19,0){sorted index $i$}
	\put(2,.5){\vector(0,1){6}}
	\put(.6,6){$\bar \beta_i$}
		
	\put(3,5){\bigdot}
	\put(4,5){\bigdot}
	\put(5,5){\bigdot}
	\put(6,5){\bigdot}
	\put(7,3){\bigdot}
	\put(8,3){\bigdot}
	\put(9,3){\bigdot}
	\put(10,2){\bigdot}
	\put(11,2){\bigdot}
	\put(12,2){\bigdot}
	\put(13,2){\bigdot}
	\put(14,2){\bigdot}
	
	\put(15,1){\bigdot}
	\put(16,1){\bigdot}
	\put(17,1){\bigdot}
	\put(18,1){\bigdot}
	\put(19,1){\bigdot}
	\put(20,1){\bigdot}
	
\end{picture}
	\caption{Sorted absolute values of a doubly-sparse vector $\beta\in\Sc_{12,3} \subset\mathbb{R}^{18}$.} 
	\label{piece_const}
\end{figure}
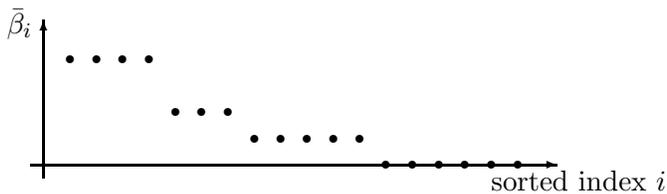

\subsubsection{A Statistical Analysis}
Consider a measurement model $y = X\beta^\star + \epsilon$, where $X\in\mathbb{R}^{n\times p}$ is the {\em design} matrix and $\epsilon\in\mathbb{R}^n$ is the {\em noise} vector. We then consider the following estimator,
\begin{align}\label{eq:estimator-kd}
\hbeta \equiv \argmin_\beta ~~ \frac{1}{2n}\|y - X\beta\|_2^2 + \lambda \norm{\beta}_{k\square d}\,,
\end{align}
where $\lambda$ is the regularization parameter. In \autoref{sec:kd-final}, we analyze \eqref{eq:estimator-kd} and provide {\em prediction error} bounds, namely bounds for $\norm{X(\beta^\star - \hat\beta)}_2$. 

More generally, we consider regularization with any norm. 
In providing a prediction error bound, we show how {\em norm-specific aggregation measures} can be used to bound the regularization parameter (\autoref{sec:concise-var}). For estimation error, we provide a general tight analysis through the introduction of {\em relative diameter} (\autoref{sec:new-bounds}, \autoref{sec:varphi}, and \autoref{sec:varphi-insight}). We make partial progress in computing the relative diameter, namely we do so for $\nor_{k\square 1}$, and its dual, but we also provide computations for some important classes of polyhedral norms to showcase possible strategies; for ordered weighted $\ell_1$ norms studied in \citep{figueiredo2016ordered}, and, for weighted $\ell_1$ and $\ell_\infty$ norms. 
See \autoref{sec:varphi-owl} and \autoref{app:varphi} for details of computations.

\subsubsection{Optimization Procedures}\label{sec:opt-algs-kd}
In computing $\hbeta$ from~\eqref{eq:estimator-kd}, or more generally \eqref{eq:estimator}, one can use different optimization algorithms. While $\nor_{k\square d}$ might seem complicated to even be evaluated, we show in \autoref{sec:the-norm} that there exists an efficient procedure for computing its proximal mapping (for a definition, see \autoref{def:prox}, and for a characterization in the case of $\nor_{k\square d}$, see \autoref{sec:qcqp}). Therefore, here, we only discuss two proximal-based optimization strategies to illustrate the computational efficiency of the estimator in \eqref{eq:estimator-kd}. The optimization program in~\eqref{eq:estimator-kd} is unconstrained and its objective is convex and the sum of a smooth and a non-smooth term. Therefore, as we have access to the proximal mapping associated to the non-smooth part, proximal gradient algorithm seems like a natural choice for optimization. For $t=1,\ldots, T$, we compute 
\begin{align}\label{eq:iterative}
g^t = \frac{1}{n}X^\sT X \beta^{t} - X^\sT y~,~~
\beta^{t+1} = \prox(\beta^t - \eta_t\, g^t; \nor) 
\end{align}
where $\eta_t$ is the step size. The algorithm, with an appropriate choice of step size, reaches an $\delta$-accurate solution (in prediction loss) in $O(1/\delta)$ steps. See \citep{parikh2014proximal} for further details on proximal algorithms.

As we will see later, the proximal mapping is the solution to a convex optimization program and may not admit a closed form representation unlike simple norms such as the $\ell_1$ norm (whose proximal mapping is soft-thresholding). Therefore, it might be inevitable to work with approximate solutions. In such case, {\em inexact proximal methods} \citep{schmidt2011convergence} may be employed which allow for a controlled inexactness in computation of the proximal mapping (more specifically, inexactness in the objective) but provide similar convergence rates as in the exact case.

Alternating Direction Method of Multipliers (ADMM) may also be used to solve \eqref{eq:estimator-kd}, similar to the discussions in Section 6.4 of \citep{boyd2011distributed} for the $\ell_1$ norm. The non-trivial ingredient of such strategy is the proximal mapping for the regularizer, which is available here. 

While this paper is concerned with the regularized estimation, it is worth mentioning that the ability to compute the proximal mapping also enables solving the generalized Dantzig selector (defined in \eqref{eq:estimator-dantzig}) as discussed in \citep{chatterjee2014generalized}. 
	
\section{Projection-based Norms}\label{sec:structure-norms}

Given a compact set $\Ac\subset \mathbb{R}^{p}$ of desired model parameters, which is symmetric, spans $\mathbb{R}^p$, and none of its members belongs to the convex hull of the others, the {\em atomic norm} framework \citep{bonsall1991general,chandrasekaran2012convex} defines a norm through 
\begin{align}	\label{eq:atomic_repr}
\norm{\x}_\Ac = \inf \bigl\{ \sum_{\omega\in\Ac} c_\omega  :\;  \x = \sum_{\omega\in\Ac} c_\omega\, \omega  \;,\; c_\omega \geq 0 \, \bigr\}.
\end{align}
This optimization problem is hard to solve in general and one might end up with linear programs that are difficult to solve or might have to resort to discretization (e.g., \citep{shah2012linear}) or to case-dependent reformulations (e.g., \citep{tang2013compressed}). 

Alternatively, one might consider the dual norm as the building block for further computations: the support function 
to the norm ball or to the atomic set, namely
\begin{align}\label{eq:dual-norm-lin}
\norm{\y}_\Ac^\star 
\equiv 
\sup_{ \norm{\x}_\Ac \leq 1} ~\langle \x,\y\rangle 
= \sup_{a\in \Ac} ~\langle a,\y\rangle .
\end{align} 
Assuming $\Ac\subseteq \mathbb{S}^{p-1}$, using the above variational characterization, and $\dist^2(\y,\Ac) 
\equiv \inf_{a\in \Ac} \|a-\y\|_2^2$, we get 
\begin{align}\label{eq:atomic-dual-dist}
\dist^2(\y,\Ac) 
= 1 + \norm{\y}_2^2 - 2 \norm{\y}_\Ac^\star.
\end{align}
While the dual norm is 1-homogeneous, the other terms above are not, which limits the uses of this expression. As evident from the result of \autoref{lem:dual-len-proj}, homogenizing the atomic set $\Ac$ into $\Sc \equiv \cone(\Ac) = \{\lambda a:~ \lambda\in\mathbb{R},\, a\in\Ac\}$ provides a better object to work with. Next, we elaborate on this direction and provide a framework for defining norms that comes with a computational suite for computing various quantities associated to these norms.

Some of the material in \autoref{sec:Snorms-def} and \autoref{sec:the-norm} have been previously mentioned in \citep{jalali2013convex} without proof and restricted to the so-called {\em $d$-valued models}. We generalize this framework and use it for addressing the problem of interest in \autoref{sec:intro}.
 
\subsection{Definition and Characterizations}\label{sec:Snorms-def}
Given a closed set $\Sc\subseteq \mathbb{R}^p$ that is scale-invariant (closed with respect to scaling by any $a\in\mathbb{R}$ which make it symmetric with respect to the origin as well) and spans $\mathbb{R}^p$, consider an associated convex set $\Bc_\Sc$ defined as
\begin{align}\label{eq:Snorm-gauge}
\Bc_\Sc = \conv \{ \beta: \; \beta\in\Sc \,,\; \norm{\beta}_2 = 1 \} \,.
\end{align}
Since $\Bc_\Sc$ is a symmetric compact convex body with the origin in its interior, the corresponding symmetric gauge function is defined as 
\begin{align} \label{eq:gauge}
	\norm{\beta}_\Sc \equiv \inf \{ \gamma>0 :~ \beta \in \gamma \Bc_\Sc \},
\end{align}
is a norm with $\Bc_\Sc$ as the unit norm ball. 
%
One can view $\norm{\cdot}_\Sc$ as an atomic norm with atoms given by the extreme points of the unit norm ball as $\Ac_\Sc = \ext(\Bc_\Sc)$. Using atoms, we can express $\norm{\cdot}_\Sc$ as in \eqref{eq:atomic_repr} with $\Ac = \Ac_\Sc$. 
As we will see later, $\x\in\Ac_\Sc$ if and only if $\norm{\x}_\Sc = \norm{\x}_2 = \norm{\x}_\Sc^\star$.

As an alternative to \eqref{eq:dual-norm-lin}, \autoref{lem:dual-len-proj} provides a way to compute the dual to this norm. 
Denote by \[\proj(\theta; \Sc) = \proj_\Sc(\theta)= \textstyle\argmin_\beta\{ \norm{\theta - \beta}_2:~ \beta \in \Sc \}\] the (non-convex) orthogonal projection onto $\Sc$. Note that the projection mapping onto a non-convex set is set-valued in general. We refer to \autoref{app:Snorm-summary} for further details and proofs of the following statements. 
\begin{figure}[htbp]
\begin{center}
\begin{tikzpicture}[scale=0.6]
\draw [gray, ->] (-4,0) -- (4,0);
\draw [gray, ->] (0,-3) -- (0,3);
\draw [thick] (-1,-1.5) rectangle (1,1.5);
\draw [gray, dashed, thick] (-1.3,-1.3*1.5) -- (1.3,1.3*1.5) node[above,black]{$\Sc$};
\draw [gray, dashed, thick] (-1.3,1.3*1.5) -- (1.3,-1.3*1.5);
\draw [gray, dotted, thick] (0,0) circle ({sqrt(13)/2});
\draw (0,13/6) -- (13/4,0) -- (0,-13/6) -- (-13/4,0) -- (0,13/6);
\end{tikzpicture}
\caption{The value of dual norm, $\nor_\Sc^\star$, is equal to the length of projection onto the structure set~$\Sc$. In the above schematic, $\Sc$ is the union of the two dashed lines. The norm ball is the convex hull of $\Sc\cap\Bc_2$ and is represented by the thick rectangle. The skewed diamond represents the dual norm ball. }
\label{fig:dual-norm-proj}
\end{center}
\end{figure}
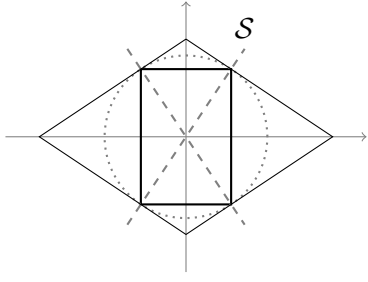

\begin{proposition}[The Dual Norm]\label{lem:dual-len-proj}
	Given any closed scale-invariant set $\Sc\subseteq \mathbb{R}^p$ which spans $\mathbb{R}^p$, the dual norm to $\norm{\cdot}_\Sc$ is given by 
	\begin{align}\label{eq:dual_norm_gen}
		\norm{\theta}_\Sc^\star 
		\equiv \sup_{\norm{\beta}_\Sc\leq 1} \langle\beta,\theta \rangle		
		= \norm{ \proj(\theta; \Sc) }_2
	\end{align}
	where $\norm{ \proj(\theta; \Sc) }_2$ refers to the $\ell_2$ norm of any member of the set and is well-defined. Moreover, 
\begin{equation}\label{eq:dual-pair}
\langle \theta , \proj(\theta; \Sc) \rangle = \norm{\proj(\theta;\Sc)}_2^2= \norm{\proj(\theta; \Sc)}_\Sc\, \norm{\theta}_\Sc^\star 
\end{equation}
which illustrates the pair of achieving vectors in the definition of dual norm and yields
\begin{align}\label{eq:Snorm-dual-dist}
(\norm{\theta}_\Sc^\star )^2
= \norm{\theta}_2^2 - \dist^2(\theta,\Sc). 
\end{align}
\end{proposition}
\autoref{fig:dual-norm-proj} illustrates \autoref{lem:dual-len-proj}. \autoref{eq:dual-pair} is also known as the {\em alignment property} in the literature. In contrast with \eqref{eq:atomic-dual-dist}, the expression in \eqref{eq:Snorm-dual-dist} is 2-homogeneous in $\theta$. 
With the above characterization for the dual norm we get 
\begin{align}\label{eq:Snorm-from-dual}
\norm{\beta}_\Sc = \sup \, \{\, \langle \beta , \theta \rangle :~ \norm{ \proj(\theta; \Sc) }_2 \leq 1 \, \}. 
\end{align}

Since the optimal $\beta$ in the definition of the dual norm in \eqref{eq:dual_norm_gen} is known to be $\proj_\Sc(\theta)$, we can easily characterize the subdifferential as in the following. 

\begin{lemma}[Subdifferential of dual norm] \label{subdiff_dualnorm_gen}
The subdifferential of the dual norm at $\x\neq 0$ is given by
\begin{equation*}
\partial\norm{\x}_\Sc^\star = \frac{1}{\norm{\proj_\Sc(\x)}_2} \conv\left(\proj_\Sc(\x)\right) 
\end{equation*}
which in turn implies $\partial (\tfrac{1}{2}{\norm{\x}_\Sc^\star}^2) = \conv \left(\proj_\Sc(\x) \right)$.
\end{lemma}
Proof of \autoref{subdiff_dualnorm_gen} is given in \autoref{app:projection}.

While an oracle that computes the projection enables us to carry out many computations for quantities related to the structure norm (e.g., the value of $\norm{\x}_\Sc\,$, the proximal operators for the norms and squared norms, 
as well as projection onto $\Bc_\Sc$, as discussed in the rest of this section), some properties of the structure set can highly simplify these computations. In the following, we consider the {\em invariance} properties of the structure (under permutations and sign changes) and in \autoref{lem:monot}, we discuss {\em monotonicity} properties of the structure. \autoref{lem:inv-proj} is not entirely new and has been discussed in the literature in one form or another. 

\begin{lemma}[Invariance in Projection]\label{lem:inv-proj}
Consider a closed set $A\subseteq\mathbb{R}^p$, convex or non-convex, and the orthogonal projection mapping $\proj(\cdot\,;A)$. Then, 
\begin{itemize}
\item Provided that $A$ is closed under a change of signs of entries (i.e., $\beta\in A$ implies $s \circ \beta\in A$ for any sign vector $s\in\{\pm1\}^p$) then $\theta \circ \beta \geq 0$ for any $\theta\in\proj(\beta; A)$. 

\item Provided that $A$ is closed under permutation of entries (i.e., $\beta\in A$ implies $\pi(\x)\in A$ for any permutation operator $\pi(\cdot)$) then $\beta$ and any $\theta\in\proj(\beta; A)$ have the same ordering: $\beta_i > \beta_j$ implies $\theta_i \geq \theta_j$ for all $i,j\in[p]$. 
\end{itemize}
\end{lemma}
Proof of \autoref{lem:inv-proj} is given in \autoref{app:projection}.

\subsection{Examples}
In the following, we provide a few examples of structure norms, both existing and new; 
\begin{itemize}
\item projection of $\beta$ onto $\Sc = \{\lambda e_i :\lambda \in\mathbb{R},\, i\in[p]\}$, where $e_i$ is the $i$-th standard basis vector, is the set of all $\|\beta\|_\infty e_{i^*}$ with $i^* \in \arg\max\{i\in[p]:\, \beta_i = \|\beta\|_\infty\}$. The length of such projections is indeed the $\ell_\infty$ norm which is dual to $\ell_1$ norm. 

\item When $\Sc$ is the set of all rank-1 matrices, projection onto $\Sc$ is the principal component and its length is the largest singular value of the matrix, the operator norm.

\item For structure norms defined based on $\Sc_{k,d}$, given in \eqref{eq:def-Skd}, see \autoref{sec:the-norm}. \autoref{fig:kdnorms} provides a schematic of this family of norms, for different values of $k$ and $d$, as well as their dual norms.
\item 
consider $w\in\mathbb{R}^p$ satisfying $w_1\geq w_2 \geq \cdots \geq w_p >0$ and $\Sc =\{\gamma Qw: \gamma\in\mathbb{R},~Q\in\mathcal{P}_\pm \}\subset \mathbb{R}^{p}$ where $\mathcal{P}_\pm $ is the set of signed permutation matrices. As established in \autoref{lem:extBst-owl}, we have
\begin{align*}
\norm{\beta}_\Sc \equiv \norm{w}_2 \cdot \norm{\beta}_\owl^\star
\end{align*}
where $\norm{\beta}_\owl \equiv \langle w,\bar\beta\rangle$ is the ordered weighted $\ell_1$ norm associated to $w$. Projection onto $\Sc$ requires sorting the absolute values of the input vector.

\item As another example, consider $\Sc =\{\gamma Q: \gamma\in\mathbb{R},~Q\in\mathcal{P}_\pm \}\subset \mathbb{R}^{p\times p}$ where $\mathcal{P}_\pm $ is the set of signed permutation matrices. 
Given a matrix $A$, its projection onto $\Sc$ can be derived by projecting $\abs{A}$ onto $\{\gamma P:~ \gamma\in\mathbb{R},~P\in\mathcal{P}\}$ where $\mathcal{P}$ is the set of permutation matrices. However, we already know efficient algorithms for finding the nearest permutation matrix (without a scaling factor $\gamma$); algorithms for solving the assignment problem such as the Hungarian method. \autoref{lem:proj_S_SB2} establishes that these two solutions are related.

\end{itemize}
\begin{lemma}\label{lem:proj_S_SB2}
We have $\cone(\proj_\Sc(\x)) = \cone(\proj_{\Sc\cap\mathbb{S}^{p-1}}(\x))$. In other words, one can project onto $\Sc\cap\mathbb{S}^{p-1}$ and later find the correct scaling of the projected point to get $\proj_\Sc(\x)$. 
\end{lemma}

Proof of \autoref{lem:proj_S_SB2} is given in \autoref{app:projection}. The above is also helpful in making use of $\proj(\cdot\,;\Sc)$ in place of $\proj(\cdot\,; \Sc\cap \mathbb{S}^{p-1})$ in greedy algorithms such as the one studied in \citep{tewari2011greedy}.

\subsection{Quantities based on a Representation}\label{sec:repr}
Note that while the dual norm (or its subdifferential, characterized in \autoref{subdiff_dualnorm_gen}) can be directly computed from the projection, computation of quantities such as the norm value in \eqref{eq:Snorm-from-dual}, or objects we discuss next, namely the projection onto the dual norm ball, the proximal mapping for the norm, or the subdifferential for the norm, could greatly benefit from a {\em representation} of the projection onto the structure which can then be plugged into the aforementioned optimization programs. For the structure $\Sc_{k,d}$ considered in \autoref{sec:the-norm}, we have access to an efficient representation for the dual norm in terms of a quadratically constrained quadratic program (QCQP). 

The subdifferential of a norm is useful in devising subgradient-based algorithms and can be computed via
\begin{align}\label{eq:subd-proj}
\partial \norm{\beta} = \Argmax_\theta \bigl\{ \langle \beta,\theta \rangle : \norm{\theta}^\star \leq 1 \bigr\}.
\end{align}
Alternatively, consider the proximal mapping associated to $\nor$ which is defined as the unique solution to the following optimization program, 
\begin{align}\label{def:prox}
\prox(\pran) \equiv \argmin_\theta ~ \frac{1}{2}\norm{\beta-\theta}_2^2 + \norm{\theta} .
\end{align}
The proximal mapping enables a wide range of optimization strategies that are commonly more efficient that subgradient-based methods; e.g., \citep{parikh2014proximal}. For example, in \autoref{sec:opt-algs-kd}, we briefly mentioned proximal gradient descent as well as ADMM for solving the regularized least-squares problem \eqref{eq:estimator-kd} or \eqref{eq:estimator} assuming an efficient routine for evaluating the proximal mapping. 

The proximal mapping admits a closed form solutions for simple cases such as the $\ell_1$ norm or the nuclear norm; soft-thresholding. However, more generally it can be computed through projection onto the dual norm ball, namely as
\begin{align}\label{eq:prox-proj}
\prox(\pran) 
= \beta - \argmin_\theta \bigl\{ \norm{\beta-\theta}_2^2 
:~ \norm{\theta}^\star \leq 1
\bigr\}.
\end{align} 
For computing \eqref{eq:subd-proj} or \eqref{eq:prox-proj}, one may express the dual norm ball as $\Bc^\star = \{\theta:~ \langle \beta, \theta \rangle \leq 1 ~\forall \beta\in\Bc \}$ where $\Bc = \{\beta:~ \norm{\beta} \leq 1\}$. 
Therefore, the proximal mapping may be computed through 
\begin{align*}
\prox(\pran) 
= \beta - \argmin_\theta \bigl\{ \norm{\beta-\theta}_2^2 
:~ \langle \tilde\beta,\theta \rangle \leq 1 ~~ \forall \tilde\beta\in\Bc
\bigr\} .
\end{align*}
Since $\Bc$ may have an infinite number of elements, or exponentially-many, it is not straightforward to solve such a quadratic optimization problem especially in each iteration of another algorithm such as proximal gradient descent or ADMM described in \autoref{sec:opt-algs-kd}. Therefore, a more efficient representation of the dual norm ball could enable an efficient computation of the proximal mapping, subgradients, etc.

\paragraph{Black-box versus Representable.} In the case of structure norms, namely $\nor_\Sc$, we have (by assumption) an efficient routine to evaluate the projection onto $\Sc$ which allows us to check membership (feasibility) in $\{\theta:~\norm{\theta}_\Sc^\star\leq 1\} = \{\theta:~ \norm{\proj(\theta;\Sc)}_2\leq 1\}$. Optimization (for \eqref{eq:subd-proj} or \eqref{eq:prox-proj}) given only a feasibility oracle is still not easy. However, in cases such as $\Sc_{k,d}$, it is possible to derive {\em an efficient representation} for the projection onto $\Sc$ and the dual norm, which can then replace the dual norm ball membership constraints and yield the objects of interest (subgradients or the proximal mapping) as solutions to manageable convex optimization programs. More concretely, assume we can establish 
\begin{align}\label{eq:dualnorm-eff-repr}
\norm{\theta}^\star = \min_u \bigl\{ f(\theta, u):~ (\theta, u)\in \mathcal{T} \bigr\}
\end{align}
where $\mathcal{T}$ is a finite-dimensional convex set and $f$ is a convex function. Then, the proximal mapping can be expressed as 
\[
\prox(\pran) = \beta - \argmin_\theta \bigl\{ \norm{\beta-\theta}_2^2 
:~ f(\theta,u)\leq 1,~ (\theta,u)\in\mathcal{T}
\bigr\}.
\]
Deriving a representation as in \eqref{eq:dualnorm-eff-repr} is the main focus of \autoref{sec:the-norm} for $\nor_{k\square d}^\star$; given in \autoref{lem:projS-qcqp}.

\subsection{Doubly-sparse Norms (\texorpdfstring{$k\square d$-norm}{kd-norm})}\label{sec:the-norm}
Here, we discuss a structure motivated by the statistical estimation problem at hand, namely regression in the presence of rare features. As we show, a fast discrete algorithm, namely the 1-dimensional K-means algorithm, can be used to define a norm for feature aggregation as well as for computing its optimization-related quantities.

For two fixed values $1\leq d \leq k \leq p$, the structure set $\Sc=\Sc_{k,d}$ in \eqref{eq:def-Skd} is scale-invariant and spans $\mathbb{R}^p$. Therefore, we consider the structure norm associated to $\Sc_{k,d}$ to which we refer as the $(k\square d)$-norm and we denote by $\norm{\cdot}_{k\square d}$, or $\norm{\cdot}_\sq$ when clear from the context. Specifically,
\begin{align}\label{eq:kd-def}
\norm{\beta}_{k\square d} \equiv \inf \{ \gamma>0 :~ \beta \in \gamma \Bc_{\Sc_{k,d}}\}\,,
\end{align}
with $\Bc_{\Sc_{k,d}} = \conv \{ \beta: \; \beta\in\Sc_{k,d} \,,\; \norm{\beta}_2 = 1 \}$.
According to \autoref{lem:dual-len-proj}, we have $\norm{\theta}_{k\square d}^\star(\theta) = \norm{ \proj(\theta; \Sc_{k,d}) }_2$, and in turn, $\norm{\beta}_{k\square d} = \sup\{\langle \theta, \beta\rangle:~ \norm{\theta}_{k\square d}^\star \leq 1 \}$. Next, we address the computational aspects.

\subsubsection{Examples; for Different Values of $k$ and $d$}\label{sec:kdnorms-examples}
It is clear from~\eqref{eq:def-Skd} that $\Sc_{k,d_1}\subset \Sc_{k,d_2}$ for $d_1\leq d_2$: since $k$ is fixed, if $\abs{\{\bar{\beta}_1,\ldots,\bar{\beta}_k\}} \leq d_1$ then $\abs{\{\bar{\beta}_1,\ldots,\bar{\beta}_k\}} \leq d_2$. Therefore, $\norm{\cdot}_{k\square 1} \geq \cdots \geq \norm{\cdot}_{k\square k} $ for any $k\in\{1,\ldots,p\}$. 
\begin{remark}
Note that a similar monotonicity does not hold with respect to $k$. Consider $1\leq d\leq k_1\leq k_2 \leq p$. If $\card(\beta)\leq k_1$ then $\card(\beta)\leq k_2$. However, if $\abs{\{\bar{\beta}_1,\ldots,\bar{\beta}_{k_1}\}} \leq d$, the addition of elements $\bar{\beta}_{k_1+1}=\ldots=\bar{\beta}_{k_2}=0$ to the set may increase the number of distinct values by $1$. Therefore, $\Sc_{k_1,d}\subseteq \Sc_{k_2,d+1}$ for any $1\leq d\leq k_1\leq k_2 \leq p$. 

However, with $\val(\beta)\equiv \abs{\{ \abs{\beta_i}\neq 0:~ i\in[p]\}}$ and $\widetilde{\Sc}_{k,d} \equiv \{\beta:~\card(\beta)\leq k,~ \val(\beta)\leq d\}$, the addition of the extra zero elements do not change $\val$, and we get $\widetilde{\Sc}_{k_1,d} \subseteq \widetilde{\Sc}_{k_2,d}$ for any $1\leq d\leq k_1\leq k_2 \leq p$. The new definition differs from~\eqref{eq:def-Skd} in not counting zero as a separate value among the top $k$ entries. For example, the dual norm corresponding to $\widetilde{\Sc}_{p,1}$ is $(\norm{\beta}^\star)^2 = \max_{r\in[p]} \frac{1}{r}(\sum_{i=1}^r \bar\beta_i)^2$. 
\end{remark}
Nonetheless, we have $ \norm{\cdot}_1 = \norm{\cdot}_{1\square 1}\geq \cdots \geq \norm{\cdot}_{p\square p}=\norm{\cdot}_2$. It is worth noting that for any $k\in \{1,\ldots,p\}$, $\norm{\cdot}_{k\square k}$ coincides with the $k$-support norm \citep{argyriou2012sparse}. Furthermore, \autoref{lem:all-k-1} (\autoref{lem:all-k-1-normk1}) establishes that 
\begin{align}\label{lem:norm-k-1}
\norm{\beta}_{k\square 1} = \max\{\frac{1}{\sqrt{k}}\norm{\beta}_1 , \sqrt{k}\norm{\beta}_\infty \}.
\end{align} 
As a corollary, we get $\norm{\cdot}_{p\square 1} =\sqrt{p}\norm{\cdot}_\infty$. See \autoref{fig:kdnorms} for a full picture for $\nor_{k\square d}$ and $\nor_{k\square d}^\star$.


\begin{figure}[h]
\vskip 2.3in
\begin{center}
\hskip -6.5in
\begin{picture}(0,0)
	\setlength{\unitlength}{.5cm}
	\linethickness{0.25mm}
	\put(2.5,1){\vector(1,0){11}}
	\put(14,.8){$k$}
	\put(3,.5){\vector(0,1){9}}
	\put(2.3,9){$d$}
		
	\put(4.5,2.5){\bigdot}
	\put(3.6,3){{$\nor_1$}}
	\put(6,2.5){\bigdot}
		\put(7,2.5){\smalldot}
		\put(7.5,2.5){\bigdot}
		\put(3.8,1.5){{
		\footnotesize $\max\{ \tfrac{1}{\sqrt{k}} \nor_1, \sqrt{k} \nor_\infty \}$
		}}
		\put(8,2.5){\smalldot}
	\put(9,2.5){\bigdot}
	\put(10.5,2.5){\bigdot}
	\put(10.8,2.4){{$\nor_{\infty}\cdot\sqrt{p}$}}
	
	\put(6,4){\bigdot}
		\put(7,4){\smalldot}
		\put(7.5,4){\smalldot}
		\put(8,4){\smalldot}
	\put(9,4){\bigdot}
	\put(8.3,4.4){}
	\put(10.5,4){\bigdot}

		\put(7,5){\smalldot}
		\put(7.5,5.5){\bigdot}
		\put(5.7,6.1){{$\nor_\ksp$}}
		\put(8,6){\smalldot}	

		\put(9,5){\smalldot}	
		\put(9,5.5){\smalldot}	
		\put(9,6){\smalldot}	

		\put(10.5,5){\smalldot}	
		\put(10.5,5.5){\bigdot}	
		\put(10.8,5.4){{$\nor_{p \sq d}$}}
		\put(10.5,6){\smalldot}	

	\put(9,7){\bigdot}
	\put(10.5,7){\bigdot}

	\put(10.5,8.5){\bigdot}	
	\put(10.8,8.4){{$\nor_2$}}
\end{picture}

\vskip -.2in \hskip -.2in
\begin{picture}(0,0)
	\setlength{\unitlength}{.5cm}
	\linethickness{0.25mm}
	\put(2.5,1){\vector(1,0){11}}
	\put(14,.8){$k$}
	\put(3,.5){\vector(0,1){9}}
	\put(2.3,9){$d$}
		
	\put(4.5,2.5){\bigdot}
	\put(3.6,3){{$\nor_\infty$}}
	\put(6,2.5){\bigdot}
		\put(7,2.5){\smalldot}
		\put(7.5,2.5){\bigdot}
		\put(5.1,1.6){{
		$\nor_{1,\text{top-}k} \cdot \frac{1}{\sqrt{k}}$
		}}
		\put(8,2.5){\smalldot}
	\put(9,2.5){\bigdot}
	\put(10.5,2.5){\bigdot}
	\put(10.8,2.4){{$\nor_1\cdot\tfrac{1}{\sqrt{p}}$}}
	
	\put(6,4){\bigdot}
		\put(7,4){\smalldot}
		\put(7.5,4){\smalldot}
		\put(8,4){\smalldot}
	\put(9,4){\bigdot}
	\put(8.3,4.4){}
	\put(10.5,4){\bigdot}

		\put(7,5){\smalldot}
		\put(7.5,5.5){\bigdot}
		\put(5.4,6.3){{$\nor_{2,\text{top-}k}$}}
		\put(8,6){\smalldot}	

		\put(9,5){\smalldot}	
		\put(9,5.5){\smalldot}	
		\put(9,6){\smalldot}	

		\put(10.5,5){\smalldot}	
		\put(10.5,5.5){\bigdot}	
		\put(10.8,5.4){{$\nor_{p \sq d}^\star $}}
		\put(10.5,6){\smalldot}	

	\put(9,7){\bigdot}
	\put(10.5,7){\bigdot}

	\put(10.5,8.5){\bigdot}	
	\put(10.8,8.4){{$\nor_2$}}
\end{picture}
\end{center}
	\caption{Doubly-sparse norms $\nor_{k,d}$ (on the left) and their dual norms (on the right) for all possible pairs $(k,d)$ which unifies some new and existing vector norms. $\nor_\ksp$ denotes the $k$-support norm and is dual to the $\ell_2$ norm of top $k$ entries in absolute value denoted by $\nor_{2,\text{top-}k}\,$. The $\ell_1$ norm of top $k$ entries in absolute value is denoted by $\nor_{1,\text{top-}k}\,$. This figure has been adapted from \citep{jalali2013convex}.} 
	\label{fig:kdnorms}
\end{figure}
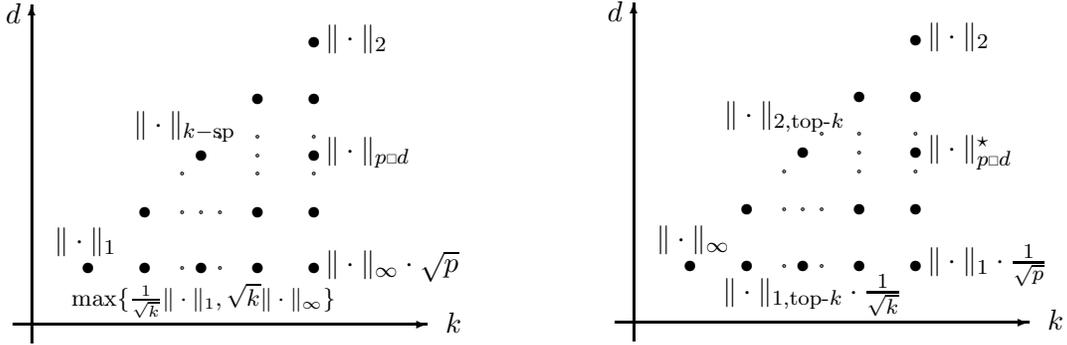

\subsubsection{The Projection and its Combinatorial Representation}
Before discussing the projection onto~$\Sc_{k,d}$, in \autoref{lem:proj-Skd}, we state a lemma to establish a reduction principle that allows simplifying such projection. This reduction makes use of the {\em invariance} and {\em monotonicity} properties for such projection. We established the former in \autoref{lem:inv-proj}. For the latter, \autoref{lem:monot} can be thought of as an implication of the Occam razor principle. In simple terms, if the characteristic property that defines a structure ignores zero values, the projected vector will have a support included in the support of original vector; there is no need to have new values in those places when computing the projection. Similarly, if the characteristic property treats similar values as one value, there is no need to map them to distinct values in the projection. These suggest that we can always consider problems in a reduced space; only considering non-zero entries and distinct values in our structure of interest, namely $\Sc_{k,d}$.

\begin{lemma}[Monotonicity]\label{lem:monot}
Consider a closed scale-invariant set $\Sc\subseteq\mathbb{R}^p$ that spans $\mathbb{R}^p$. Moreover, consider any orthogonal projection $\theta \in \proj(\beta; \Sc)$. We have: 
\begin{itemize}
\item 
If $u\in\Sc$ implies $u-u_ie_i\in \Sc$ for all $i\in[p]$, then $\supp(\theta) \subseteq \supp(\beta)$ for any $\theta \in \proj(\beta; \Sc)$; i.e., $\beta_i=0$ implies $\theta_i = 0$ for any $i\in[p]$ and any $\theta\in \proj(\beta;\Sc)$. 

More generally, consider an orthogonal projection matrix $P=P^\sT=P^2$. If (i) $u\in\Sc$ implies $Pu\in\Sc$, and, (ii) $\beta=P\beta$, then, $\theta\in\proj(\beta;\Sc)$ implies $P\theta=\theta$. 

\item 
If $u\in\Sc$ implies $u - (u_i-u_j)e_i\in\Sc$ for all $i,j\in[p]$, then $\beta_i = \beta_j$ implies $\theta - (\theta_i-\theta_j)e_i \in\proj(\beta;\Sc)$ for any $\theta\in\proj(\beta;\Sc)$.

More generally, consider a pair $(A,B)$ of oblique projection matrices, i.e., $A^2=A$ and $B^2=B$, satisfying $A^\sT A + B^\sT B =2I$. Assume $A\beta=B\beta=\beta$, and that $u\in\Sc$ implies $Au,Bu\in\Sc$. Then, for any $\theta\in\proj(\beta;\Sc)$, we have $A\theta, B\theta\in \proj(\beta;\Sc)$. 

\end{itemize}
\end{lemma}
Proofs for \autoref{lem:monot}, \autoref{lem:proj-cardk}, \autoref{lem:proj-bar-Skd}, \autoref{lem:proj-Skd}, and \autoref{lem:proj-Skd-combinat-rep}, are given in \autoref{app:kd-norm}.

\begin{lemma}\label{lem:proj-cardk}
If $\Sc$ is sign and permutation invariant and $\Sc\subseteq \{\beta:~ \card(\beta)\leq k\}$, then for all $\theta\in\proj(\beta; \Sc)$ we have $\theta_i=0$ whenever $\abs{\beta_i} < \bar\beta_k$. 
\end{lemma}

\begin{lemma}\label{lem:proj-bar-Skd}
For a given $\beta$, consider $\sign(\beta)$ (where $\sign(0)$ is arbitrary from $\{+1,-1\}$) and a permutation $\pi$ for which $\pi(\abs{\beta})$ is sorted in descending order. Then 
\[
\proj(\beta;\Sc_{k,d}) = \bigl\{
\pi^{-1}(\theta) \circ \sign(\beta):~ \theta\in\proj(\bar\beta; \Sc_{k,d})
\bigr\}~,~~
\proj(\bar\beta;\Sc_{k,d}) = \bigl\{
\pi(\abs{\theta}):~ \theta\in\proj(\beta; \Sc_{k,d})
\bigr\}
\]
\end{lemma}

\begin{lemma}[\citep{jalali2013convex}]\label{lem:proj-Skd}
The following procedure returns all of the projections of $\beta\in\mathbb{R}^p$ onto $\Sc_{k,d}$ defined in~\eqref{eq:def-Skd}:
\begin{itemize}
\item[(i)] project $\beta$ onto $\Sc_{k,k}$ (zero out all entries except the $k$ of the entries with largest absolute values) and consider the shortened output $\beta^{(k)}\in\mathbb{R}^k$, 
\item[(ii)] project $\beta^{(k)}$ onto $\Sc_{k,d} \subset \mathbb{R}^k$ (perform the 1-dimensional K-means algorithm on entries of $\abs{\beta^{(k)}}$ and stack the corresponding centers with signs according to $\beta^{(k)}$), 
\item[(iii)] put the new entries back in a $p$-dimensional vector, by padding with zeros. 
\end{itemize}
Repeat this procedure when there are multiple choices in steps (i) or (ii). 
\end{lemma}

We will use \autoref{eq:dual_norm_gen} to compute the dual norm and further derive a combinatorial representation for it. Note that while computing the projection itself can be done through K-means, we are interested in a representation for this projection which can can then be used in computing other quantities; as discussed in \autoref{sec:repr}.

\begin{lemma}\label{lem:proj-Skd-combinat-rep}
For a given vector $\theta\in \reals^p$, denote by $\bar{\theta}$ the sorted version of $|\theta|$ in descending order, i.e., $\bar{\theta}_1\ge \cdots\ge \bar{\theta}_p \ge 0$. Then, 
\begin{align*}
\norm{ \proj(\theta; \Sc_{k,d}) }_2^2 
= \max \bigl\{ \sum_{i=1}^d \frac{1}{\abs{\int_i}} (\one^\sT \bar{\theta}_{\int_i})^2 :~ (\int_1,\cdots,\int_d)\in\ptnn(k,d) \bigr\} \nonumber
\end{align*}
where $\ptnn(k,d)$ is the set of all partitions of $\{1,\ldots,k\}$ into 
$d$ groups of consecutive elements. Then, 
\begin{align*}
\left[
	\frac{\one^\sT \bar{\theta}_{\int_i}}{\abs{\int_i}}\one_{\int_i}, 
\cdots
	\frac{\one^\sT \bar{\theta}_{\int_d}}{\abs{\int_d}}\one_{\int_d}, 
	0, \cdots, 0
\right]^\sT
\in \proj(\bar{\theta}; \Sc_{k,d}).
\end{align*}
\end{lemma}

Using \autoref{eq:dual_norm_gen}, the statement of the \autoref{lem:proj-Skd-combinat-rep} can be alternatively represented as 
\begin{align}\label{dual-norm-characterization}
	(\norm{\beta}_{k\square d}^\star)^2
	=(\norm{\bar{\beta}}_{k\square d}^\star)^2 
	= \sup_{A\in \BDbar(k,d)} \, \bar{\beta}^\sT A \bar{\beta}
\end{align}
where $\bar{\beta}$ is nonnegative and non-increasing, and $\BDbar(k,d)$ is the set of block diagonal matrices 
with $d$ blocks exactly covering the first $k$ rows and columns and zero elsewhere, where on each block of size $q$, all of the entries are equal to~$\frac{1}{q}$. 
Note that if the input is not a sorted nonnegative vector, then we need to consider $\BD(k,d)\equiv\{PAP^\sT:P\in \mathcal{P}_\pm,A\in\BDbar(k,d)\}$, 
where $\mathcal{P}_\pm $ is the set of signed permutation matrices. 
This brings us to
\begin{align}\label{dual-norm-var}
	(\norm{\beta}_{k\square d}^\star)^2
	= \sup_{A\in \BD(k,d)} \, \beta^\sT A \beta.
\end{align}
The aforementioned representations, in \autoref{lem:proj-Skd-combinat-rep}, \autoref{dual-norm-characterization}, and \autoref{dual-norm-var}, all depend on an efficient characterization of combinatorial sets such as $\ptnn(k,d)$ or $\BD(k,d)$. \autoref{lem:BD-size} below shows that $\BD(k,d)$ is of exponential size, which renders direct optimization inefficient.

		\begin{lemma}\label{lem:BD-size}
				$|\BD(k,d)| < (\frac{2epd}{k})^k$.
		\end{lemma}
\autoref{lem:BD-size} is proved in \autoref{app:kd-norm}.

Next, we review a dynamic programming approach to reformulate the above in terms of a quadratic program. 

\subsubsection{A Dynamic Program and a QCQP Representation}\label{sec:qcqp}

Consider a non-negative sorted vector $\bar{\beta}$ with $\bar{\beta}_1\geq \cdots \geq \bar{\beta}_p\geq 0$. A dynamic program can be used to perform 1-dimensional K-means clustering required in the second step of projection onto $\Sc_{k,d}$ (detailed in \autoref{lem:proj-Skd}) as well as in \autoref{lem:proj-Skd-combinat-rep}. For example, see \citep{wang2011ckmeans} for how a 1-dimensional K-means clustering problem can be cast as a dynamic program. Furthermore, this dynamic program can be represented as a quadratically-constrained quadratic program (QCQP) \citep{jalali2013convex} as discussed next. 
More specifically, the following two lemmas describe how projection onto $\Sc_{k,d}$ and the dual norm unit ball $\cB^*$ can be computed as solving a QCQP. 
See \autoref{fig:parallelogram} for an illustration related to $\ptnn(k,d)$ and the dynamic program. 

\begin{lemma}\label{lem:projS-qcqp}
We have 
\begin{align*}
\norm{\proj(\bar{\beta}; \Sc_{k,d})}_2^2 = 
\min_{\{\nu_{m,e}\}} \Bigl\{ \nu_{k,d}:~
\frac{1}{s-m+1} ( \one^\sT \bar{\beta}_{[m,s]})^2 \leq \nu_{s,e} - \nu_{m-1,e-1} ~~ \forall (e,m,s)\in \T(k,d)
\Bigr\},
\end{align*}
where 
$\T(k,d)\; \equiv\; \{(e,m,s) : 1\leq e \leq d ,~ e \leq m \leq s \leq k-d+e\}$, and
$u_{[m,s]} = [u_m, \cdots, u_s]$.
\end{lemma}
Proof for \autoref{lem:projS-qcqp} is given in \autoref{app:kd-norm}.

\begin{lemma}\label{lem:proj-dual-ball-qcqp}
For $\Bc^\star = \{u:~\norm{u}^\star_{k\square d} \leq 1\}$, we have
\begin{multline*}
\proj(\bar{\theta}; \cB^\star) 
= \argmin_{u} \min_{\{\nu_{m,e}\}} 
\Bigl\{ \norm{\bar{\theta}-u}_2^2 :~ 
	\nu_{k,d}\leq 1, ~ u_1\geq \cdots \geq u_p \geq 0,\\
\frac{1}{s-m+1} ( \one^\sT u_{[m,s]})^2 \leq \nu_{s,e} - \nu_{m-1,e-1} ~~ \forall (e,m,s)\in \T(k,d)
\Bigr\}
\end{multline*}
which is a QCQP. 
\end{lemma}
Proof for \autoref{lem:proj-dual-ball-qcqp} is given in \autoref{app:kd-norm}.

The above provides us with the proximal mapping through $\prox_{\norm{\cdot}_\sq} (\bar{\theta}) 
= \bar{\theta} - \proj(\bar{\theta}; \cB^\star)$. A QCQP such as the one above can be solved via interior point methods among many others.

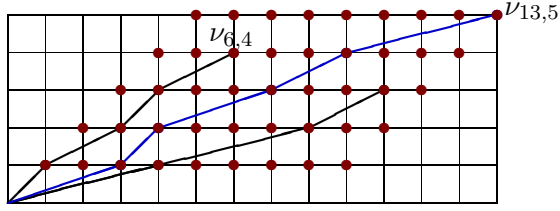
\begin{figure}[h]

\vskip 1in
\centering
~~~~~~~~
\begin{picture}(200,0)
	\setlength{\unitlength}{.5cm}
	
\linethickness{0.075mm}
\multiput(0,0)(1,0){14}{\line(0,1){5}}
\multiput(0,0)(0,1){6}{\line(1,0){13}}

\linethickness{0.025mm}
\multiput(0,0)(.85,.85){6}{\line(1,1){.7}}
\multiput(13,5)(-.85,-.85){6}{\line(-1,-1){.7}}
\put(13.2,5){$\nu_{13,5}$}
\thicklines

\put(0,0){\line(4,1){4}}	\put(4,1){\circle{.2}}
\put(4,1){\line(4,1){4}}	\put(8,2){\circle{.2}}
\put(8,2){\line(2,1){2}}	\put(10,3){\circle{.2}}

\put(0,0){\line(1,1){1}}	\put(1,1){\circle{.2}}
\put(1,1){\line(2,1){2}}	\put(3,2){\circle{.2}}
\put(3,2){\line(1,1){1}}	\put(4,3){\circle{.2}}
\put(4,3){\line(2,1){2}}	\put(6,4){\circle{.2}}
\put(5.35,4.3){$\nu_{6,4}$}

\put(0,0){\line(4,1){4}}	\put(4,1){\circle*{.2}}

\color{darkblue}
\put(0,0){\line(3,1){3}}	\put(3,1){\circle*{.3}}
\put(3,1){\line(1,1){1}}	\put(4,2){\circle*{.3}}
\put(4,2){\line(3,1){3}}	\put(7,3){\circle*{.3}}
\put(7,3){\line(2,1){2}}	\put(9,4){\circle*{.3}}
\put(9,4){\line(4,1){4}}	\put(13,5){\circle*{.3}}
\color{black}

\color{red!50!black}
\multiput(1,1)(1,0){9}{\circle*{.3}}
\multiput(2,2)(1,0){9}{\circle*{.3}}
\multiput(3,3)(1,0){9}{\circle*{.3}}
\multiput(4,4)(1,0){9}{\circle*{.3}}
\multiput(5,5)(1,0){9}{\circle*{.3}}
\color{black}

\end{picture}

	\caption{Red dots correspond to coordinates $(m,e)$ where $1\leq e \leq d$ and $e \leq m \leq k-d+e$, $k=13$, and $d=5$. Each $(\int_1,\cdots,\int_d)\in\ptnn(k,d)$ can be uniquely represented as a continuous union of $d$ line segments $\ell_i$ connecting $( \max(\int_{i-1}),i-1)$ to $( \max(\int_{i}),i)$, for $i\in[d]$, with $\int_0\equiv \{0\}$. 
	Associating to a segment $\ell_i$ a cost of $\frac{1}{\abs{\int_i}} (\one^\sT \bar{\theta}_{\int_i})^2$, the dynamic program (or its reformulation as a QCQP) is aimed at finding the most expensive path of $d$ segments (of the form described above) from $(0,0)$ to $(k,d)$. The optimal values of $\nu_{m,e}$ will correspond to the maximal cost of such paths from $(0,0)$ to $(m,e)$.
	}
	\label{fig:parallelogram}
\end{figure}

\begin{remark}
The representation of the dual norm in~\eqref{eq:atomic-dual-dist} is through a maximization. Therefore, in replacing a dual norm constraint with this representation, we will have as many as $\abs{\Ac}$ constraints which leads to a semi-infinite optimization program in many cases of interest. The representation in \eqref{eq:Snorm-dual-dist} is also a maximization problem ($\ell_2$ squared minus distance squared) with possibly many constraints. 
However, in the case of $\Sc_{k,d}$, the use of \eqref{eq:Snorm-dual-dist} allows for reformulation in terms of a dynamic program which reduces the number of constraints from exponentially-many, namely $\abs{\BD(k,d)}$, to $\abs{\T(k,d)} \leq k^2d$. 
\end{remark}

\section{Prediction Error Bound for Regularized Least-Squares}\label{sec:concise-var}
Consider a measurement model $y = X\beta^\star + \epsilon$, where $X\in\mathbb{R}^{n\times p}$ is the {\em design} matrix and $\epsilon\in\mathbb{R}^n$ is a noise vector. For any given norm $\nor$, and not only those studied in \autoref{sec:Snorms-def}, we then consider the regularized estimator in~\eqref{eq:estimator} with $\lambda$ as the regularization parameter. Rather standard analysis of \eqref{eq:estimator} yields {\em prediction error bounds}, namely bounds for $\norm{X(\beta^\star - \hat\beta)}_2$, as well as {\em estimation error bounds}, namely bounds on $\|\hbeta-\beta^\star\|$ and $\|\hbeta-\beta^\star\|_2$. In this section, we review a standard prediction error bound (\autoref{lem:oracle}) and then present a novel analysis for establishing bounds on the regularization parameter which is needed in such prediction error bound. Estimation error bounds will be studied in \autoref{sec:est} building upon the results presented here. 

	\begin{lemma}[Prediction Error]\label{lem:oracle} 
	If $\lambda \ge \|\frac{2}{n} X^\sT \epsilon\|^\star$, then $\hbeta$ obtained from \eqref{eq:estimator-kd} satisfies 	\begin{align}\label{prediction-error}
		\frac{1}{n}\norm{X(\beta^\star - \hbeta)}_2^2 \leq 3\lambda \norm{\beta^\star} \,.
	\end{align}
	\end{lemma}
	
	\autoref{lem:oracle} follows from a standard oracle inequality and is proved in \autoref{app:prediction}.

The prediction error bound in \autoref{lem:oracle}, and the estimation error bounds in \autoref{thm:estimation}, are conditioned on $\lambda \geq \norm{\frac{2}{n}X^\sT\epsilon}^\star$. In this section we make a novel use of the Hanson-Wright inequality to compute this bound for a broad family of noise vectors $\epsilon \in\mathbb{R}^n$ while our bounds are deterministic with respect to the design matrix. Our proof assumes a concise variational representation for the dual norm (as in \eqref{eq:squared-dual-norm-var}) and provides a bound in terms of novel {\em aggregate measures} of the design matrix {\em induced by the norm} (given in \eqref{eq:def-all-phi}). In the following, we elaborate on the variational formulation. In \autoref{sec:UoS}, we examine this property for structure norms (defined in \autoref{sec:Snorms-def}). In \autoref{sec:concise-examples}, we provide examples of norms admitting a concise representation, and finally, in \autoref{sec:bounds-reg-param}, we state the bounds.

\paragraph{A Concise Variational Formulation.} 
Any squared vector norm can be expressed in a variational form (\citep{bach2012optimization} (Prop.~1.8~and Prop.~5.1) and \citep{jalali2017variational}): consider any norm $\norm{\cdot}$ and its dual $\norm{\cdot}^\star$. Then, 
\begin{align}\label{eq:squared-dual-norm-var}
(\norm{\beta}^\star)^2 
= \sup_{\norm{\theta}\leq 1} \langle \theta, \beta \rangle ^2 
= \sup_{\Mmat\in\Mset} \beta^\sT \Mmat \beta
\end{align}
where $\Mset= \{\theta\theta^\sT:~ \norm{\theta}\leq 1\}$. It is easy to see that the set $\Mset$ that is used in the variational representation above is not unique. For example, $\conv(\Mset)$ or $\Mset = \{\theta\theta^\sT:~ \theta \in\ext(\Bc_{\norm{\cdot}})\}$ also work. 
For an atomic norm (defined in \eqref{eq:atomic_repr}), it is clear from the above that $\abs{\Mset} \leq \abs{\Ac}$. However, in cases such as $\norm{\cdot}_{k\square d}$, one can find a set $\Mset$ which is much smaller than $\Ac$. For example, in \autoref{ex:group-l1-Mset}, \autoref{ex:ksupp-Mset}, as well as for $\nor_{k\square d}$, the atomic set is infinite while we can find a small finite-size $\Mset$. For a norm such as the ordered weighted $\ell_1$ norm \citep{zeng2014ordered}, which has a finite number of atoms, it seems that a smaller $\Mset$ cannot be found; see \autoref{ex:Mset-dual-owl}. 

For a norm that admits a representation as in \eqref{eq:squared-dual-norm-var} with a reasonably-sized $\Mset$, we can provide a prediction error bound in terms of $\abs{\Mset}$ as well as certain {\em aggregation} quantities defined based on the elements in $\Mset$. For example, in the case of $\nor_{k\square d}$, with a corresponding variational representation given in \eqref{dual-norm-var}, we provide the prediction error bound in \autoref{thm:pred-err-kd}. As another example, in \autoref{sec:examples-kd}, we provide these calculations for the case of $k$-support norm as well as the $(k\square 1)$-norm.

\subsection{Example: Structure Norms with Finite Unions of Subspaces}\label{sec:UoS}

Consider a closed scale-invariant set $\Sc$ that spans $\mathbb{R}^p$ and the corresponding structure norm $\nor_\Sc$ and unit norm ball $\Bc_\Sc = \{\beta:~\norm{\beta}_\Sc\leq 1\}$. 
In this section, we connect a representation for $\nor_\Sc^\star$ as in \eqref{eq:squared-dual-norm-var} to a representation of $\Sc$ as a union of subspaces. 

A closed scale-invariant set $\Sc$ can always be represented as a union of subspaces. However, imagine this is possible for a given set with finitely many subspaces; namely $m\geq 1 $ subspaces. For $i\in[m]$, denote by $U_i\in\mathbb{R}^{p\times d_i}$ an orthonormal basis for the $i$-th subspace. Then, 
\begin{align*}
\norm{\theta}_\Sc^\star 
&= \sup\{ \langle \beta, \theta \rangle:~ \norm{\beta}_\Sc\leq 1 \} \\
&= \sup\{ \langle \beta, \theta \rangle:~ \beta\in \ext(\Bc_\Sc) \} \\
&= \sup\{ \langle \beta, \theta \rangle:~ \beta = U_iw,~ w\in\mathbb{S}^{d_i-1},~i\in[m] \} \\
&= \max_{i\in[m]}~\norm{U_i^\sT\theta}_2.
\end{align*}
Then, it is easy to see that we get a representation as in \eqref{eq:squared-dual-norm-var} with 
\begin{align}
\Mset = \bigl\{ U_iU_i^\sT:~ i \in[m] \bigr\}
\end{align}
where each element of $\Mset$, namely $U_iU_i^\sT$, is an orthogonal projector of rank~$d_i$. \autoref{lem:dualvar-UoS} summarizes these observations and its proof is given in \autoref{app:prediction}. 

\begin{lemma}\label{lem:dualvar-UoS}
Consider a finite set of positive semidefinite matrices $\Mset = \{M_1, \ldots, M_m\}\subset \mathbb{R}^{p\times p}$ and $f(\beta) \equiv \sup_{\Mmat\in\Mset} \beta^\sT \Mmat \beta$. Then, $\sqrt{f}$ is a semi-norm. 

Suppose $\conv(\Mset)\cap \mathbb{S}_{++}^p\neq \emptyset$. Then, 
(i) $\sqrt{f}$ is a norm. 
(ii) if each $M_i$ is an orthogonal projector then $\sqrt{f}\equiv \nor_\Sc^\star$ for $\Sc = \bigcup_{i\in[m]} \range(M_i)$. 
\end{lemma}

\subsection{Examples of Norms with a Concise Variational Representation}\label{sec:concise-examples}
In the following, we review some examples with a concise variational representation. 

\begin{example}\label{ex:group-l1-Mset}
Consider the group $\ell_1$ norm with $K$ non-overlapping groups (sum of $\ell_2$ norms over each group). Then, in the representation of the dual norm, we can use $\Mset = \{\Mmat_1, \ldots, \Mmat_K\}$ where $\Mmat_i$ is the identity matrix over rows and columns corresponding to the $i$-th group and zero elsewhere. We get $\abs{\Mset} = K$, the number of groups.

More generally, consider the overlapping group Lasso norm \citep{jacob2009group} defined as 
\[
\norm{\beta} \equiv \inf \bigl\{ \sum_{i=1}^K \norm{v^{(i)}}_2 :~ \beta = \sum_{i=1}^K v^{(i)},~ v^{(i)}\in\mathbb{R}^p,~\supp(v^{(i)})\subseteq \cG_i, ~ \text{for }i\in[K] \bigr\}
\]
where $\cG=(\cG_1,\ldots, \cG_K)$ is a given set of $K$ subsets of $[p]$ that may overlap; if they do not overlap and they partition $[p]$, $\nor$ reduces to the group $\ell_1$ norm mentioned above. As characterized in Lemma 2 of \citep{jacob2009group}, the dual norm is given by 
\[
\norm{\theta}^\star = \max_{i\in [K]}~ \norm{\theta_{\cG_i}}_2
\]
where $\theta_{\cG_i}$ is the restriction of $\theta\in\mathbb{R}^p$ to the entries in $\cG_i\subseteq[p]$. The above representation of $\nor^\star$ can be used to derive a representation as in~\eqref{eq:squared-dual-norm-var} where $\Mset = \{M_1,\ldots, M_K\}$, and, for each $i\in[K]$, $M_i$ is the identity matrix over rows and columns corresponding to $\cG_i$ and zero elsewhere. 
\end{example}

The bound given in \autoref{lem:BD-size} quickly deteriorates as $d$ gets close to $k$ or $1$. \autoref{ex:ksupp-Mset} and \autoref{ex:Mset-k1} are presented to provide improved bounds for $\nor_{k\square k}$ and $\nor_{k\square 1}$, respectively.

\begin{example}\label{ex:ksupp-Mset}
Consider the $k$-support norm, denoted by $\norm{\cdot}_\ksp$ and defined as the symmetric gauge function corresponding to $\Ac = \{x:~ \norm{x}_0 \leq k,~ \norm{x}_2 =1\}$ \citep{argyriou2012sparse}. It is easy to see that the $k$-support norm coincides with the doubly-sparse norm for $k=d$. It has been shown that $(\norm{\theta}_\ksp^\star)^2 = \sum_{i=1}^k \bar\theta_i^2$ \citep{argyriou2012sparse}. A representation as in \eqref{eq:squared-dual-norm-var} through outer products of atoms of the $k$-support norm ball, namely $\Mset = \{\theta\theta^\sT:~ \theta\in\ext(\Bc_\ksp)\}$, leaves us with a set $\Mset$ with infinite number of elements. On the other hand, it is easy to verify that 
\begin{align}
\Mset = \bigl\{ \diag(s):~ s\in \{0,1\}^p, \norm{s}_0 = k \bigr\}
\end{align}
provides a valid expression for $(\norm{\theta}_\ksp^\star)^2$ as in \eqref{eq:squared-dual-norm-var}. Observe that $\abs{\Mset} = \binom{p}{k} \leq (ep/k)^k$. 
\end{example}

\begin{example}\label{ex:Mset-k1}
It is shown in \autoref{lem:all-k-1} that 
\begin{itemize}
\item 
$\ext(\Bc_{k\square 1}) = \Sc_{k,1}\cap \mathbb{S}^{p-1} = \{Q\theta:~ Q\in \mathcal{P}_\pm,~ \theta = \frac{1}{\sqrt{k}}[\one_k^\sT ~,~ \zero_{p-k}^\sT]^\sT\}$,

\item 
$\ext(\Bc_{k\square 1}^\star) = \{Q\theta:~\theta\in A,~Q\in\mathcal{P}_\pm\}$ where $A = \{\sqrt{k}e_1, \frac{1}{\sqrt{k}}\one_p\}$.
\end{itemize}
Therefore, it is easy to see that a concise representation exists,
\begin{itemize}
\item in the case of regularization with $\nor_{k\square 1}$, with $\abs{\Mset}\leq \binom{p}{k}\leq (ep/k)^k$,
\item in the case of regularization with $\nor_{k\square 1}^\star$, with $\abs{\Mset}\leq p+1$,
\end{itemize}
for representing their dual norms. 
\end{example}

From \autoref{lem:all-k-1} we know that $\nor_{k\square 1}^\star$ is an ordered weighted $\ell_1$ norm with $w = \frac{1}{\sqrt{k}}[\one_k^\sT~,~ \zero_{p-k}^\sT]^\sT$. While \autoref{ex:Mset-k1} establishes a concise variational formulation in this case, an arbitrary ordered weighted $\ell_1$ norm may not be concisely representable, as discussed next. 

\begin{example}\label{ex:Mset-dual-owl}
Here, we provide a quadratic variational representation for $\nor_\owl$ inspired by Example 1.2~in \citep{chen2015structured}. 
Recall the atomic set for $\nor_\owl$ from Theorem 1 of \citep{zeng2014ordered} and the variational representation from \eqref{eq:squared-dual-norm-var} with 
\[
\Mset= \{\theta\theta^\sT:~ \theta\in \ext(\Bc_{\nor_\owl})\}
= \bigcup_{i\in[p]} \bigcup_{S:\abs{S}=i} \Bigl\{\frac{1}{(\sum_{j=1}^i w_j)^2} v_Sv_S^\sT: v\in\{\pm1\}^p \Bigr\}.
\]
It is easy to see that $\abs{\Mset} \leq \sum_{i=1}^p \binom{p}{i}2^{i-1} = (3^p-1)/2$ which is not a good bound for problems in which $p$ is big. 
\end{example}

\begin{example}
Consider two arbitrary norms $\nor_{(1)}$ and $\nor_{(2)}$ with representations for their squared dual norms as in \eqref{eq:squared-dual-norm-var} through $\Mset_1$ and $\Mset_2$, respectively. Then, for the infimal convolution of the two norms, defined as 
\begin{align}
\norm{\beta} \equiv \inf \bigl\{ \norm{u}_{(1)} + \norm{v}_{(2)}:~ \beta = u+v \bigr\}\,,
\end{align}
we know (e.g., see Fact 2.21~in \citep{artacho2014applications}) that $\nor^\star \equiv \max\{ \nor_{(1)}^\star , \nor_{(2)}^\star \} $. Therefore, we get a representation for $\nor^\star$ as in \eqref{eq:squared-dual-norm-var} with $\Mset = \Mset_1\cup \Mset_2$. See \citep{jalali2010dirty,agarwal2012noisy} for applications of the infimal convolution in regularization. 
\end{example}

\begin{remark}
The above is not an exhaustive list of norms with a concise variational representation for their dual. For example, consider $\Omega_2^\star(\cdot)$ ($p=q=2$) defined in Equation (2) of \citep{obozinski2016unified}. Depending on the submodular function $F$ used in this definition, one might be able to get smaller representations. 
\end{remark}

\subsection{Bounds on the Regularization Parameter}
\label{sec:bounds-reg-param}

		\begin{definition}[Convex concentration property] \label{def:cvx-conc-prop}
		Let ${x}$ be a random vector in $\reals^n$. We will say that ${x}$ has the convex concentration property with constant $K$ if for every $1$-Lipschitz convex function $h:\reals^n\rightarrow\reals$, we have $\mathbb{E}[h({x})]<\infty$ and for every $t>0$,
			\begin{align*}
				\mathbb{P}\big\{\abs{h({x})-\mathbb{E}[h({x})]}\ge t\big\}\le 2 e^{-\frac{t^2}{2K^2}}.
			\end{align*}
		\end{definition}

\begin{lemma}[Hanson-Wright inequality; \citep{adamczak2015note}]\label{lem0} Let $u$ be a mean-zero random vector in $\reals^n$. There exists a constant $c>2$, such that if $u$ has the convex concentration property with constant $K$ then for any matrix ${B}\in\reals^{n\times n}$ and every $t>0$,
			\begin{align}\label{eq:hanson}
				\mathbb{P}\big\{\abs{u^\sT Bu-\mathbb{E}[u^\sT{B}u]}\ge t\big\} \le 2\exp\left(-\frac{1}{c}\min\left(\frac{t^2}{2K^4\|B\|_F^2},\frac{t}{K^2\|B\|}\right)\right).
			\end{align}
		\end{lemma}
		
	\begin{proposition}\label{lem:noise-dual-norm-Mset}
		Suppose that $\epsilon\in \reals^n$ is a zero-mean random vector with covariance matrix $\Sigma \equiv\mathbb{E}[\epsilon \epsilon^\sT] \in \reals^{n\times n}$, such that $\Sigma^{-1/2} \epsilon$ satisfies the convex concentration property (\autoref{def:cvx-conc-prop}) with parameter at most $\eta$. 
Moreover, assume \autoref{eq:squared-dual-norm-var} holds for $\nor$ and a finite set $\Mset\subset\mathbb{R}^{p\times p}$. 		
		Then, for any value of $0<p_0<\frac{1}{2}$, the following holds true with probability at least $1 - 2 p_0$,
		\begin{align}
			\norm{\frac{1}{n} X^\sT \epsilon}^\star \le \agg
		\end{align}
		where 
		\begin{align}
		\begin{split}\label{eq:def-all-phi}
		\agg &\equiv \frac{1}{\sqrt{n}}\left( \agg_0 + 2\eta^2 \cdot \max\left\{ \agg_2 \sqrt{\kappa} ~,~ \agg_1 \kappa \right\}
			\right)^{1/2} \\
			\tilde{X} &\equiv \Sigma^{1/2} X\,,\\
		\agg_0 &\equiv \sup_{A\in\Mset}~ \frac{1}{n} \Tr(\tilde{X}A\tilde{X}^\sT), \\
		\agg_1 &\equiv \sup_{A\in\Mset}~ \frac{1}{n}\norm{\tilde{X}A\tilde{X}^\sT}_{\rm op},\\
		\agg_2 &\equiv \sup_{A\in\Mset}~ \frac{1}{n}\norm{\tilde{X}A\tilde{X}^\sT}_F\,,\\
		\kappa &\equiv \frac{c}{2}\log\frac{\abs{\Mset}}{p_0}\,,
        \end{split}
		\end{align}
		where $c>2$ is the constant in the Hanson-Wright inequality given in \autoref{lem0}.
	\end{proposition}
\begin{proof}[Proof of \autoref{lem:noise-dual-norm-Mset}]
Define $g = \frac{1}{n} X^\sT \epsilon$ which is a random vector. Invoking the characterization of dual norm $\|\cdot\|_{k\square d}^\star$ given by~\eqref{eq:squared-dual-norm-var}, we have
		\begin{align*}
			(\norm{g}^\star)^2
			= \sup_{A\in \Mset} \, g^\sT A g 
			= \sup_{A\in \Mset} \, \frac{1}{n} \epsilon^\sT (\frac{1}{n}X A X^\sT) \epsilon . 
		\end{align*}
		We next use a Hanson-Wright inequality to upper bound the right-hand side with high probability. More specifically, we use a result by~\citep{adamczak2015note} on the Hanson-Wright inequality given in \autoref{lem0}. 

For any fixed $A\in\Mset$ (need not be positive semidefinite) define $B = \frac{1}{n}XAX^\sT$. Then, 
\[
\mathbb{E} [\epsilon^\sT B \epsilon] 
= \langle \mathbb{E} [\epsilon\epsilon^\sT], B \rangle 
= \langle \Sigma, B \rangle \leq \agg_0,
\]
where $\agg_0$ is defined in \autoref{eq:def-all-phi}. 
Therefore, for any $t>0$, Hanson-Wright inequality implies 
\[
\mathbb{P}\left( \epsilon^\sT B \epsilon \geq \agg_0 + t \right)
\leq 2 \exp \left(
	-\frac{1}{c} \min\left\{ \frac{t^2}{2\eta^4 \agg_2^2} ~,~
	\frac{t}{\eta^2\agg_1} \right\}
\right)
\]
where $\agg_1$ and $\agg_2$ are defined in \eqref{eq:def-all-phi}. 
Taking a union bound over all $A\in\Mset$, we get 
\begin{align*}
\mathbb{P}\left( 
	\sup_{A\in \Mset} \, \epsilon^\sT (\frac{1}{n}X A X^\sT) \epsilon 
	\,\geq\, \agg_0 + t \right)
&\leq 2 \exp \left(
	-\frac{1}{c} \min\left\{ \frac{t^2}{2\eta^4 \agg_2^2} ~,~
	\frac{t}{\eta^2\agg_1} \right\}
\right) \cdot \abs{\Mset} \\
&\leq 2p_0 \cdot \exp \left(
	-\frac{1}{c} \min\left\{ \frac{t^2}{2\eta^4 \agg_2^2} ~,~
	\frac{t}{\eta^2\agg_1} \right\}
	+ \log \frac{\abs{\Mset}}{p_0}
\right). 
\end{align*}

The right-hand side will be bounded by $2p_0$ (as desired in the statement) if the argument to the exponential is non-positive. This provides a lower bound for $t$ which is consistent with the fact that we would like $t$ to be as small as possible in the left-hand side of the above chain of inequalities. Therefore, we choose 
\[
t = \eta^2 \cdot \max\left\{
	\agg_2 \sqrt{2c \log \frac{\abs{\Mset}}{p_0}}
	~,~
	\agg_1 c \log \frac{\abs{\Mset}}{p_0}
\right\}
\]
which establishes the claim. 
\end{proof}

\begin{remark}\label{rem:other-HW}
In proving \autoref{lem:noise-dual-norm-Mset}, we use a variation of the Hanson-Wright inequality given in \autoref{lem0}, from \citep{adamczak2015note}. This result is particularly useful when matrices $A\in \Mset$ are not necessarily positive semidefinite. As an example, see Example 1 in \citep{jalali2017variational}. 
On the other hand, when $\Mset\subset \mathbb{S}_+^p$, other variations of the Hanson-Wright inequality may be used (a tail inequality -- not necessarily a two-sided inequality -- suffices) to establish variations of \autoref{lem:noise-dual-norm-Mset}. These variations may allow for other classes of noise distributions. As an example, working with the Hanson-Wright inequality in \citep{hsu2012tail} requires $\Mset\subset \mathbb{S}_+^p$ but allows for $\epsilon\in\mathbb{R}^n$ to be {\em a subgaussian random vector}; for some $K\geq 0$, $\mathbb{E}\exp\langle \epsilon, u\rangle \leq \exp(K^2 \norm{u}_2^2/2)$ for all $u\in\mathbb{R}^n$. This class neither covers nor is included in the class with the convex concentration property. 
\end{remark}

Finally, let us complement the bound of \autoref{lem:noise-dual-norm-Mset} with an upper bound on $\lambda$. The following bound is well-known but has been provided for completeness. The proof is given in \autoref{app:prediction}.
\begin{lemma}\label{lem:lam-upper-bnd}
Consider measurements of the form $y = X\beta^\star + \epsilon$ and the estimator in \eqref{eq:estimator-kd}. If $\lambda \geq \frac{1}{n}\norm{X^\sT y}^\star$, then $\hbeta=0$.
\end{lemma}

\subsection{Existing Approaches}\label{sec:lambda-bnds-existing}
\citep{jalali2018missing} also leverage the Hanson-Wright inequality in regularized regression where they consider a modification of Lasso for recovery of a sparse transition matrix in a vector autoregressive process with subgaussian noise and incomplete observations. In such problem, the design is constructed through the action of the transition matrix on previous innovations. Therefore, instead of aggregate quantities $\agg_0$, $\agg_1$, and $\agg_2$ here, for the design matrix, they arrive at structural summary quantities for the transition matrix (Section 1.3~in this reference) which allow for quantifying the {\em dependence} within design caused by autoregression. The bounds of \citep{jalali2018missing} in terms of these structural summary quantities can be compared with the bounds in \cite[Theorem 3.3]{melnyk2016estimating} that are agnostic to the model properties. Following a similar line of thought as that of \citep{jalali2018missing}, combined with the general machinery provided in this section, one can derive bounds on the regularization parameter for many correlation scenarios (beyond autoregression) in the design matrix. 

On the other hand, most of the existing literature for bounding the regularization parameter assume both $X$ and $\epsilon$ are drawn from well-known random ensembles for which concentration results exist. Most notably, generic chaining \citep{talagrand2014upper} is used leading to bounds in terms of the Gaussian width (or subgaussian width, sub-exponential width, etc) of the unit norm ball. For example, 
see \citep{banerjee2014estimation,chen2016structured} for certain subgaussian design matrices, 
\citep{sivakumar2015beyond} for results on sub-exponential noise and design, \cite[Theorem 3.3]{melnyk2016estimating} for the case of autoregressive models, and \citep{johnson2016structured} for an active sampling scenario.

Even beyond the random nature of existing results, computing the Gaussian width of a norm ball is not straightforward and requires a case by case consideration; e.g., see \citep{chen2015structured}. General approaches for bounding this Gaussian width include bounding the Gaussian width of {\em all} tangent cones (Lemma 3 in \citep{banerjee2014estimation}) as well as careful partitioning of the extreme points of the norm ball (Lemma 2 in \citep{maurer2014inequality}).

\section{Estimation Error Bounds and the Relative Diameter}\label{sec:est}
Consider the setup of \autoref{sec:concise-var}: a measurement model $y = X\beta^\star + \epsilon$, where $X\in\mathbb{R}^{n\times p}$ is the {\em design} matrix and $\epsilon\in\mathbb{R}^n$ is the noise vector. For any given norm $\nor$, and not only those studied in \autoref{sec:Snorms-def}, we then consider the regularized estimator in \autoref{eq:estimator}. Rather now-well-known analysis of \eqref{eq:estimator} yields {\em estimation error bounds}, namely bounds on $\|\hbeta-\beta^\star\|$ and $\|\hbeta-\beta^\star\|_2$. In this section, we review existing estimation error bounds (e.g., see \citep{wainwright2014structured} for a review) and provide proofs for the sake of completeness. Let us summarize the main ingredients in establishing these bounds:
\begin{itemize}
\item 
Optimality condition for the regularized estimator in \eqref{eq:estimator}, with $\lambda \geq \norm{\frac{2}{n}X^\sT \epsilon}^\star$, yields $v = \hbeta-\beta^\star \in \ErrSet(\prans)$ where 
\begin{align}\label{eq:ErrSet}
	\ErrSet(\prans) &\equiv \bigl\{v 
		:~ \frac{1}{2}\|v\| + \|\beta^\star\| \ge \|\beta^\star+v\| \bigr\}
\end{align}
is in general a non-convex set and hard to characterize.

\item 
The restricted eigenvalue (RE) constant, defined as
\begin{align}\label{eq:def-REc}
\REc(A) = \min_{u\in A\backslash \{0\}}\frac{ \frac{1}{n}\norm{Xu}_2^2 }{ \norm{u}_2^2},
\end{align}
characterizes the effect of $X$ on the error $v$, and when evaluated positive on $\ErrSet(\prans)$ allows for transforming the prediction error bound into estimation error bounds.


\item The restricted norm compatibility constant \citep{negahban2012unified} is defined as 
\begin{align}\label{eq:norm-compat}
	\vpsi(A) = \sup_{u\in A\backslash 0} \frac{\norm{u}}{\norm{u}_2}, 
\end{align}	
and when evaluated on $\ErrSet(\prans)$, allows for relating $\norm{v}$ and $\norm{v}_2$ in establishing estimation error bounds using a prediction error bound and the restricted eigenvalue condition. 
\end{itemize}

\begin{theorem}[Estimation Error]\label{thm:estimation}
	Suppose that the sample covariance $\hSigma\equiv (X^\sT X)/n$ satisfies the RE condition on $\ErrSet$ with constant $\REc>0$. For $\lambda \ge \|\frac{2}{n}X^\sT \epsilon\|^\star$, 
	then, the estimator $\hbeta$ given by~\eqref{eq:estimator} satisfies the bounds
	\begin{align}
		\|\hbeta-\beta^\star\| &\le \frac{3}{\REc}\lambda\vpsi^2\,,\label{eq:square-B}\\
		\|\hbeta-\beta^\star\|_2 &\le \frac{3}{\REc} \lambda \vpsi\,. \label{eq:L2-B}
	\end{align}
where $\vpsi = \vpsi(\ErrSet)$; see \eqref{eq:ErrSet} and \eqref{eq:norm-compat}. 
\end{theorem}
\autoref{thm:estimation} is proved in \autoref{sec:estimation}. 

However, the main point of deviation from the existing standard analysis is the introduction of a new quantity, namely {\em the relative diameter of the norm ball at $\beta^\star$}; see \autoref{def:varphi}. Using this quantity, we define a superset for $\ErrSet(\prans)$, in \autoref{lem:cones}, which allows for bounding all of the above quantities and leads to concrete (as opposed to {\em conceptual}) bounds.

\subsection{Relative Diameter}
Replacing $\ErrSet$ with a more computational-friendly {\em superset of $\ErrSet$}, in computing $\vpsi$ and $\REc$, allows for deriving valid bounds that can be explicitly evaluated. We do so by introducing a new quantity, namely the {\em relative diameter of the dual norm ball with respect to $\beta^\star$}, and by providing \autoref{lem:cones} which replaces $\ErrSet$ with a simple cone defined in terms of the relative diameter. Further elaborations and discussions on the notion of relative diameter are postponed to \autoref{sec:varphi} and \autoref{sec:varphi-insight}.

Before defining our main quantity in \autoref{def:varphi}, let us review some definitions from convex geometry. 
	Let $A$ and $B$ be two non-empty subsets of~$\reals^p$. Define the Hausdorff distance $\dist_H(A,B)$ by
	\begin{align*}
		\dist_H(A,B) \equiv \max\,\{\sup_{a\in A} \dist(a,B),\, \sup_{b\in B} \dist(b,A)\},
	\end{align*}
	where for a given point $a$ and a set $B$, $\dist(a,B) = \inf_{b\in B} \|a-b\|_2$ denotes the distance of point $a$ from set $B$ in $\ell_2$ norm. 
	For a given set $A\subset \reals^p $, the corresponding support function $\sigma_A(v):\reals^p \mapsto \reals$ is defined as 
	$
	\sigma_A(v) \equiv \sup_{a\in A}~ \<a,v\>
	$. 
	Note that $B\subseteq A \subset \reals^p$ if and only if $\sigma_B(v) \le \sigma_A(v)$ for all $v\in \reals^p$. 
The Hausdorff distance can then be defined alternatively as
\begin{align}\label{eq:hausdorff}
	\dist_H(A,B) = \sup_{\|v\|_2\le1 } |\sigma_A(v) - \sigma_B(v)|\,.
\end{align}

\begin{definition}[Relative Diameter]\label{def:varphi}
Given a norm $\nor$ on $\mathbb{R}^p$, denote the unit ball in the dual norm by $\cB^\star\equiv \{z\in\reals^p: \|z\|^\star \le1\}$ and the subdifferential of $\nor$ at $\beta$ by $\partial \norm{\beta}$. 
We define {\em a measure of complexity of $\beta\in\mathbb{R}^p \backslash \{0\}$ with respect to the norm $\nor$} denoted by $\varphi(\beta; \nor)$ as follows, 
\begin{eqnarray}\label{eq:varphi}
	\varphi = \varphi(\beta; \nor)\equiv \dist_H(\cB^\star, \partial\norm{\beta}).
\end{eqnarray}
Furthermore, since $\partial \norm{\beta}$ is a subset of (in fact, a face of) $\Bc^\star$ we have 
\begin{align}\label{eq:varphi-max}
\varphi (\beta; \nor)
= \adjustlimits\max_{z\in \Bc^\star} \min_{g\in \partial\norm{\beta}}~ \norm{z-g}_2.
\end{align}
\end{definition}
	As an example, for the case of $\ell_1$ norm we have $\varphi(\beta; \nor_1) = 2\sqrt{\norm{\beta}_0}$. 
In \autoref{sec:varphi}, we present a few strategies for computing or upper bounding the relative diameter accompanied by detailed computations for a few families of norms in \autoref{sec:varphi-owl} and \autoref{app:varphi}. In \autoref{sec:varphi-insight}, we provide further insights on $\varphi(\beta; \nor)$.

\subsection{New Estimation Bounds}\label{sec:new-bounds}
Recall the error set $\ErrSet = \ErrSet(\pran)$ defined in \eqref{eq:ErrSet}. As it may be seen from the definition, this is generally a non-convex set with a complicated structure. Therefore, it is not in general easy to compute the associated restricted norm compatibility constant $\vpsi(\ErrSet)$ or the restricted eigenvalue constant $\REc(\ErrSet)$ for a given design. Therefore, a reasonable strategy is to find a simpler set to which $\ErrSet$ is a subset. Computing the two aforementioned constants for such a superset of $\ErrSet$ cannot decrease $\vpsi$ and cannot increase $\REc$. Therefore, the prediction error bound of \autoref{lem:oracle} and the estimation error bounds of \autoref{thm:estimation} cannot decrease meaning that we will have new valid error bounds. 

Next, we use the notion of relative diameter to define a computationally-friendly set that covers $\ErrSet$ and replaces it in the computation of $\vpsi$ and $\alpha$.

\begin{lemma}\label{lem:cones}
	Consider the set $\ErrSet(\pran)$ from \eqref{eq:ErrSet} and the cone $\cC(\varphi)$ defined as
\begin{align}\label{eq:cone-varphi}
\cC(\varphi) = \bigl\{v:~ \norm{v} \le 2\varphi \norm{v}_2\bigr\}
\end{align}	 
with $\varphi=\varphi(\pran)$ defined in \autoref{def:varphi}. Then, $\ErrSet(\pran)\subseteq \cC(\varphi)$.
\end{lemma}

In the above, $\varphi = \varphi(\pran)$ is defined based on the Hausdorff distance between the dual norm ball and the subdifferential of the norm at $\beta$. 
For example, for the case of $\ell_1$ norm we have $\varphi = 2\sqrt{\norm{\beta^\star}_0}$ and hence \autoref{thm:estimation}, with $\ErrSet$ replaced by $\cC(2\sqrt{\norm{\beta^\star}_0})$ recovers the classical estimation result on Lasso \citep{buhlmann2011statistics}.

\begin{proof}[Proof of \autoref{lem:cones}]
	For $v\in \ErrSet$, we have
	\begin{align}\label{dummy1}
		\frac{1}{2}\|v\| \le \|v\| + \|\beta^\star\| - \|\beta^\star+v\|\,.
	\end{align}
	By convexity of $\|\cdot\|$ we have 
	\[
	\sup_{w\in \partial\|\beta^\star\|} \<w,v\> \le \|\beta^\star+v\| - \|\beta^\star\|\,.
	\]
	Therefore,
	\begin{align}\label{Hauss1}
		\|\beta^\star\| - \|\beta^\star+v\| + \|v\| 
		\le \sup_{\|z\|^\star\le 1} \<z,v\> - \sup_{w\in \partial\|\beta^\star\|} \<w,v\>\,.
	\end{align}
Recall the notation $\cB^\star$ for the unit ball in the dual norm. 
		We proceed by writing the right-hand side of~\eqref{Hauss1} in terms of support functions:
	\begin{align}
	\begin{split}\label{Hauss2}
		\|\beta^\star\| - \|\beta^\star+v\| + \|v\| 
		&\le \|v\|_2 \left[\sigma_{\cB^\star}\Big(\frac{v}{\|v\|_2}\Big) - \sigma_{\partial\|\beta^\star\|}\Big(\frac{v}{\|v\|_2}\Big) \right]
		\\
		&\stackrel{(a)}{=} \|v\|_2 \left|\sigma_{\cB^\star}\Big(\frac{v}{\|v\|_2}\Big) - \sigma_{\partial\|\beta^\star\|}\Big(\frac{v}{\|v\|_2}\Big) \right|
		\\
		&\stackrel{(b)}{=} \|v\|_2\, \dist_H(\cB^\star, \partial\|\beta^\star\|) = \varphi \|v\|_2\,,
	\end{split}
	\end{align}
	where $(a)$ follows from the characterization of subdifferential \citep{watson1992characterization} as $\partial\|\beta^\star\| = \{w: \<w,\beta^\star\> = \|\beta^\star\|, \, \|w\|^\star = 1 \} \subset \cB^\star$ and the fact that $\sigma_A (\cdot) \leq \sigma_B(\cdot)$ for $A\subseteq B$, and $(b)$ follows from the characterization of Hausdorff distance, given by~\eqref{eq:hausdorff}. 
	By combining~\eqref{dummy1} and \eqref{Hauss2}, we get $\|v\|\le 2\varphi \|v\|_2$, and hence $v\in \cC(\varphi)$. This completes the proof. 
\end{proof}

Recall from above that evaluating different ingredients of the statistical error bounds on a superset of $\ErrSet$ yields valid bounds. As an example, recall the restricted norm compatibility constant defined in \eqref{eq:norm-compat} as $\vpsi(A) = \sup_{u\in A\backslash 0} \frac{\norm{u}}{\norm{u}_2}$. It is then easy to see from \autoref{lem:cones} that 
\begin{align}
\vpsi(\ErrSet(\pran)) \leq \vpsi(\cC(\varphi(\pran))) = 2 \varphi(\pran).
\end{align}
In the sequel, we study the RE condition for a family of subgaussian design matrices where in the proof we leverage \autoref{lem:cones} and compute the RE constant for $\cC(\varphi)$ instead of $\ErrSet$.

\begin{theorem}\label{thm:RE-random}
Consider 
\begin{itemize}
\item A closed scale-invariant set~$\Sc$, spanning $\mathbb{R}^p$, that further satisfies $\Sc \subseteq \{\beta:~ \card(\beta)\leq k\}$, and the corresponding cone $\cC(\varphi)$ for $\varphi=\varphi(\beta^\star;\nor_\Sc)$. 
\item 
A sequence of design matrices $X\in \reals^{n\times p}$, with dimensions $n\to \infty$, $p = p(n)\to \infty$ satisfying the following assumptions, for constant $\lambda_{\min}, \lambda_{\max}, \kappa$ independent of $n$. For each $n$, $\Sigma\in \reals^{p\times p}$ is such that $\lambda_{\min}(\Sigma) \ge c_{\min}>0$ and $\lambda_{\max}(\Sigma)\le c_{\max}<\infty$. 
\item Assume that the rows of $X$ are independent subgaussian random vectors in $\reals^p$ rows with second moment matrix $\Sigma$.
\end{itemize}	
	Then, for any fixed constant $c>0$, the empirical covariance $\hSigma \equiv (X^\sT X)/n$ satisfies the RE condition over $\cC(\varphi)$ for $\REc = \lambda_{\min}/2$, with probability at least $1-2p^{-ck}$, provided that 
	\begin{align}
	n\ge C\lambda_{\min}^{-2} \varphi^4 k\log p, \label{n-condition} 
	\end{align}
	where $C = C(c, \lambda_{\min},\lambda_{\max},\kappa)$. 
\end{theorem}
Proof of \autoref{thm:RE-random} is given in \autoref{app:proof-RE}. We follow a similar approach to that of \citep{loh2012}. However, instead of considering as many atoms as present in the target model, we only consider two atoms which allows for easy generalization to cases beyond sparsity.

\begin{remark}\label{rem:Eq}
For any $\qc>1$, consider 
\begin{align}\label{eq:ErrSet-q}
	\ErrSet^{(\qc)}(\prans) &\equiv \bigl\{v 
		:~ \frac{1}{\qc}\|v\| + \|\beta^\star\| \ge \|\beta^\star+v\| \bigr\},
\end{align}
which for $\qc=2$ yields $\ErrSet^{(2)} = \ErrSet $ defined in \eqref{eq:ErrSet}. Note that $\ErrSet^{(\qc)}$ is the whole space for $0<\qc\leq 1$ which is not of interest in our discussion. 
An easy adaptation of \autoref{lem:cones} yields $\vpsi(\ErrSet^{(\qc)}) \leq \frac{\qc}{\qc-1}\varphi(\pran)$. Notice the complicated dependence of the left-hand side on $\qc$ while the right-hand side's dependence is clear.

Define $\theta = \norm{\frac{1}{n}X^\sT \epsilon}^\star$. Then, for any $\lambda > \theta$ used in \eqref{eq:estimator}, the prediction error bound of \autoref{lem:oracle} and the estimation error bounds of \autoref{thm:estimation} read as
\begin{align*}
	\frac{1}{n}\norm{X(\beta^\star - \hbeta)}_2^2 
	\leq 2\,(\lambda+\theta) \norm{\beta^\star}\,,~	
	\norm{\hbeta-\beta^\star} 
	\leq 2\,\frac{\lambda^2(\lambda+\theta)}{(\lambda-\theta)^2} \Bigl(\frac{\varphi^2}{\REc}\Bigr)\,,~	
	\norm{\hbeta-\beta^\star}_2 
	\leq 2\,\frac{\lambda(\lambda+\theta)}{\lambda-\theta} \Bigl(\frac{\varphi}{\REc}\Bigr) \,,
\end{align*}
where $\REc = \REc(\cC^{(\lambda/\theta)})$. 
Moreover, an adaptation of \autoref{thm:RE-random} yields $\REc = \REc(\cC^{(\lambda/\theta)}) = \lambda_{\min}/2$ for 
\[
n \geq (36C^2 k \log p) (\lambda_{\min}^{-2} \varphi^4) (\frac{\lambda}{\lambda-\theta})^4.
\]
Proof of the above statements is deferred to \autoref{sec:estimation}.

For future reference, we define $\ErrSet^{(\infty)} \equiv \bigl\{v :~ \|\beta^\star\| \ge \|\beta^\star+v\| \bigr\}$ known as the set of descent directions at $\beta$ with respect to $\nor$. The closed convex hull of $\ErrSet^{(\infty)}$ is the tangent cone at $\beta$. We refer to $\ErrSet^{(\infty)}$ as the {\em constrained error set}, as it an important object in the analysis of the Dantzig selector \citep{chatterjee2014generalized,chen2015structured}.
\end{remark}

\section{Computing the Relative Diameter}\label{sec:varphi}
Recall the definition of relative diameter $\varphi(\pran)$ in \autoref{def:varphi}. 
Here, we provide some tools to exactly compute or upper bound $\varphi$. The rest of this section focuses on such computations for a few major classes of norms: ordered weighted $\ell_1$ norms and their dual norms (which are polyhedral norms) as well as doubly-sparse norms and their dual norms.

\subsection{Tools for Computing $\varphi$}\label{sec:varphi-props}
The following is easy to see from the definition.
\begin{lemma}
$\varphi(\beta; \nor)$ is order-0 homogeneous with respect to its first argument and order-1 homogeneous with respect to its second argument.
\end{lemma}

\begin{lemma}\label{lem:varphi-max-ext}
Denote by $\ext(A)$ the set of extreme points of a compact convex set $A$. Then, 
\begin{align}\label{eq:varphi-max-ext}
\varphi(\beta;\nor) 
= \adjustlimits\max_{z\in \ext\Bc^\star} \min_{g\in \partial\norm{\beta}} ~\norm{z-g}_2.
\end{align}
\end{lemma}
\begin{proof}[Proof of \autoref{eq:varphi-max-ext}]
Distance to a convex set is a continuous convex function. Moreover, $\Bc^\star$ is a compact convex set. Therefore, by Bauer's Maximum Principle (e.g., see \citep[Proposition~1.7.8]{schirotzek2007nonsmooth}) a maximizer can be found among the extreme points of $\Bc^\star$. 
\end{proof}
Recall that $\varphi^2(\beta;\nor_1)=4\norm{\beta}_0$ and observe that $\varphi^2(\beta; \nor_2) = 4$, for any $\beta\neq 0$. 
\autoref{lem:varphi-max-ext} provides us with a procedure to compute $\varphi$ for many other common norms:
\begin{enumerate}
\item characterize $\ext(\Bc^\star)$ as well as $\partial \norm{\cdot}$,
\item characterize $\dist(z, \partial \norm{\beta})$ for each $z\in \ext(\Bc^\star)$, possibly making use of any structure in members of $\ext(\Bc^\star)$, 
\item possibly simplify the previous step by ignoring those $z\in \ext(\Bc^\star)$ that can be seen that are sub-optimal in the final maximization over all $z\in \ext(\Bc^\star)$, 
\item take the maximum of all the computed distances $\dist(z, \partial \norm{\beta})$ over $z\in \ext(\Bc^\star)$.
\end{enumerate}
We follow this procedure to exactly compute $\varphi$, 
\begin{itemize}
	\item for weighted $\ell_1$ norms in \autoref{lem:varphi-weighted-l1-exact}, 
	\item for weighted $\ell_\infty$ norms in \autoref{lem:varphi-weighted-linf-exact}, and directly for the $\ell_\infty$ norm in \autoref{lem:varphi-linf-exact}, 
	\item for $\nor_{k\square 1}$ in \autoref{lem:varphi-k-1}.
\end{itemize}
Furthermore, \autoref{lem:varphi-max-ext} allows for simplifying the computation of $\varphi$, when the dual norm is a structure norm; i.e., all of the extreme points of $\Bc^\star$ have the same $\ell_2$ norm, namely $\eta$. Then, since we are only dealing with the extreme points and not all members of $\Bc^\star$ as in the original definition, we get 
\begin{align*}
\varphi^2(\beta;\nor) 
= \eta^2 + \adjustlimits\max_{z\in \ext\Bc^\star} \min_{g\in \partial\|\beta\|} ~\norm{g}_2^2 - 2\<z,g\>.
\end{align*}
For example, the dual to an ordered weighted $\ell_1$ norm is a structure norm; see \autoref{lem:extBst-owl}.

For structure norms (norms whose extreme points are all on the unit sphere), we can simplify $\varphi(\beta; \nor_\Sc^\star)$ as follows. Recall that the orthogonal projection onto a non-convex set, such as $\proj_\Sc(\beta)$, is a set-valued mapping in general. However, in the case of closed scale-invariant sets $\Sc$, \autoref{lem:dual-len-proj} establishes that all of the outputs have the same $\ell_2$ norm. 
\begin{lemma}\label{lem:varphi-dualSnorm}
Given a closed scale-invariant set $\Sc\subset \mathbb{R}^p$, consider the corresponding structure norm $\nor_\Sc$. Then, 
\begin{align}
\varphi(\beta;\nor_\Sc^\star) 
&= \adjustlimits\max_{z} \min_{g} \left\{ \norm{z-g}_2:~
z\in \ext\Bc,~g\in \partial\norm{\beta}_\Sc^\star \right\} \nonumber\\
&= \adjustlimits\max_{z} \min_{g} \left\{ \norm{z-g}_2:~
z\in \Sc \cap \mathbb{S}^{p-1},~
g\in \frac{1}{\norm{\proj_\Sc(\beta)}_2} \conv\left(\proj_\Sc(\beta)\right)
\right\}
\end{align}
where we used \autoref{eq:Snorm-gauge} and \autoref{subdiff_dualnorm_gen}. 
\end{lemma}

\paragraph{Upper-bounding $\varphi$.} In some cases, it is not straightforward to follow the procedure we discussed before for exact computation of $\varphi$. In such cases, we upper bound $\varphi$ instead:
\begin{itemize}
\item Ordered weighted $\ell_1$ norms $\nor_\owl$ in \autoref{lem:varphi-OWL}, implying an upper bound for $\ell_\infty$ norm in \autoref{lem:varphi-linf}, 

\item \autoref{fig:kdnorms} illustrates the doubly-sparse norms and their dual norms. We provide an upper bound for $\nor_{k\square 1}^\star$ in \autoref{lem:varphi-bnd-norm-k-1-dual}. 

\end{itemize}
Here is an upper bounding strategy: 
\begin{lemma}\label{lem:varphi-maxmin}
The max-min inequality gives
\[
\varphi(\beta; \nor) \leq 
\adjustlimits \min_{g\in \partial\norm{\beta}} \max_{z\in \ext\Bc^\star} ~\norm{z-g}_2.
\]
\end{lemma}

In the following, we present the bound for $\varphi$ for ordered weighted $\ell_1$ norms as a sample of results in \autoref{app:varphi}.

\subsection{Ordered Weighted $\ell_1$ Norms}\label{sec:varphi-owl}
Here, we provide bounds on $\varphi$ for a class of norms, namely the ordered weighted $\ell_1$ norms. The main technique is to upper bound \eqref{eq:varphi-max} using the max-min inequality as given in \autoref{lem:varphi-maxmin}.

Given $\beta$, sort $\abs{\beta}$ in descending order to get $\bar\beta$. Given $w_1\geq w_2 \geq \cdots \geq w_p \geq 0$, the ordered weighted $\ell_1$ norm is defined as $\norm{\beta}_\owl = \sum_{i=1}^p w_i \bar\beta_i$. This norm encompasses $\ell_1$, $\ell_\infty$, and OSCAR \citep{bondell2008simultaneous}.

\begin{lemma}\label{lem:varphi-OWL}
Given $\beta\in\mathbb{R}^p$, set $d = \abs{\{ \abs{\beta_i}\neq 0:~ i\in[p]\}}$. Moreover, define $\cG = (\cG_1, \cdots, \cG_d)$ as the partition of $\supp(\bar\beta)$ into $d$ subsets where for any $i,j\in\supp(\bar\beta)$ and any $t\in[d]$: $i,j\in \cG_t$ if and only if $\bar\beta_i = \bar\beta_j$. Then, for~$\norm{\cdot}_\owl$, 
\begin{align*}
\varphi^2(\beta; \nor_\owl) 
~\leq~ \norm{w_\cG}_2^2 + 3\sum_{t=1}^d \frac{1}{\abs{\cG_t}} (\sum_{j\in \cG_t} w_j)^2 
~\leq~ 4\norm{w_\cG}_2^2\,,
\end{align*}
where we abuse the notation with $\cG = \cG_1\cup\cdots\cG_d = \supp(\bar\beta)$. The bounds are achieved with equality for $w=\one$ (the $\ell_1$ norm).
\end{lemma}
Proof of \autoref{lem:varphi-OWL} is given in \autoref{app:OWL}.

\begin{corollary}\label{lem:varphi-linf}
Setting $w$ to the first standard basis vector we get $\norm{\cdot}_\owl = \norm{\cdot}_\infty$. Hence, $\varphi(\beta; \nor_\infty) \leq \sqrt{1+3/t} \leq 2$ where $t = \abs{\{i\in[p]:~ \abs{\beta_i} = \norm{\beta}_\infty\}} \ge 1$. 
\end{corollary}
We next employ \autoref{lem:varphi-max-ext}, to precisely compute $\varphi$ for $\ell_\infty$ norm.
\begin{lemma}\label{lem:varphi-linf-exact}
For the $\ell_\infty$ norm and $\beta\neq 0$, 
\begin{align*}
\varphi^2(\beta;\nor_\infty) = 1 + \frac{1}{\max\{t-1,1/3\}}
\end{align*} 
where $t = \abs{\{i\in[p]:~ \abs{\beta_i} = \norm{\beta}_\infty\}} \ge 1$.
\end{lemma}
Proofs for \autoref{lem:varphi-linf} and \autoref{lem:varphi-linf-exact} are given in \autoref{app:OWL}.

\begin{remark}
In the case of ordered weighted $\ell_1$ norms \citep{zeng2014ordered}, in \autoref{lem:varphi-OWL}, we provide a simple and interpretable bound on $\varphi(\beta;\nor_\owl)$ for any $\beta\in\mathbb{R}^p$. The bound relies on the clustering of values in $\beta$ as well as the sparsity pattern of $\beta$ in interaction with $w$, and is closely connected to the K-means objective for the entries of $\beta$. 

On the other hand, the computations in Theorem 5 and Example 3.2~of \citep{chen2015structured} rely on upper bounding $\nor_\owl$ with $\ell_1$ and $\ell_2$ norm and provide a crude bound on $\vpsi$ for the constrained error set in terms of $\norm{\beta}_0$, $w_1$, and the average of entries of~$w$, as $\frac{2pw_1^2}{\norm{w}_1}\sqrt{s}$ where $s = \card(\beta^\star)$. Note that the constrained error set is contained in $\ErrSet$, hence has a smaller value for~$\vpsi$. 

Since the bound in \cite[Example 3.2]{chen2015structured} is derived through upper bounding with $\ell_1$ norm (which coincides with $\nor_\owl$ for $w=\one_p$), it is easy to construct examples of $w$ for which the bound in \autoref{lem:varphi-OWL} is much better. For example, as an extreme case, consider the $\ell_\infty$ norm corresponding to $w=e_1$. In such case, for $\beta\neq0$, \autoref{lem:varphi-linf} gives $\varphi(\beta; \nor_\infty) \leq \sqrt{1+3/t}\leq 2$, for $t=\abs{\{i\in[p]:~\abs{\beta_i} = \norm{\beta}_\infty\}} \leq \norm{\beta}_0$, while \cite[Example 3.2]{chen2015structured} gives a bound of $(p+1)\sqrt{\norm{\beta}_0}$ for $\vpsi$ evaluated on the constrained error set.
\end{remark}

\section{Insights on Relative Diameter}\label{sec:varphi-insight}
Recall the discussion in the beginning of \autoref{sec:new-bounds} on the complexity of the error set $\ErrSet = \ErrSet(\pran)$, defined in \eqref{eq:ErrSet}, and how finding and working with a computationally-friendly superset of $\ErrSet$ allows for simplifying the computation of the associated restricted norm compatibility constant and the restricted eigenvalue constant for a given design. In the following, we review some of the existing approaches to finding such a superset and provide comparisons with the proposed superset in \autoref{lem:cones}.

\paragraph{When decomposable.}
For example, let us consider the class of norms that satisfy the decomposability condition of \cite[Definition~1]{negahban2012unified}. More specifically, suppose that $A\subseteq \bar A \subseteq \mathbb{R}^p$ and $\bar A^\perp = \{v:~ \langle u,v\rangle=0 ~ \forall u\in \bar A \}$ are such that for all $u\in A$ and all $v\in \bar A^\perp$ we have $\norm{u+v} = \norm{u}+\norm{v}$. This assumption is satisfied by the $\ell_1$ norm and the nuclear norm but is otherwise very restrictive. Relying on such assumption, namely the decomposability of $\nor$ with respect to $(A,\bar A)$, it is easy to show that (e.g., see end of Section 2 in \citep{negahban2012unified}) for $\beta\in A$, 
\[
\ErrSet(\pran) \subset \bigl\{ v:~ \norm{v} \leq 4 \norm{\proj(v; \bar A)}\}, 
\]
which then yields tight prediction and estimation error bounds. 
However, the above strategy cannot be applied to general norms; as easy examples as the $\ell_\infty$ norms or a weighted $\ell_1$ norm.

\paragraph{When the width is all we need.}
As discussed above, the approximation of $\ErrSet$ with a superset is being used to upper bound $\vpsi(\ErrSet)$ and to lower bound $\REc(\ErrSet)$. We are not aware of any proposals in the literature for the former and one of our main contributions lies in the introduction of the relative diameter and the associated superset for $\ErrSet$, provided in \autoref{lem:cones}, that makes both of these tasks possible. However, an alternative strategy has been used in the literature to lower bound $\alpha(\ErrSet)$ through connections to constrained estimators: 
\begin{align}
\hbeta_{\rm D} &\equiv \argmin_\beta \, \bigl\{ \norm{\beta} :~ \norm{ X^\sT(X\beta - y)}^\star \leq \lambda \bigr\} \,, \label{eq:estimator-dantzig} \\
\hbeta_{\rm E} &\equiv \argmin_\beta \, \bigl\{ \norm{\beta} :~ X\beta = y \bigr\} \,, \label{eq:estimator-constrained} \\
\hbeta_{\rm T} &\equiv \argmin_\beta \, \bigl\{ \norm{\beta} :~ \norm{X\beta - y}_2\leq \delta \bigr\} \,, \label{eq:estimator-tube} \\
\hbeta_{\rm N} &\equiv \argmin_\beta \, \bigl\{ \norm{X\beta - y}_2:~ \norm{\beta}\leq \tau \bigr\} \,, \label{eq:estimator-ball} 
\end{align}
where \eqref{eq:estimator-dantzig} is discussed in \citep{chatterjee2014generalized,banerjee2014estimation,chen2015structured,cai2016geometric}, \eqref{eq:estimator-constrained} and \eqref{eq:estimator-tube} are discussed in \citep{chandrasekaran2012convex}, and \eqref{eq:estimator-ball} is discussed in \citep{li2015geometric}, {and the analysis for all of them models the norm ball with its tangent cone at $\beta^\star$ and studies the interaction of the design matrix and the noise with such model (i.e., the tangent cone)}. More specifically, \citep{banerjee2014estimation} shows that the Gaussian width of the regularized error set $\ErrSet(\pran)$ and the constrained error set (namely $\{v:~ \norm{\beta+v} \leq \norm{\beta}\}$, whose closure is the tangent cone at $\beta$) are of the same order, which then allows for providing a sample complexity result to attain a desired RE constant (in the nature of \autoref{thm:RE-random}). See \citep{tropp2015convex} for general sample complexity results, in relation to RE, for independent subgaussian measurements established through tools for bounding a nonnegative empirical process as well as the notion of Gaussian width.

\paragraph{Relative diameter enables required computations.}
Alternatively, in this work, we observe that the error set can be bounded as in \autoref{lem:cones}: 
\[
\ErrSet(\pran) \subset \cC(\varphi) = \bigl\{v:~ \norm{v} \le 2 \norm{v}_2 \cdot \varphi(\pran) \bigr\}.\
\]
where $\varphi$, the relative diameter with respect to $\nor$ at $\beta$, is defined in \autoref{def:varphi}. 
This readily implies $\vpsi(\ErrSet)\leq 2\varphi$. Moreover, as illustrated through \autoref{thm:RE-random}, $\varphi$ and the associated superset also allow for a straightforward lower bounding of the RE constant $\alpha(\ErrSet)$.

\paragraph{Some Remarks.}
\begin{itemize}
\item Let us recall \autoref{lem:cones} implying $\vpsi(\ErrSet) \leq 2\varphi$ where 
\begin{align*}
\vpsi(\ErrSet) &= \sup_v\, \Bigl\{\frac{\norm{v}}{\norm{v}_2} :~ \frac{1}{2}\norm{v}+\norm{\beta} \geq \norm{\beta+v} \Bigr\}, \\
\varphi(\beta;\nor) &= \adjustlimits\sup_{z} \inf_g \Bigl\{ \norm{z-g}_2:~ \norm{z}^\star \leq 1,~ \norm{g}^\star\leq 1,~ \langle g,\beta\rangle =\norm{\beta} \Bigr\} .
\end{align*}
On a high level, the transformation from $\vpsi(\ErrSet)$ to $\varphi$ can be seen as going from a primal quantity to a dual quantity.

\item 
Note that, as clear from the definition of $\varphi$, it is not a local quantity, and as it can be seen from the example in \autoref{fig:maxWL1-quants}, can change with the changes in the norm even though the tangent cone at $\beta$ is being kept the same. This hints on suitability of $\varphi$ in analyzing the regularized problem (while tangent cone is relevant for constrained problems). However, the tangent cone still affects the computation of $\varphi$ through its relation to the subdifferential: the dual to tangent cone is the cone of subdifferential. 

\item 
It is worth mentioning that \citep{chen2015structured} is concerned with the Dantzig selector, not the regularized estimator, and only provides strategies to bound $\vpsi$ for the {\em constrained} error set. 

\item 
Several geometric quantities related to a norm have been studied in the high-dimensional statistics literature. Gaussian width \citep{Gordon88,chandrasekaran2012convex} has been a prominent quantity in linear models. 
See \citep{amelunxen2014living,foygel2014corrupted, jalali2014minimum,banerjee2014estimation,chen2015structured,vaiter2015model, su2016slope, figueiredo2016ordered} for other quantities.

\end{itemize}

\paragraph{An Illustrative Example.} 
Here, we consider a parametrized family of norms and examine the values of $\vpsi(\ErrSet^{(\infty)})$, $ \vpsi(\ErrSet)$, and $\varphi$, to showcase how $\varphi$ remains faithful to the true quantity $\vpsi(\ErrSet)$ as the norm changes, where $\ErrSet^{(\infty)} \equiv \{v:~\norm{\beta+v} \leq \norm{v}\}$; see \autoref{rem:Eq}. 

For any value $\gamma>0$, we consider the norm 
\begin{align*}
\norm{\beta} \equiv \max\bigl\{ 
	\abs{\beta_1}
	+\frac{3}{4}\abs{\beta_2} ~,~
	\frac{\gamma}{\gamma+4}\abs{\beta_1}
	+\frac{9}{10}\abs{\beta_2} ~,~
	\frac{\gamma}{\gamma+5}\abs{\beta_1}
	+\frac{9}{2}\abs{\beta_2} 
	\bigr\}
\end{align*}
in $\mathbb{R}^2$. Considering $\beta = [0,1]^\sT = e_2$, it is easy to see that $\varphi$ has three separate modes; i.e., as $\gamma$ changes, the optimal $z\in\Bc^\star$ jumps among three (distinct) possible choices. 
The subdifferential, and hence the tangent cone, do not change with $\gamma$. However, $\vpsi$ for the tangent cone (equal to $\vpsi(\ErrSet^{(\infty)})$) is not going to be a constant, as the norm changes with $\gamma$. 

\begin{figure}[h]
	\vskip .1in
	\centering
	\includegraphics[width=.6\textwidth]{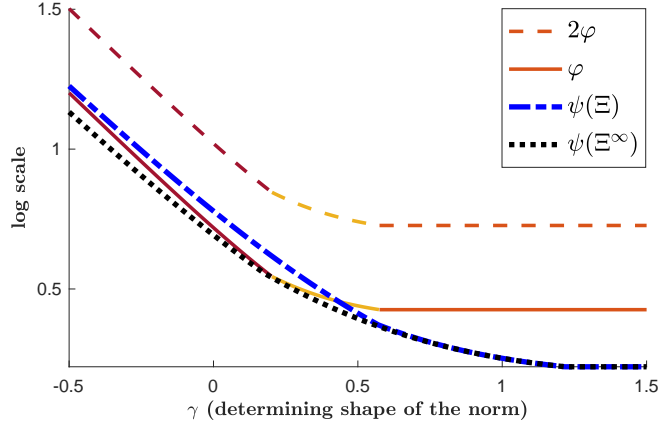}
	\caption{Values of $\varphi(e_2; \nor)$ (solid line with warm colors), $\vpsi = \vpsi(\ErrSet)$ (blue dash-dotted line), and $\vpsi(\ErrSet^{(\infty)})$ (black dotted line), evaluated numerically, for $\beta=[0,1]^\sT$ and different values of $\gamma$. The three colors on the solid line indicate the regimes under which the achieving $z\in\Bc^\star$ is the same. Observe that $\varphi$ closely follows the other two, in all three regimes.}
	\label{fig:maxWL1-quants}
\end{figure}

From \autoref{lem:cones}, we expect $\vpsi(\ErrSet^{(\infty)}) \leq \vpsi(\ErrSet) \leq 2\varphi$. Moreover, \autoref{rem:Eq} establishes $\vpsi(\ErrSet^{(q)}) \leq \frac{q}{q-1}\varphi$ for all $q>1$, which implies $\vpsi(\ErrSet^{(\infty)}) < \varphi$. All of these can be observed in \autoref{fig:maxWL1-quants} as well. As established in \autoref{rem:Eq}, larger values of the regularization constant $\lambda$ allow for basing the analysis on $\ErrSet^{(q)}$ for larger values of $q$, which in turn makes the error bounds in terms of $\varphi$ closer to those in terms of $\ErrSet^{(q)}$.

\paragraph{Comparison over maximum of weighted $\ell_1$ norms.} 
In this experiment, we randomly generate maximum of weighted $\ell_1$ norms and compute and plot $\varphi$, $\vpsi(\ErrSet)$, and $\vpsi(\ErrSet^{(\infty)})$ for them. \autoref{fig:varphi-vpsi-randnorms} provides the results. 
As it can be seen from \autoref{fig:varphi-vpsi-randnorms}, $\varphi$ closely approximates $\vpsi(\ErrSet)$ for most cases. In generating a norm, we first pick a random integer to determine the number of weighted $\ell_1$ norms that are involved. We always include $w=[1,1]^\sT$ (corresponding to the $\ell_1$ norm), and we choose the rest of the weight vectors as random points in the positive orthant to the right and below of $w=[1,1]^\sT$. 

As discussed in the previous experiment, we expect $\vpsi(\ErrSet^{(\infty)}) \leq \vpsi(\ErrSet) \leq 2\varphi$ and $\vpsi(\ErrSet^{(\infty)}) < \varphi$, both of which can be observed in \autoref{fig:varphi-vpsi-randnorms} as well. However, $\varphi$ has a lower bound as $2\,\dist(0,\partial \norm{\beta}) \leq \varphi(\pran)$ which holds generally whenever $\proj(0,\partial \norm{\beta}) \in\partial \norm{\beta}$. 

\begin{figure}[h]
\vskip .1in
\centering
\includegraphics[width=0.6\textwidth]{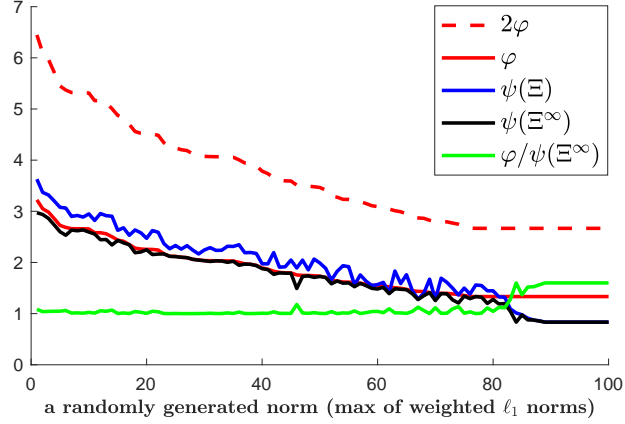}
\caption{For $100$ randomly generate maximum of weighted $\ell_1$ norms, and for $\beta=[0,1]^\sT$, we plot $\varphi$, $2\varphi$, $\vpsi(\ErrSet)$, $\vpsi(\ErrSet^{(\infty)})$, as well as the ratio between the first and the last, which as predicted by \autoref{rem:Eq}, is always above 1. }
\label{fig:varphi-vpsi-randnorms}
\end{figure}

\section{Doubly-Sparse Regularization; Optimization and Statistical Bounds}\label{sec:kd-final}

\subsection{Prediction Error for $\nor_{k\square d}$}\label{sec:PredErr}
Here, we consider the linear measurement model $y = X\beta^\star+\epsilon$, with $X\in \reals^{n\times p}$ the design matrix and $\epsilon\in \reals^n$ a noise vector. We apply the prediction bounds established in~\autoref{sec:concise-var} to the case of doubly-sparse regularized estimator given by~\eqref{eq:estimator-kd}. As a result, we bound $\norm{X(\beta^\star - \hbeta)}_2$ in terms of the $k$ and $d$ used in defining the regularizer $\norm{\cdot}_{k\square d}$, the properties of $\beta^\star$ (number of nonzeros and distinct values), and certain properties of $X$. As we will see, column aggregation in $X$ plays a natural role in the final bound.

\begin{theorem}\label{thm:pred-err-kd}
	Suppose that noise vector $\epsilon$ is zero mean Gaussian vector with covariance matrix $\Sigma \equiv \mathbb{E}[\epsilon \epsilon^\sT]$. Define
	\begin{equation}\label{eq:phi0-1}
	\begin{aligned}
		\phi_0\;&\equiv \sup_{J \subseteq [p]: |J|\le k-d+1} \, \frac{\norm{\Sigma^{1/2} X_{J} \one}_2^2}{n |J|} , \\ 
		\phi_1 \;&\equiv \sup_{J\subseteq[p]: |J|\le k} \frac{\norm{\Sigma^{1/2} X_J}_{\rm op}^2}{n}\,,
 	\end{aligned}\end{equation}
	and for an arbitrary fixed value of $0<p_0<1/2$, let 
\begin{align}\label{eq:phi}	
	\phi\; \equiv\; \frac{1}{\sqrt{n}} \left(d \phi_0 + c \min(d \phi_0, \phi_1) \Big[k \log({2epd}/{k}) + \log (1/p_0)\Big] \right)^{1/2}\,,
\end{align}
	where $c>2$ is the numerical constant in the Hanson-Wright inequality given in \autoref{lem0}. Let $\hbeta$ be obtained from \eqref{eq:estimator-kd} with $\lambda \geq \phi$. Then, with probability at least $1 - 2 p_0$, it satisfies
	\begin{align}\label{prediction-error-kd}
		\frac{1}{n}\norm{X(\beta^\star - \hbeta)}_2^2 \leq 3\lambda \|\beta^\star\|_{k\square d} \,.
	\end{align} 
\end{theorem}
\begin{proof}[Proof of \autoref{thm:pred-err-kd}]
We first apply \autoref{lem:noise-dual-norm-Mset} to the case of doubly-sparse regularization and show that in this case $\agg \le \phi$, where $\phi$ is given by~\eqref{eq:phi}.
The result then follows readily from \autoref{lem:oracle}.

In specializing \autoref{lem:noise-dual-norm-Mset} to the case of doubly-sparse regularization, it is easy to see that $\mathcal{M} = \BD(k,d)$ due to the characterization~\eqref{dual-norm-var}. In addition, by the concentration inequality of Lipschitz function of Gaussian vectors, we have that $\Sigma^{-1/2}\epsilon$ satisfies the convex concentration with constant one. 

Let $\tilde{X}\equiv \Sigma^{1/2} X$ and write
\begin{align*}
\agg_0 \equiv \sup_{A\in \BD(k,d)} \frac{1}{n} \Tr(\tilde{X} A \tilde{X}^\sT ) \le d \times \sup_{J\subseteq[p], |J|\le k-d+1} \frac{\norm{\tilde{X}_J\one}_2^2}{n|J|} = d \phi_0\,,
\end{align*}
where we uses the structure of $A\in \BD(k,d)$, namely it has only a nonzero principle sub-matrix of size $k$. Further, this sub-matrix is block diagonal with $d$ blocks and for a block of size $q$, all of its entries are $1/q$. 

We also have
\begin{align*}
\agg_1 \equiv \sup_{A\in\BD(k,d)} \frac{1}{n} \norm{\tilde{X}A\tilde{X}^\sT}_{\rm op} \le \frac{1}{n} \sup_{A\in \BD(k,d)} \norm{A}_{\rm op}\times \sup_{J\subseteq{p}, |J|\le k} \norm{\tilde{X}_J}_{\rm op}^2 \le \phi_1\,,
\end{align*}
since $\norm{A}_{\rm op}\le 1$, for $A\in \BD(k,d)$. As another bound on $\agg_1$, note that any $A\in \BD(k,d)$ can be written as $A = u_1u_1^\sT+\dotsc+ u_du_d^\sT$, where each $u_i$ has entries $1/\sqrt{|J_i|}$ on a set $J_i\subseteq[p]$, with $|J_i|\le k-d+1$ and zero everywhere else. Hence,
\begin{align*}
 \frac{1}{n} \norm{\tilde{X} A\tilde{X}^\sT}_{\rm op} = \frac{1}{n} \norm{\tilde{X}u_iu_i^\sT \tilde{X}^\sT}_{\rm op} \le \frac{1}{n}\sum_{i=1}^d \norm{\tilde{X}u_i}_2^2 =\frac{1}{n|J_i|}\sum_{i=1}^d \norm{\tilde{X}_{J_i} \one}_2^2 \le d \phi_0\,.
\end{align*}
Combining the above two bounds we obtain $\agg_1 \le \min(d\phi_0,\phi_1)$.

By using \autoref{lem:BD-size}, we have
\begin{align*}
\kappa \equiv \frac{c}{2} \log \frac{|\BD(k,d)|}{p_0} < \frac{c}{2} \Big(k\log(2epd/k) + \log(1/p_0) \Big)\,.
\end{align*}
Finally, we note that 
\begin{align*}
\agg_2 \equiv \sup_{A\in \BD(k,d)} \frac{1}{n}\norm{\tilde{X}A\tilde{X}^\sT}_F \le \sqrt{d} \sup_{A\in \BD(k,d)} \frac{1}{n}\norm{\tilde{X}A\tilde{X}^\sT}_{\rm op} \le \sqrt{d} \agg_1 \,,
\end{align*}
where in the first inequality we used the fact that the matrices in $\BD(k,d)$ are at most of rank $d$. Consequently, $\agg_2< \agg_1\sqrt{\kappa}$. By plugging the above bounds on $\agg_0$, $\agg_1$, $\agg_2$, and $\kappa$ in \autoref{eq:def-all-phi}, we obtain that $\agg\le \phi$, which completes the proof. 
\end{proof}

\subsection{Examples} \label{sec:examples-kd}

\paragraph{Lasso.} Note that for $k=d=1$, the structure norm $\|\cdot\|_{1\square 1}$ becomes exactly the $\ell_1$ norm and the estimator $\hbeta$ in \eqref{eq:estimator-kd} reduces to the Lasso estimator with regularization parameter $\lambda$. We show that \autoref{thm:pred-err-kd} recovers the prediction bound of Lasso~\citep[Corollary 6.1]{buhlmann2011statistics}. 
Suppose that the noise $\epsilon$ has i.i.d.~zero mean Gaussian entries with variance at most $\sigma^2$, and the columns of $X$ are normalized so that each column has $\ell_2$-norm $\sqrt{n}$. Then, $\phi_0= \phi_1 = \sigma^2$. Setting $p_0 = 1/(2ep)$, we get $\phi = (\sigma/\sqrt{n}) (1+ 2c \log(2ep))^{1/2}$. Therefore, with $\lambda = \phi$, the bound~\eqref{prediction-error-kd} simplifies to
\begin{align}\label{eq:pred-lasso}
	\frac{1}{n}\norm{X(\beta^\star - \hbeta)}_2^2 
	\leq 3 \sigma \sqrt{\frac{1+ 2c \log(2ep)}{n}} \|\beta^\star\|_1 
	\lesssim \sigma \sqrt{\frac{\log p}{n} } \|\beta^\star\|_1\,.
\end{align}
We denote the right-hand side of~\eqref{eq:pred-lasso} by $\errL$.
Note that the design matrix $X$ appears in the prediction error bound through the quantities $\phi_0$ and $\phi_1$, which for rare-features are expected to be small.

\paragraph{Gain over Lasso with Doubly-sparse Norms.}
We next want to discuss the gain that the estimator~\eqref{eq:estimator-kd} achieves over Lasso when the true underlying parameter $\beta^\star$ is sparse and takes only a few distinct values. 
\begin{lemma}\label{lem:Gain}
Consider a sequence of design matrices $X\in \reals^{n\times p}$, with dimension $n \to \infty$, and $p = p(n) \to \infty$, satisfying the following assumptions for constants $C_{\max}, C > 0$ independent of $n$. For each $n$, $\Psi\in \reals^{p\times p}$ is such that $$\sigma_{\max} (\Psi)\le C_{\max} < \infty,\quad \sup_{J\subseteq[p], |J|\le k} \frac{1}{|J|} (\one^\sT \Psi_{J,J} \one) \le C_* \le C_{\max}\,.$$ In addition, $X\Psi^{-1/2}$ that has i.i.d.~subgaussian rows, with zero mean and subgaussian norm $\kappa=\|\Psi^{-1/2} x_1\|_{\psi_2}$, and the noise vector $\epsilon\in\reals^n$ has i.i.d.~Gaussian entries with variance at most $\sigma^2$. Then, there exist constants $ c_0, c, C >0$, depending on the subgaussian norm $\kappa$, such that the following holds. With probability at least $1- 2p^{-ck} - 2 p^{-c(k-d+1)}$, the following holds for $\phi_0$ and $\phi_1$ given by~\autoref{eq:phi0-1}:
\begin{align*}
\phi_0 \le C_{*}\sigma^2 \left(1+ C \sqrt{\frac{(k-d+1)\log p}{n}}\right) \,,\quad \quad \phi_1\le C_{\max}\sigma^2 \left(1+ C\sqrt{\frac{k \log p}{n}}\right) \,.
\end{align*}
Consequently, by \autoref{eq:phi}, if $n\ge c_0 k \log p$ we have 
\begin{align*}
\phi\le \tilde{C} \sigma\left[\min(dC_*,C_{\max}) \frac{k}{n} \log\Big(\frac{2epd}{k}\Big) \right]^{1/2}\,,
\end{align*}
for a constant $\tilde{C}>0$.
\end{lemma}
We refer to \autoref{app:prediction} for the proof of \autoref{lem:Gain}. Plugging $\lambda \asymp \phi$ in~\eqref{prediction-error-kd} gives that with probability at least $1 - (pd/k)^{-k}$,
\begin{align}\label{eq:pred-gainDS}
	\frac{1}{n}\norm{X(\beta^\star - \hbeta)}_2^2 
	\lesssim \sigma \sqrt{\frac{k \log(pd/k))}{n}}\, \|\beta^\star\|_{k\square d}\,.
\end{align} 
We denote the right-hand side of~\eqref{eq:pred-gainDS} by $\errDS$.
Comparing the bounds~\eqref{eq:pred-gainDS} with the Lasso prediction bound~\eqref{eq:pred-lasso}, we get
\begin{align}\label{gain-ratioDS}
\frac{\errDS}{\errL} \le C \sqrt{k- \frac{k\log(k/d)}{\log p}}\times \frac{\|\beta^\star\|_{k\square d}}{\|\beta^\star\|_1}\,.
\end{align}
Note that $\|\beta^\star\|_{k\square d}/\|\beta^\star\|_1 \le 1$. To see this, note that the 1-sparse vectors are in $\Sc_{k,d}$ for all $k, d\ge 1$ and hence the $\ell_1$ unit ball is inside $\cB_{\Sc_{k,d}}$, which by definition implies the claim. To show the gain over Lasso (which corresponds to $k = d =1$), we consider the following two cases:

\begin{itemize}
\item Assume that $\max_{i\in \supp(\beta^\star)}|\beta^\star_i| / \min_{i\in \supp(\beta^\star)} |\beta^\star_i|\le c_0$. For $d=1$ and a value of $1\le k\le p$, by using \autoref{lem:norm-k-1}, we have
\begin{align*}
\frac{\norm{\beta^\star}_{k\square 1}}{\norm{\beta^\star}_1} \le \max\Big\{\frac{1}{\sqrt{k}}, c_0\frac{\sqrt{k}}{k^\star}\Big\} = \frac{1}{\sqrt{k}} \max\{1, c_0{k}/{k^\star}\}\,.
\end{align*}
Using this bound in \autoref{gain-ratioDS}, we obtain
\begin{align*}
\frac{\errDS}{\errL} \le C \sqrt{1- \frac{\log k}{\log p}}\times \max\{1, c_0{k}/{k^\star}\}\,.
\end{align*}
Since $k$ can grows as large as $p$, we see that the ratio above can be made arbitrarily small, showcasing the gain over Lasso.

\item Assume the doubly-sparse estimator with $k\ge k^\star$ and $d\ge d^\star$. Then, $\beta^\star \in \Sc_{k,d}$ and hence $\|\beta^\star\|_{k,d} = \|\beta^\star\|_2$ by definition of structured norms; see~\eqref{eq:kd-def}. Therefore, $\|\beta^\star\|_{k\square d}/\|\beta^\star\|_1$ can be made as small as $1/\sqrt{k^*}$ (when $d^\star =1$). Therefore, the bound in \autoref{gain-ratioDS} becomes
\[
\frac{\errDS}{\errL} \le C \sqrt{k- \frac{k\log(k/d)}{\log p}}\times \sqrt{\frac{1}{k^\star}} = C \sqrt{1- \frac{\log(k/d)}{\log p}}\times \sqrt{\frac{k}{k^\star}}\,.
\]

Again, as $k/k^\star\ge 1$ can get arbitrarily close to one, $k$ can grow up to $p$, and $d$ can be as small as one, this ratio can be made arbitrarily small which demonstrates the gain over Lasso in prediction error.
\end{itemize}

\paragraph{The $k$-support norm.}
The $k$-support norm coincides with $\nor_{k\square k}$ and the results of \autoref{sec:PredErr} can be specialized to yield prediction error bounds for the regularized regression with the $k$-support norm. However, in setting $d$ equal to $k$ in \autoref{lem:BD-size}, we can get a tighter bound on the size of the corresponding $\Mset$. More specifically, \autoref{ex:ksupp-Mset} improves the bound $\abs{\Mset}\leq (2ep)^k$ from \autoref{lem:BD-size}, to a bound $\abs{\Mset} \leq (ep/k)^k$. Using this bound and calculating $\phi_0$ and $\phi_1$ in Theorem~\ref{thm:pred-err-kd} for case of $k=d$, we obtain the following corollary which is analogous to \autoref{lem:Gain} for the $k$-support norm regularization:
\begin{corollary}
Consider a sequence of design matrices $X\in \reals^{n\times p}$, with dimension $n \to \infty$, and $p = p(n) \to \infty$, satisfying the following assumptions for constants $C_{\max}, C > 0$ independent of $n$. For each $n$, $\Psi\in \reals^{p\times p}$ is such that $$\sigma_{\max} (\Psi)\le C_{\max} < \infty,\quad \sup_{J\subseteq[p], |J|\le k} \frac{1}{|J|} (\one^\sT \Psi_{J,J} \one) \le C_* \le C_{\max}\,.$$ In addition, $X\Psi^{-1/2}$ that has i.i.d.~subgaussian rows, with zero mean and subgaussian norm $\kappa=\|\Psi^{-1/2} x_1\|_{\psi_2}$, and the noise vector $\epsilon\in\reals^n$ has i.i.d.~Gaussian entries with variance at most $\sigma^2$. 
 
Then, specializing \autoref{lem:Gain} for $k = d$, with probability at least $1- 2p^{-ck} - 2 p^{-c}$, 
\begin{align*}
\phi_0 \le C_{*}\sigma^2 \left(1+ C \sqrt{\frac{\log p}{n}}\right) \,,\quad \quad \phi_1\le C_{\max}\sigma^2 \left(1+ C\sqrt{\frac{k \log p}{n}}\right) \,,
\end{align*}
In addition, by \autoref{eq:phi}, if $n\ge c_0 k \log p$, for some constant $c_0> 0$, we obtain the following bound on $\phi$ for case of $k$-support norm 
\begin{align*}
\phi\le \tilde{C} \sigma\left[\min(kC_*, C_{\max})\frac{k}{n} \log\Big(\frac{ep}{k}\Big) \right]^{1/2}\,,
\end{align*}
for a constant $\tilde{C}>0$.
\end{corollary}
Plugging $\lambda \asymp \phi$ in~\eqref{prediction-error-kd} gives the following prediction bound for the $\|\cdot\|_{k\square k}$ regularized estimator~$\hat{\beta}$:
\begin{align}\label{eq:errkk}
	\frac{1}{n}\norm{X(\beta^\star - \hbeta)}_2^2 
	\lesssim \sigma \sqrt{ \min(kC_*, C_{\max}) \frac{k \log(p/k)}{n}}\, \|\beta^\star\|_{k\square k}\,.
\end{align}

\paragraph{The $\|\cdot\|_{k\square 1}$ norm.} Our next example is the other extreme case, namely $d =1$. We characterize the prediction error for $\norm{\cdot}_{k\square 1}$ regularized estimator in lemma below. The next corollary follows from \autoref{thm:pred-err-kd}. 
\begin{corollary}\label{lem:k1}
Consider a sequence of design matrices $X\in \reals^{n\times p}$, with dimension $n \to \infty$, and $p = p(n) \to \infty$, satisfying the following assumptions for constants $C_{\max}, C > 0$ independent of $n$. For each $n$, $\Psi\in \reals^{p\times p}$ is such that $$\sigma_{\max} (\Psi)\le C_{\max} < \infty,\quad \sup_{J\subseteq[p], |J|\le k} \frac{1}{|J|} (\one^\sT \Psi_{J,J} \one) \le C_* \le C_{\max}\,.$$ In addition, $X\Psi^{-1/2}$ that has i.i.d.~subgaussian rows, with zero mean and subgaussian norm $\kappa=\|\Psi^{-1/2} x_1\|_{\psi_2}$, and the noise vector $\epsilon\in\reals^n$ has i.i.d.~Gaussian entries with variance at most $\sigma^2$. 

There exist constants $C, c_0, c>0$ such that the following holds. Assume $n\ge c_0 k \log p$ and let
\begin{align*}
\phi = C \sigma \sqrt{C_* \frac{k \log(p/k)}{n}} 
\end{align*}
Let $\hbeta$ be obtained from \eqref{eq:estimator-kd} with $d =1$ and $\lambda \geq \phi$. Then, with probability at least $1- 2p^{-ck} - 2 (ep/k)^{-k}$, we have
\begin{align}\label{eq:pred-gain2}
	\frac{1}{n}\norm{X(\beta^\star - \hbeta)}_2^2 
	\lesssim 3\lambda \|\beta^\star\|_{k\square 1}\,.
\end{align} 
\end{corollary}

Using $\lambda \asymp \phi$ in~\eqref{eq:pred-gain2} gives the following prediction bound for the $\|\cdot\|_{k\square 1}$ regularized estimator~$\hat{\beta}$:
\begin{align}\label{eq:pred-gain3}
	\frac{1}{n}\norm{X(\beta^\star - \hbeta)}_2^2 
	\lesssim \sigma \sqrt{ C_{*} \frac{k \log(p/k)}{n}}\, \|\beta^\star\|_{k\square 1}\,.
\end{align} 
To compare with the $\norm{\cdot}_{k\square k}$ regularizer, we denote by $\errkk$ and $\errk1$ the right-hand side of~\eqref{eq:errkk} and \eqref{eq:pred-gain3}. 
We then have 
\begin{align*}
\frac{\errk1}{\errkk} \lesssim \sqrt{\frac{C_*}{\min(kC_*, C_{\max})}} \times \frac{\|\beta^\star\|_{k\square 1}}{\|\beta^\star\|_{k\square k}}\,.
\end{align*}
Now suppose that $\beta^\star \in \Sc_{k,1}$. Then, $\norm{\beta^\star}_{k\square 1} = \|\beta^\star\|_{k\square k} = \|\beta^\star\|_2$ and the above ratio becomes $\sqrt{\frac{C_*}{\min(kC_*, C_{\max})}}$. Recall that $C_*$ was the maximum of the quadratic forms $(\one^\sT \Psi_{J,J} \one)/|J|$, over all subsets $J\subseteq[p]$, with $|J|\le k$. In addition, $C_{\max}$ is the bound on the operator norm of the covariance~$\Psi$. Hence, $C_*\le C_{\max}$ and depending on $\Psi$, this ratio can be made as small as $1/\sqrt{k}$.

\section{Discussions}\label{sec:disc}

\paragraph{Challenges without Decomposability.}\label{sec:new-challenges}
Most of the existing work on norm regularization can be unified under the notion of {\em decomposability}; see \citep{negahban2012unified,candes2013simple,vaiter2015model} for slightly different definitions. While most of the works on statistical analysis for norm regularization, and especially the earlier works, do not explicitly mention decomposability, it is the main proof ingredient; e.g., see Lemma 4.1~in \citep{bickel2009simultaneous} for how decomposability comes into play. 
Therefore, common mechanisms established for analyzing Lasso, nuclear norm regularized estimators, or more generally those with decomposable norms, cannot be used in our case. Therefore, similar to \citep{banerjee2014estimation}, we aim at identifying more general geometric quantities but extend beyond conceptual bounds, introducing computation-friendly quantities.

\paragraph{Algorithms Based on Non-convex Projection.}
Only assuming access to the non-convex projection (onto the set of desired models) can also be used in devising algorithms. For example, Iterative Hard Thresholding algorithms \citep[Section 3]{blumensath2008iterative} \citep{blumensath2011sampling} (projects onto the set of $k$-sparse vectors, namely $\Sc_{k,k}$), \citep[Section 2]{jain2010guaranteed} (projects onto the set of rank-$r$ matrices), \citep{roulet2017iterative} (does K-means which is projection onto the set of $d$-valued models \citep{jalali2013convex}), belong to this class. However, the machinery proposed in this work allows for devising convex regularization functions (norms) which then can be combined with general loss functions and constraints; unlike the specific constrained loss minimization setups required in the aforementioned works.

\paragraph{Gaussian width of the Norm Ball and Unions of Subspaces.}
Lemma 2 of \citep{maurer2014inequality} provides an upper bound for the Gaussian width of a norm ball by splitting the computation over subsets of extreme points. Consider a structure norm associated to a set $\Sc$ which is a finite union of subspaces $\Sc_1,\ldots,\Sc_m$, with dimensions $d_1,\ldots, d_m$, respectively. Then, the Gaussian width of $\Sc_i\cap \mathbb{S}^{p-1}$ is given by 
\[
\omega(\Sc_i\cap \mathbb{S}^{p-1})
= \mathbb{E} _g \,\sup_{z\in \Sc_i\cap \mathbb{S}^{p-1}}\, \langle z, g \rangle 
= \mathbb{E} _g \,\sup_{z\in \Sc_i\cap \mathbb{S}^{p-1}}\, \langle z, \proj(g; \Sc_i) \rangle
= \mathbb{E} _g \norm{\proj(g; \Sc_i)}_2 \leq \sqrt{d_i}
\]
for $g\sim\mathcal{N}(0,I_p)$. Applying Lemma 2 of \citep{maurer2014inequality} to this splitting of $\Sc\cap \mathbb{S}^{p-1}$ yields
\[
\omega(\Bc_\Sc) 
= \omega(\Sc\cap \mathbb{S}^{p-1}) 
\leq \max_{i\in[m]} \sqrt{d_i} + 2 \sqrt{\log m}.
\]
The Gaussian width of the unit norm ball is the quantity used in \citep{banerjee2014estimation,chen2015structured} to bound $\frac{1}{n}\norm{X^\sT \epsilon}_\Sc^\star$ related to the regularization parameter. We instead make use of the Hanson-Wright inequality to get \autoref{lem:noise-dual-norm-Mset}, providing a bound that is deterministic with respect to the design (and not restricted to a few random ensembles of design) and is also sensitive to norm-induced properties of the design. 

\paragraph{Possible Generalizations.}
Our result can be easily extended to regularized loss minimization for smooth loss functions and beyond the least-squares loss. The introduction and characterization of $\varphi$ can also be used beyond the regression setup in this paper; e.g., see \citep{goldstein2018structured} for a possible application domain. 
For least-squares with random design, results of \autoref{lem:noise-dual-norm-Mset} and \autoref{thm:pred-err-kd} can be extended to many more noise distributions, as discussed in \autoref{rem:other-HW} and \autoref{sec:lambda-bnds-existing}, as well as to sub-exponential noise as remarked by \citep[Remark 2.8]{adamczak2015note}.

\section*{Acknowledgements}
Adel Javanmard was partially supported by an Outlier Research in Business (iORB) grant from the USC Marshall School of Business, a Google Faculty Research award and the NSF CAREER Award DMS-1844481. 
Maryam Fazel was supported in part by grants NSF TRIPODS CCF 1740551, ONR N00014-16-1-2789, and NSF CCF-1409836. This work was carried out in part while the authors were visiting the Simons Institute for the Theory of Computing.

\clearpage
\appendix

\section{Proofs: Projection-based Norms}\label{app:Snorm-summary}\label{app:projection}

\begin{proof}[Proof of \autoref{lem:dual-len-proj}]
	Since $\omega_0\in \proj_\Sc(\x)$ and $\Sc$ is scale invariant, we have $\{\lambda\omega_0 :\; \lambda\in\mathbb{R}\}\subset \Sc$ which in turn implies $\proj_{\{\lambda\omega_0 :\; \lambda\in\mathbb{R}\}}(\x) =\omega_0$. Note that projection onto a line is a singleton. Therefore, 
	\begin{equation} \label{eq:dual_val}
\langle{\x},{\omega_0}\rangle = \langle{\proj_{\{\lambda\omega_0 :\; \lambda\in\mathbb{R}\}}(\x)},{\omega_0}\rangle = \norm{\omega_0}_2^2 \,.
\end{equation}
Optimality of projection yields $\norm{\omega_0 - \x}_2 \leq \norm{\norm{\omega_0}_2 \cdot \y - \x}_2$ for all $\y\in\Sc \cap \mathbb{S}^{p-1}$. This, after algebraic manipulations and an application of \eqref{eq:dual_val}, yields
\begin{align*}
\norm{\omega_0}_2 
= \langle \frac{\omega_0}{\norm{\omega_0}_2}, \x \rangle
\geq \sup_{\y\in\Sc \cap \mathbb{S}^{p-1}} \langle \y,\x\rangle 
= \sup_{\y\in\Bc_\Sc} \langle \y,\x\rangle
= \norm{\x}_\Sc^\star .
\end{align*}
Since $\frac{\omega_0}{\norm{\omega_0}_2} \in \Sc\cap \mathbb{S}^{p-1}$, we get equality and the proof is finished. 
\end{proof}

The above also establishes 
$\langle{\x},{\proj_\Sc(\x)}\rangle = \norm{\proj_\Sc(\x)}_2^2= \norm{\proj_\Sc(\x)}_\Sc\, \norm{\x}_\Sc^\star$ which illustrates the pair of achieving vectors in the definition of dual norm. This has been known as the {\em alignment property} in the literature. As a corollary, we get the following. 
\begin{corollary}\label{cor:proj-nonexp}
The projection onto a closed scale-invariant set $\Sc$ is non-expansive; i.e., $\norm{\proj_\Sc(\x)}_2 \leq \norm{\x}_2$ for all $\x\,$.	
\end{corollary}
\begin{proof}
The proof is by expanding $\norm{\proj_\Sc(\x)-\x}_2^2\geq 0$ and using \autoref{lem:dual-len-proj}. 

Alternatively, since $\nor_\Sc\geq \nor_2$, we get $\nor_\Sc^\star \geq \nor_2$, which also establishes the claim. 
\end{proof} 
Note that while projection onto convex sets is always non-expansive, projection onto general non-convex sets can be expansive. However, the distance to a general set is still non-expansive (e.g., see \citep[Proposition 2.4.1, page 50]{clarke1990optimization}). 
This should not be confused with the Kolmogorov criterion for projection onto a {\em convex} set $\cC$, i.e., $\y = \proj_\cC(\x)$ if and only if $\y\in\cC$ and $\langle{\z-\y},{\x-\y}\rangle \leq 0$ for all $\z\in\cC\,$, since we are interested in projection onto a non-convex set $\Sc\,$.

\begin{proof}[Proof of \autoref{subdiff_dualnorm_gen}]
Consider the following characterization of the subdifferential \citep{watson1992characterization},
\begin{align*}
\partial \norm{\x}_\Sc^\star
&= \left\{ \g :\; \langle{\g},{\x}\rangle = \norm{\x}_\Sc^\star\,,\; \norm{\g}_\Sc \leq 1 \right\}	\\
&= \left\{ \g :\; \langle{\g},{\x}\rangle = \norm{\proj_\Sc(\x)}_2 \,,\; \norm{\g}_\Sc \leq 1 \right\} \,.
\end{align*}
Using the results of \autoref{lem:dual-len-proj}, one can check that any $\y\in\proj_\Sc(\x)/ \norm{\proj_\Sc(\x)}_2$ satisfies the definition of subgradients. Since subdifferential is a convex set, we get a one-sided inclusion; i.e.,~$\supseteq$.
Next, notice that the squared dual norm can be written as
\begin{align}\label{eq:dum0sub}
(\norm{\x}_\Sc^\star)^2 
= \norm{\beta}_2^2 - \norm{\beta - \proj(\beta; \Sc)}_2^2
= \norm{\beta}_2^2 - \inf_{\theta\in \Sc} \norm{\beta-\theta}_2^2
= \sup_{\y\in\Sc}\; 2\langle{\x},{\y}\rangle - \norm{\y}_2^2
\end{align}
where the inner function, say $f(\x,\y)\,$, is continuous, linear in $\x$, and concave quadratic in $\y$. 

On the other hand, by \autoref{cor:proj-nonexp}, we have $\norm{\beta}_\Sc^\star \leq \norm{\beta}_2$. Moreover, $\nor_\Sc^\star$ is continuous and Lipschitz. Therefore, there exists a neighborhood $U \ni \beta$ such that for every $u\in U$ 
\[
(\norm{u}_\Sc^\star)^2 
= \sup_{\y}\bigl\{ 2\langle{u},{\y}\rangle - \norm{\y}_2^2 :~ \y\in\Sc,~ \norm{\y}_2 \leq 2\norm{\x}_\Sc^\star\bigr\}.
\]
Note that the constraint set (indexing $\theta$) is compact as $\Sc$ is assumed to be a closed set. 
Moreover, it is clear from the second equality in \eqref{eq:dum0sub} that the optimal solutions to the above parametric minimization are the members of $\proj(u; \Sc)$. All in all, the above parametric minimization satisfies the requirements of Theorem 3 in \citep{yu2012differentiability} and we get equality for the subdifferential. 
\end{proof}

\begin{proof}[Proof of \autoref{lem:inv-proj}]
Consider the orthogonal projection mapping as $\proj(\beta;A) = \Argmin_{u\in A} \norm{\beta-u}_2^2 = \Argmin_{u\in A} \sum_{i=1}^p (u_i-\beta_i)^2$. Consider $\theta\in \proj(\beta;A)$. 

If $A$ is invariant with respect to sign flips, $\theta\in A $ implies $\abs{\theta} \circ \sign(\beta) \in A$, where $\sign(0)$ can be chosen as either $+1$ or $-1$. By optimality, $\norm{\theta-\beta}_2^2 \leq \norm{\abs{\theta} \circ \sign(\beta) - \beta}_2^2$ which implies $\sum_{i=1}^p \theta_i\beta_i = \langle \theta, \beta\rangle \geq \langle \abs{\theta}, \abs{\beta}\rangle = \sum_{i=1}^p \abs{\theta_i\beta_i} \geq \sum_{i=1}^p \theta_i\beta_i$. Therefore, all inequalities hold with equality implying $\theta_i\beta_i \geq 0$ for all $i\in[p]$. 

If $A$ is invariant under a permutation of the entries, $\theta \in A$ implies $\pi_{\beta}^{-1}(\pi_\theta(\theta))\in A$ where $\pi_u$ is any permutation for which $\pi_u(u)$ is sorted in non-increasing order. Optimality implies $\norm{\theta-\beta}_2^2 \leq \norm{\pi_\beta^{-1}(\pi_\theta(\theta)) - \beta}_2^2 = \norm{\pi_\theta(\theta) -\pi_\beta(\beta)}_2^2$. This implies $\langle \theta, \beta \rangle \geq \langle \pi_\theta(\theta), \pi_\beta(\beta)\rangle$. The reverse inequality also holds as a result of the rearrangement inequality. Therefore, $\langle \pi_\theta(\theta), \pi_\beta(\beta)\rangle = \langle \theta, \beta \rangle$. Consider any $i,j\in[p]$ for which $\beta_i>\beta_j$. If $\theta_i<\theta_j$, define $\tilde \theta$ with all entries the same as $\theta$ except for $\tilde \theta_i = \theta_j$ and $\tilde \theta_j= \theta_i$. Then, $\langle \pi_{\tilde\theta}(\tilde\theta), \pi_\beta(\beta) \rangle \geq \langle \tilde \theta, \beta\rangle > \langle \theta,\beta \rangle = \langle \pi_\theta(\theta),\pi_\beta(\beta)\rangle $ while the first and last terms are equal. This is a contradiction which implies that the claim should hold true.
\end{proof}

\begin{proof}[Proof of \autoref{lem:proj_S_SB2}]
Consider any $\y \in \proj_{\Sc\cap\mathbb{S}^{p-1}}(\x)$ and any $\z \in \proj_{\Sc}(\x)\,$. The optimality of $\z$ gives 
$\norm{\z}_2^2 - 2\langle{\x},{\z}\rangle \leq \norm{\norm{\z}_2\y}_2^2 - 2\langle{\x},{\norm{\z}_2\y}\rangle$ 
and the optimality of $\y$ gives $\langle{\x},{\y}\rangle \geq \langle{\x},{\z/\norm{\z}_2}\rangle$. Combining these two inequalities proves $\langle{\x},{\y}\rangle = \langle{\x},{\z/\norm{\z}_2}\rangle = \norm{\z}_2$ (last equality uses \eqref{eq:dual_val}) which illustrates that $\norm{\z}_2\y \in \proj_\Sc(\x)$ and $\z/\norm{\z}_2\in \proj_{\Sc\cap\mathbb{S}^{p-1}}(\x)$. 
In other words, given $\y \in \proj_{\Ac_\Sc}(\x)$, we have $\langle{\x},{\y}\rangle \y \in \proj_\Sc(\x) \,$.
\end{proof}
Here is another explanation: since $\Sc$ is scale-invariant, one can first find the direction of projection on $\Sc$ and later find the correct scaling as
\begin{align*}
	\proj_\Sc(\x) 
	&= \arg\left\{\min_{\y\in\Sc}\, \norm{\x-\y}_2^2 \right\} \\
	&= \arg\left\{\min_{\y\in\Sc}\, \norm{\y}_2^2 - 2\langle{\x},{\y}\rangle \right\} \\
	&= \arg\left\{\min_{\tau\geq0}\, \left(\tau^2 - 2\tau \max_{\y\in\Sc\cap\mathbb{S}^{p-1}}\, \langle{\x},{\y}\rangle \right)\right\} 
\end{align*}
which shows that for finding the direction of $\proj_\Sc(\x)$ it suffices to project onto $\Sc\cap\mathbb{S}^{p-1}\,$.
Yet another explanation comes from the result that says the dual norm is equal to the largest inner product with atoms. Hence, combining this with \autoref{lem:dual-len-proj}, we get
\begin{equation}
\max_{\y\in\Sc\cap\mathbb{S}^{p-1}} \langle{\x},{\y}\rangle=\norm{\x}_\Sc^\star= \langle{\x},{\proj_\Sc(\x)/\norm{\proj_\Sc(\x)}_2}\rangle,
\end{equation}
which proves our result.

\section{Proofs: The $(k\square d)$-norm}\label{app:kd-norm}

\begin{proof}[Proof of \autoref{lem:monot}]
For the first statement, we prove the more general version and then apply it to $P=I-e_ie_i^\sT$. Consider $\theta\in\proj(\beta;\Sc)$ and assume $P\beta=\beta$, $P=P^\sT=P^2$, and $Pu\in\Sc$ for all $u\in \Sc$. Then, $P\theta\in\Sc$ and optimality implies $\norm{P\theta-\beta}_2^2\geq \norm{\theta-\beta}_2^2$ which is equivalent to $\norm{(I-P)\theta}_2^2=0$ and in turn to $P\theta=\theta$. 

For the second statement, we prove the more general statement and then apply it to $A = I-e_ie_i^\sT + e_ie_j^\sT$ and $B=I-e_je_j^\sT + e_je_i^\sT$. Consider $\theta\in\proj(\beta;\Sc)$. By the assumption, $A\theta\in\Sc$. Therefore, optimality of $\theta$ implies $\norm{A\theta-A\beta}_2^2 = \norm{A\theta-\beta}_2^2 \geq \norm{\theta-\beta}_2^2$ which in turn implies $(\theta-\beta)^\sT(A^\sT A - I)(\theta-\beta)\geq 0$. A similar argument establishes $(\theta-\beta)^\sT(B^\sT B - I)(\theta-\beta)\geq 0$. Adding up the two inequalities we get $0=0$ and hence all the inequalities so far have to hold with equality, implying that $A\theta$ and $B\theta$ are also optimal; i.e., $A\theta,B\theta\in\proj(\beta; \Sc)$. 
\end{proof}

\begin{proof}[Proof of \autoref{lem:proj-cardk}]
Suppose $\abs{\beta_i} < \bar\beta_k$ and $\theta_i\neq 0$ for some $\theta\in\proj(\beta;\Sc)$ and some $i\in[p]$. Therefore, there exists $j\in[p]$ for which $\abs{\beta_j}\geq \bar\beta_k$ and $\theta_j=0$; otherwise, $\card(\theta)>k$. Consider a new vector $\tilde\theta$ with all entries equal to those of $\theta$ except for $\tilde\theta_i=0$ and $\tilde\theta_j = \abs{\theta_i}\sign(\beta_j)$. Then, 
$\dist^2(\theta;\Sc) - \dist^2(\tilde\theta;\Sc) 
=(\theta_i-\beta_i)^2+(\theta_j-\beta_j)^2-(\tilde\theta_j-\beta_j)^2-(\tilde\theta_i-\beta_i)^2 
=(\theta_i-\beta_i)^2+\beta_j^2-(\abs{\theta_i}-\abs{\beta_j})^2-\beta_i^2 
=2\abs{\theta_i}(\abs{\beta_j}-\beta_i \sign(\theta_i)) 
\geq 2\abs{\theta_i}(\abs{\beta_j}-\abs{\beta_i})>0$. 
This contradicts the optimality of $\theta$. Therefore, the claim is established.
\end{proof}

\begin{proof}[Proof of \autoref{lem:proj-bar-Skd}]
Observe that $\Sc_{k,d}$ satisfies all of the assumptions in \autoref{lem:inv-proj} and \autoref{lem:monot}. 
Consider $\theta\in\proj(\bar\beta; \Sc_{k,d})$. By \autoref{lem:monot}, if two entries of $\bar\beta$ are equal, the same entries in $\theta$ are going to be equal. Therefore, while there might be several options for $\pi$, $\pi^{-1}(\theta)$ is unique for all such $\pi$. Moreover, if $\bar\beta_i=0$, \autoref{lem:monot} implies $\theta_i=0$. Therefore, while there might be ambiguities in choosing $\sign(\beta)$ over its zero entries, the solution to $\pi^{-1}(\theta) \circ \sign(\beta)$ will be unique given a fixed $\theta\in\proj(\bar\beta; \Sc_{k,d})$. Therefore, we can fix a choice for $\sign(\beta)$ and a choice for $\pi$, which in turn makes $\theta\mapsto \pi^{-1}(\theta) \circ \sign(\beta)$ well-defined and {\em invertible}. 
Therefore, observe that 
$\norm{\beta - \pi^{-1}(\theta) \circ \sign(\beta)}_2 
=\norm{\abs{\beta} - \pi^{-1}(\theta)}_2 
=\norm{\pi(\abs{\beta}) - \theta}_2
=\norm{\bar\beta - \theta}_2$. This, in conjunction with sign and permutation invariance of $\Sc_{k,d}$ establishes the optimality of $\pi^{-1}(\theta) \circ \sign(\beta)$. 

On the other hand, consider $\gamma\in \proj(\beta;\Sc_{k,d})$ and define $\theta = \pi(\gamma \circ \sign(\beta))$. Again: 1) $\gamma$ will be zero off the support of $\beta$, hence $\gamma \circ \sign(\beta)$ is well-defined, 2) $\gamma \circ \beta \geq 0$ by \autoref{lem:inv-proj}, hence $\gamma \circ \sign(\beta)\geq 0$, and 3) $\gamma \circ \sign(\beta)$ will not have different entries where $\abs{\beta}$ has equal entries, therefore $\pi(\gamma \circ \sign(\beta))$ is well-defined. It remains to show that such vector is a projection of $\bar\beta$. Similar to the above, observe that $\norm{\gamma -\beta }_2 = \norm{\pi(\gamma \circ \sign(\beta)) - \pi(\beta \circ \sign(\beta))}_2 = \norm{\pi(\gamma \circ \sign(\beta)) - \bar\beta}_2$, which establishes optimality.
\end{proof}

\begin{proof}[Proof of \autoref{lem:proj-Skd}]
Preliminary observation: By \autoref{lem:monot}, $\supp(\theta) \subseteq \supp(\beta)$. By \autoref{lem:inv-proj}, $\theta \circ \beta \geq 0$ which together with the support inclusion result completely determines the sign of nonzeros in $\theta$; sign of any nonzero $\theta_i$ is $\sign(\beta_i)$. Since $\Sc_{k,d}$ is both sign and permutation invariant, and the sign of nonzero entries of the projections are determined, we will adjust the sign whenever we swap entries. 

By \autoref{lem:proj-cardk}, any projection $\theta\in\proj(\beta;\Sc_{k,d})$ will have $\card(\theta)\leq k$ and $S=\supp(\theta)\subseteq \{i:~ \abs{\beta_i}\geq \bar\beta_k\}$. 
Therefore, $\norm{\theta-\beta}_2^2 = \norm{\beta_{S^c}}_2^2 + \norm{\theta- \beta_{S}}_2^2$. Consider $A_k = \{i:\abs{\beta_i} = \bar\beta_k\}$. Then, by \autoref{lem:monot}, there exists a projection $\tilde\theta$ which takes a single absolute value over $A_k$. Therefore, the indices in $S\cap A_k$ can be re-assigned arbitrarily (with appropriate sign adjustment) within $A_k$ without changing the distance. This validates the first step of our procedure. 

Let us restrict the space to any set $A$ with $\supp(\theta)\subseteq A\subseteq \{i:~ \abs{\beta_i}\geq \bar\beta_k\}$ and $\abs{A}=k$. Observe that $\norm{\theta-\beta}_2^2 = \norm{\beta_{A^c}}_2^2 + \norm{\theta- \beta_{A}}_2^2$. Optimality of $\theta$ implies that $\theta_A$ has at most $d$ distinct absolute values and $\theta_A$ is closest to $\beta_A$ among all such vectors. This indeed is equivalent to $\theta_A = \proj(\beta_A; \Sc_{k,d})$ where $\Sc_{k,d}\subseteq \mathbb{R}^k$ here. This validates the second step of our procedure. 
\end{proof}

\begin{proof}{An alternative proof for \autoref{lem:proj-Skd}} 
We can work with $\bar\beta$ to simplify the presentation. Therefore, assume $\beta=\bar\beta$ for the rest of this proof. 
We follow the procedure discussed in proof of \autoref{lem:proj-Skd-combinat-rep} (given next) but instead keep track of the optimal solutions, rather than the optimal value, to characterize the projection itself. 

It can be seen from the reformulations in \eqref{eq:pf-dum1} that we can first project onto $\Sc_{k,k}$, namely the set of $k$-sparse vectors. This leads to zeroing out all entries except the $k$ with largest absolute values (corresponding to the first $k$ entries of $\bar\theta$).

The procedure resulting in \eqref{eq:dumm2}, as discussed, is a K-means procedure into $d$ groups. 

Finally, in comparing \eqref{eq:dumm3} and a similar expression for $\theta$, as in \eqref{eq:Snorm-dual-dist}, we can put the centers back into their original positions (before turning $\theta$ to $\bar\theta$), with the corresponding sign, to get the final result. This is the consequence of optimality in conjunction with the fact that minimal distance is achieved when two vectors have the same sign pattern and pattern of absolute values (rearrangement inequality.)
\end{proof}

\begin{proof}[Proof of \autoref{lem:proj-Skd-combinat-rep}]
Denote the optimal solution (the projection) with $\gamma^\star \in \Sc_{k,d}$. \autoref{lem:proj-bar-Skd} allows for computing $\proj(\theta; \Sc_{k,d})$ from $\proj(\bar{\theta}; \Sc_{k,d})$ and the sign and order patterns in $\theta$. In projecting $\bar{\theta}$, the optimal $\bar{\gamma}^\star$ will be nonnegative and sorted. Therefore, from \eqref{eq:Snorm-dual-dist}, the projection can be expressed as 
\begin{align}
&\norm{ \proj(\bar{\theta}; \Sc_{k,d}) }_2^2 = \norm{\bar{\theta}}_2^2 - \min \bigl\{ \norm{\bar{\theta}-\gamma}_2^2:~\gamma\in \Sc_{k,d} \bigr\} \label{eq:dumm3}\\
&= \norm{\bar{\theta}}_2^2 - \min \bigl\{ \norm{\bar{\theta}-\gamma}_2^2:~\gamma\in \Sc_{k,d} ,~ \gamma_1\geq \cdots \geq \gamma_k \geq 0 ,~
	\gamma_{k+1} = \ldots = \gamma_p = 0 \bigr\} \\ 
&= \norm{\bar{\theta}_{1:k}}_2^2 - \min \bigl\{ \norm{\bar{\theta}_{1:k}-\gamma_{1:k}}_2^2:~\gamma\in \Sc_{k,d} ,~ 	
	\gamma_1\geq \cdots \geq \gamma_k \geq 0,
	\gamma_{k+1} = \ldots = \gamma_p = 0 \bigr\}.
\label{eq:pf-dum1}
\end{align}
Considering the definition of $\Sc_{k,d}$ in~\eqref{eq:def-Skd}, observe that the last minimization is indeed a K-means clustering problem. Since $\gamma_{1:k}$ can only take $d$ distinct values, we can turn the optimization problem into choosing the optimal partition of entries and then assign the optimal value to each partition separately. This yields 
\begin{align}
&\norm{ \proj(\bar{\theta}; \Sc_{k,d}) }_2^2 \nonumber\\
&= \norm{\bar{\theta}_{1:k}}_2^2 - \min \bigl\{ 
	\sum_{i=1}^d \norm{\bar{\theta}_{\int_i}- \frac{1}{\abs{\int_i}}\one\one^\sT \bar{\theta}_{\int_i} }_2^2
	:~ (\int_1,\cdots,\int_d)\in\ptnn(k,d) \bigr\} \nonumber\\
&= \max \bigl\{ \sum_{i=1}^d \frac{1}{\abs{\int_i}} (\one^\sT \bar{\theta}_{\int_i})^2 :~ (\int_1,\cdots,\int_d)\in\ptnn(k,d) \bigr\} \label{eq:dumm2}
\end{align}
as claimed. 
\end{proof}

\begin{proof}[Proof of \autoref{lem:BD-size}]
Consider partitioning $[p]$ into $d+1$ groups, $d$ of which having a total size of $k$. This can be done by first selecting $k$ out $p$ elements and then partitioning these $k$ elements into $d$ groups. We can then get an upper bound by allowing for empty groups; hence, $|\BD(k,d)| = 2^k|\ptn(k,d)| 
		\leq 2^k\binom{p}{k} d^k 
		\leq (\frac{2epd}{k})^k$. 
\end{proof}

\begin{proof}[Proof of \autoref{lem:projS-qcqp}]
Consider a dynamic programming formulation of the 1-dimensional K-means clustering similar to \citep{wang2011ckmeans}. However, modify the formulation to align with the quantity of interest in \autoref{lem:proj-Skd-combinat-rep}; namely $\sum_{i=1}^d \frac{1}{\abs{\int_i}} (\one^\sT \bar{\theta}_{\int_i})^2$ which is to be maximized. 

More specifically, consider 
\begin{align*}
\norm{ \proj(\bar{\theta}; \Sc_{k,d}) }_2^2 
&= \max \bigl\{ \sum_{i=1}^d \frac{1}{\abs{\int_i}} (\one^\sT \bar{\theta}_{\int_i})^2 :~ (\int_1,\cdots,\int_d)\in\ptnn(k,d) \bigr\} \nonumber \\
&= \min\{t: \sum_{i=1}^d \frac{1}{\abs{\int_i}} (\one^\sT \bar{\theta}_{\int_i})^2 \leq t ~~ \text{for all } (\int_1,\cdots,\int_d)\in\ptnn(k,d)\} .
\end{align*}
Define
\begin{align*}
\tilde\nu_{s,e} = \max\{ \sum_{i=1}^e \frac{1}{\abs{\int_i}} (\one^\sT \bar{\theta}_{\int_i})^2 : (\int_1,\cdots,\int_d)\in\ptnn(s,e)\} 
\end{align*}
as the optimal {\em cost-to-go values} and observe that they satisfy the following, 
\begin{align*}
\tilde\nu_{s,e} = \max_{e\leq m \leq s} \{
\tilde\nu_{m-1,e-1} + \frac{1}{s-m+1}\abs{\bar\theta_{[m,s]}}_1^2
\}.
\end{align*}
The above notation can be turned into inequalities (as in the QCQP) which finishes the proof. Observe that the optimal values for $\nu_{s,e}$ in the QCQP are equal to the values for $\tilde\nu_{s,e}$.
\end{proof}

\begin{proof}[Proof of \autoref{lem:proj-dual-ball-qcqp}]
Observe that 
\begin{align*}
\proj(\bar{\theta}; \cB^\star) 
&= \argmin_u \bigl\{ \norm{u-\bar{\theta}}_2^2:~ \norm{u}_\sq^2 \leq 1 \bigr\} \\
&= \argmin_u \bigl\{ \norm{u-\bar{\theta}}_2^2:~ \norm{u}_\sq^2 \leq 1,~ 
	 u_1\geq \cdots \geq u_p \geq 0 \bigr\} \\
&\stackrel{(a)}{=} \argmin_u \Bigl\{ \norm{u-\bar{\theta}}_2^2:~ 
	 u_1\geq \cdots \geq u_p \geq 0 ,~ \\	 
&\qquad \qquad \qquad \min_{\{\nu_{m,e}\}} \Bigl\{ \nu_{k,d}:~
\frac{1}{s-m+1} ( \one^\sT u_{[m,s]})^2 \leq \nu_{s,e} - \nu_{m-1,e-1} ~~ \forall (e,m,s)\in \T(k,d)\Bigr\} \leq 1	 
	 \Bigr\} \\	 
&=	\argmin_u \min_{\{\nu_{m,e}\}}\bigl\{ \norm{u-\bar{\theta}}_2^2:~ 
	 u_1\geq \cdots \geq u_p \geq 0 ,~\\
&\qquad \qquad \qquad ~	 \nu_{k,d} \leq 1,~
	 \frac{1}{s-m+1} ( \one^\sT u_{[m,s]})^2 \leq \nu_{s,e} - \nu_{m-1,e-1} ~~ \forall (e,m,s)\in \T(k,d)
	 \bigr\}
\end{align*}
where in $(a)$ we uses the fact that the variable $u$ is sorted and we plugged in the representation for $\norm{u}_\sq^2$ given in \autoref{lem:projS-qcqp}. 
\end{proof}

\section{Proofs: Prediction Error} \label{app:prediction}

	\begin{proof}[Proof of \autoref{lem:oracle}]
		By optimality of $\hbeta$ we can write $\frac{1}{2n} \|y- X\hbeta\|_2^2 + \lambda \|\hbeta\| \le \frac{1}{2n} \|y- X\beta^\star\|_2^2 + \lambda \|\beta^\star\|$. By plugging in $y = X\beta^\star + \varepsilon$ and after some algebraic calculation, we get
		\[
		\frac{1}{2n}\|X(\hbeta-\beta^\star)\|_2^2 + \lambda \|\hbeta\| \le \lambda \|\beta^\star\| + \frac{1}{n}\epsilon^\sT X (\hbeta - \beta^\star)\,.
		\]
		By the choice of $\lambda$, this implies that
		\begin{align}
			&\frac{1}{2n}\|X(\hbeta-\beta^\star)\|_2^2 \\ &\le \frac{1}{n}\epsilon^\sT X (\hbeta - \beta^\star) + \lambda \|\beta^\star\| - \lambda \|\hbeta\| \nonumber \\
			&\le \|\frac{1}{n}\epsilon^\sT X\|^\star \|\hbeta - \beta^\star\| + \lambda \|\beta^\star\| - \lambda \|\hbeta\| \nonumber \\
			&\le \frac{\lambda}{2} \|\hbeta - \beta^\star\| + \lambda \|\beta^\star\| - \lambda \|\hbeta\|\label{eq:rearrange}\\
			&\le \frac{1}{2} \left(\|\hbeta\| + \|\beta^\star\| \right) + \lambda \|\beta^\star\| - \lambda \|\hbeta\|\nonumber\\
			&\le \frac{3}{2} \lambda \|\beta^\star\|\,,\nonumber
		\end{align}
		where we use the triangle inequality in the penultimate step. This concludes the proof.
	\end{proof}

\begin{proof}[Proof of \autoref{lem:dualvar-UoS}]
The positive semidefiniteness assumption on $M_i$, for $i\in[p]$, makes $\beta^\sT M_i \beta$ a convex function in $\beta$ and therefore $f$ is convex. Moreover, $f(a\beta) = a^2f(\beta)$ for any $a\in\mathbb{R}$. Therefore, \cite[Lemma 3.5]{jalali2017variational} establishes that $\sqrt{f}$ is a semi-norm. 

Next, observe that $f(\beta) = \sup\{ \beta^\sT M \beta:~ M\in\conv(\Mset) \}$ as the objective is linear in $M$. Therefore, if there exists a positive definite matrix in $\conv(\Mset)$ then $f$ is strongly convex. Then, \cite[Lemma 3.5]{jalali2017variational} establishes that $\sqrt{f}$ is a norm. 

Suppose, for each $i\in[m]$, $M_i$ is an orthogonal projector; i.e., there exists an orthonormal matrix $U_i\in\mathbb{R}^{p\times d_i}$ for some $d_i\in[p]$ where $M_i = U_iU_i^\sT$. Then, for the compact set $A = \{ \theta:~ \langle \beta, \theta \rangle\leq \sqrt{f(\beta)} \}$ and $\sigma_A(\beta) = \sup_{\theta\in A}\langle \beta,\theta\rangle$, which denotes the support function for the set $A$, we have 
\[
\sigma_{A}(\beta) 
= \sqrt{f(\beta)} 
= \max_{i\in[m]} \norm{U_i^\sT\beta}_2
= \max \bigl\{ \langle \beta, \theta \rangle:~ \theta = U_iw,~ w\in\mathbb{S}^{d_i-1},~i\in[m] \bigr\}
= \sigma_B(\beta)
\]
where $B = \bigcup_{i\in[m]}\{U_iw:~w\in\mathbb{S}^{d_i-1},~i\in[m]\}$ is a compact set. By the above equality of support functions for the two closed sets $A$ and $B$, we have $\conv(A) = \conv(B)$. On the other hand, $B$ being a subset of $\mathbb{S}^{p-1}$ implies $B = \ext(\conv(B))$. Therefore, $\ext(\conv(A)) = B$. Observe that $A$ is the dual norm ball for $\sqrt{f}$. Moreover, $B = \Sc\cap \mathbb{S}^{p-1}$ for the given set $\Sc$. Piecing all these together, we establish the claim. 
\end{proof}

\begin{proof}[Proof of \autoref{lem:lam-upper-bnd}]
From the assumption, observe that 
\begin{align*}
\lambda 
\geq \frac{1}{n}\norm{X^\sT y}^\star
= \frac{1}{n}\sup_{\beta\neq 0} \frac{\beta^\sT X^\sT y}{\norm{\beta}} 
\geq \frac{1}{n}\sup_{\beta\neq 0} \frac{\beta^\sT X^\sT y - \frac{1}{2}\norm{X\beta}_2^2 }{\norm{\beta}}
\end{align*}
which after a rearrangement yields 
\[
\frac{1}{2n}\norm{X\beta-y}_2^2 +\lambda \norm{\beta}\geq \frac{1}{2n}\norm{y}_2^2
\]
for all $\beta\neq 0$. This establishes the optimality of $\hbeta=0$. 
\end{proof}

\begin{proof}[Proof of \autoref{lem:Gain}]
We first bound $\phi_1$. Fix a subset $J\subset [p]$, with $|J|\le k$. Using the concentration bound for singular values of matrices with i.i.d.~subgaussian rows (see e.g.~\citep[Equation (5.26)]{vershynin2012}), we get
\[
\|\frac{1}{n} X^\sT_J X_J - \Psi_{J,J}\|\le C\sqrt{\frac{k\log p}{n}}C_{\max}\,,
\]
with probability at least $1 - 2p^{-ck}$, where $c=c_\kappa$ and $C= C_\kappa$ depend on the subgaussian norm $\kappa$. By choosing $C$ large enough, we can make constant $c>0$ sufficiently large. The claim for $\phi_1$ then follows by union bounding over all 
subsets $J\subseteq[p]$, with $|J|\le k$.

We next bound $\phi_0$. For a random variable $Z$, denote by $\|Z\|_{\psi_1}$ and $\|Z\|_{\psi_2}$ the sub exponential and subgaussian norms of $Z$, respectively. For a random vector $Z$, these norms are defined as $\|Z\|_{\psi_1} = \sup\{\|Z^\sT v\|_{\psi_1}:\, \|v\|_2 = 1\}$ and $\|Z\|_{\psi_2} = \sup\{\|Z^\sT v\|_{\psi_2}:\, \|v\|_2 = 1\}$. 

Fix a subset $J\subset [p]$, with $|J|\le k-d+1$ and define $Z\equiv \frac{1}{\sqrt{|J|}}(X_J\one)$. We then have $\|Z_i\|_{\psi_2} = \frac{1}{\sqrt{|J|}} \|X_{i,J} \one\|_{\psi_2} \le \frac{1}{\sqrt{|J|}}\|(\Psi_{J,J})^{1/2} \one\|_2\times \|X_{i,J} (\Psi_{J,J})^{-1/2}\|_{\psi_2} \le \frac{\kappa}{\sqrt{|J|}} \|(\Psi_{J,J})^{1/2} \one\|_2$. Then, by \citep[Lemma (5.14)]{vershynin2012} we have $\|Z_i^2\|_{\psi_1}\le 2\|Z_i\|_{\psi_2}^2\le (2\kappa^2/|J|) (\one^\sT \Psi_{J,J}\one) \le 2\kappa^2 C_*$.
We also have
\begin{align*}
\mathbb{E}[Z_i^2] = \frac{1}{|J|}\mathbb{E}[\one^\sT X_{i,J}^\sT X_{i,J} \one] = \frac{1}{|J|} (\one^\sT \Psi_{J,J} \one) \le C_*\,.
\end{align*}

 Employing concentration tail bound for sub-exponential random variables, 
see e.g.~\citep[Corollary 5.17]{vershynin2012}, we obtain 
\[
\frac{1}{n}\|Z\|_2^2 \le C_* + 2\kappa^2 C_* C \sqrt{\frac{(k-d+1) \log p}{n}}\,,
\]
with probability at least $1 - 2p^{-c(k-d+1)}$, for some constant $c>0$ (depending on constants $C, C_*, \kappa>0$). By choosing $C$ large enough, we can make constant $c>0$ sufficiently large. Recall the definition of $\phi_0$, given by~\autoref{eq:phi0-1}, specialized to i.i.d.~noise entries:
\[
\phi_0\;\equiv \sup_{J \subseteq [p]: |J|\le k-d+1} \, \frac{\sigma^2 \norm{X_{J} \one}_2^2}{n |J|}\,.
\]

The claim on the $\phi_0$ bound follows by union bounding over all subsets $J\subseteq[p]$, with $|J|\le k-d+1$.
\end{proof}

\section{Proofs: Estimation Error}\label{sec:estimation}

\begin{proof}[Proof of \autoref{thm:estimation}]
	By optimality of $\hbeta$ we have
	\[
	\frac{1}{2n} \|X\hbeta-y\|_2^2+ \lambda \|\hbeta\| \le \frac{1}{2n} \|X\beta^\star-y\|_2^2+ \lambda \|\beta^\star\|.
	\]
	By rearranging the terms we get
	\[
		\frac{1}{2n}\|X(\hbeta-\beta^\star)\|_2^2 \le 
		\frac{1}{n} \langle X^\sT \epsilon, \hbeta-\beta^\star\rangle 
		+ \lambda \|\beta^\star\| - \lambda \|\hbeta\|\,,		
	\] 
	and using the choice of $\lambda$ and the Cauchy-Schwarz inequality (similar to~\eqref{eq:rearrange}), we have
	\begin{align}\label{eq:opt-estimation-tri}
		\frac{1}{2n}\|Xv\|_2^2 
		\le \frac{\lambda}{2} \|v\| + \lambda \|\beta^\star\| - \lambda \|\beta^\star+v\| 
		\leq \frac{3}{2}\lambda \norm{v}\,,
	\end{align}
	where $v = \hbeta-\beta^\star$ and we used the triangle inequality to get the last bound. 
	As a consequence, $v\in \ErrSet$, where $\ErrSet$ is given by~\eqref{eq:ErrSet}. Define 
	\begin{align}
	\rat(\ErrSet) \equiv \sup_{u\in \ErrSet} \;\frac{\norm{u}^2}{\frac{1}{n}\norm{Xu}_2^2}.
	\end{align}
Since $v\in\ErrSet$, the definition of $\rat=\rat(\ErrSet)$ implies 
\begin{align}\label{eq:pf-est-norm}
\norm{v} 
\leq \frac{\frac{1}{n}\norm{Xv}_2^2}{\lambda \norm{v}}\lambda \rat 
\leq 3\lambda \rat 
\end{align}
where the second inequality is an application of \eqref{eq:opt-estimation-tri}. Next, to bound $\norm{v}_2$, recall the definition of the restricted eigenvalue constant $\REc=\REc(\ErrSet)$ from \eqref{eq:def-REc} and observe that 
\begin{align}\label{eq:pf-est-L2}
\REc \norm{v}_2^2 \leq \frac{1}{n}\norm{Xv}_2^2 
\stackrel{(a)}{\leq} 3\lambda \norm{v}
\stackrel{(b)}{\leq} 9\lambda^2 \rat
\end{align}
where $(a)$ is due to \eqref{eq:opt-estimation-tri} and $(b)$ is by \eqref{eq:pf-est-norm}. 

Next, observe that 
\begin{align*}
	\rat(\ErrSet) 
	= \sup_{u\in \ErrSet} \;\frac{\norm{u}^2}{\norm{u}_2^2} \frac{\norm{u}_2^2}{\frac{1}{n}\norm{Xu}_2^2}
	\leq ( \sup_{u\in \ErrSet} \;\frac{\norm{u}^2}{\norm{u}_2^2}) \cdot (\sup_{u\in \ErrSet} \; \frac{\norm{u}_2^2}{\frac{1}{n}\norm{Xu}_2^2})
	\leq \frac{\vpsi^2(\ErrSet)}{\REc(\ErrSet)} \,,
\end{align*}
which together with \eqref{eq:pf-est-norm} and \eqref{eq:pf-est-L2} establishes the desired bounds. 
\end{proof}

\begin{proof}[An Alternative Proof of \autoref{thm:estimation}; with slightly worse constants]
	By optimality of $\hbeta$ we have
	\[
	\frac{1}{2n} \|X\hbeta-y\|_2^2+ \lambda \|\hbeta\| \le \frac{1}{2n} \|X\beta^\star-y\|_2^2+ \lambda \|\beta^\star\|.
	\]
	By rearranging the terms we get
	\[
		\frac{1}{2n}\|X(\hbeta-\beta^\star)\|_2^2 \le 
		\frac{1}{n} \langle X^\sT \epsilon, \hbeta-\beta^\star\rangle 
		+ \lambda \|\beta^\star\| - \lambda \|\hbeta\|\,,		
	\] 
	and using the choice of $\lambda$ and the Cauchy-Schwarz inequality (similar to~\eqref{eq:rearrange}), we have
	\begin{align}\label{eq:opt-estimation}
		\frac{1}{2n}\|Xv\|_2^2 \le \frac{\lambda}{2} \|v\| + \lambda \|\beta^\star\| - \lambda \|\beta^\star+v\|\,,
	\end{align}
	where $v = \hbeta-\beta^\star$. 
	As a consequence, $v\in \ErrSet$, where $\ErrSet$ is given by~\eqref{eq:ErrSet}. 
From the definition of the restricted norm compatibility constant in \eqref{eq:norm-compat}, we get $\norm{v} \leq \vpsi \norm{v}_2$ for any $v\in \ErrSet$ and for $\vpsi = \vpsi(\beta^\star;\nor)$. 	
	Therefore, by RE condition on $\ErrSet$ for $\hSigma$ we obtain
	\begin{align}	\label{eq:dum-ap}
	\frac{1}{n}\|Xv\|_2^2 \ge \REc \|v\|_2^2 \ge \frac{\REc}{\vpsi^2} \|v\|^2\,.
	\end{align}
	In addition, by using triangle inequality in~\eqref{eq:opt-estimation}, we have $1/(2n)\|Xv\|_2^2\le (3/2) \lambda \|v\|$ and so
	\begin{align}\label{eq:opt-estimation2}
		\frac{1}{n}\|Xv\|_2^2 + \lambda \|v\| &\le 4\lambda \|v\|
		\stackrel{(a)}{\le} \frac{4\lambda\vpsi}{\sqrt{n\REc}} \|Xv\|_2
		\stackrel{(b)}{\le} \frac{1}{2n} \|Xv\|_2^2 + \frac{8}{\REc}\lambda^2\vpsi^2
		\,.
	\end{align}
	where $(a)$ is by \eqref{eq:dum-ap} and $(b)$ holds true because $(\frac{1}{\sqrt{n}}\norm{Xv}_2 - \frac{4\lambda \vpsi}{\sqrt{\REc}} )^2\geq 0$. 
	Therefore,
	\begin{align}\label{eq:opt-estimation3}
		\frac{1}{2n}\|Xv\|_2^2 + \lambda \|v\|\le \frac{8}{\REc}\lambda^2\vpsi^2\,.
	\end{align}
	This implies $\|v\|\le 8\lambda\vpsi^2/\REc$, which proves claim~\eqref{eq:square-B}. 
	
	To prove claim~\eqref{eq:L2-B}, we again apply the RE condition to~\eqref{eq:opt-estimation3} and write
	\[
	\REc\|v\|_2^2\le \frac{1}{n} \|Xv\|_2^2\le \frac{1}{n} \|Xv\|_2^2 + 2\lambda\|v\| \le \frac{16}{\REc} \lambda^2 \vpsi^2\,, 
	\] 
	which gives the desired result.
\end{proof}

\begin{proof}[Proof of \autoref{rem:Eq}]
For any $\qc>1$, consider 
\begin{align}\label{eq:ErrSet-q}
	\ErrSet^{(\qc)}(\prans) &\equiv \bigl\{v 
		:~ \frac{1}{\qc}\|v\| + \|\beta^\star\| \ge \|\beta^\star+v\| \bigr\},
\end{align}
which for $\qc=2$ yields $\ErrSet^{(2)} = \ErrSet $ defined in \eqref{eq:ErrSet}. Note that $\ErrSet^{(\qc)}$ is the whole space for $0<\qc\leq 1$ which is not of interest in our discussion. 
\begin{itemize}
\item An easy adaptation of \autoref{lem:cones} yields
\[
\ErrSet^{(\qc)} \subseteq \cC^{(\qc)} \equiv \bigl\{v:~ \norm{v} \leq \frac{\qc}{\qc-1}\cdot \varphi(\pran) \cdot \norm{v}_2\bigr\}
\]
and implies $\vpsi(\ErrSet^{(\qc)}) \leq \frac{\qc}{\qc-1}\varphi(\pran)$; for any $q>1$. 
\item If $\lambda \geq \ql \norm{\frac{1}{n}X^\sT \epsilon}^\star$ for some $\ql>1$, then the prediction error bound of \autoref{lem:oracle} reads as $\frac{1}{n}\norm{X(\beta^\star - \hbeta)}_2^2 \leq 2(1+\frac{1}{\ql})\lambda \norm{\beta^\star}$, and the estimation error bounds of \autoref{thm:estimation} read as: \eqref{eq:square-B} reads as $\norm{\hbeta-\beta^\star} \leq 2(1+\frac{1}{\ql})\lambda \vpsi^2/\REc$ and \eqref{eq:L2-B} reads as $\norm{\hbeta-\beta^\star}_2 \leq 2(1+\frac{1}{\ql})\lambda \vpsi/\REc$ where $\vpsi = \vpsi(\ErrSet^{(\ql)})$ and $\REc = \REc(\ErrSet^{(\ql)})$. 	

\item 
Combining the above two items, for $\qc=\ql$, we get 
\begin{align*}
	\frac{1}{n}\norm{X(\beta^\star - \hbeta)}_2^2 
	&\leq 2(1+\frac{1}{\ql})\lambda \norm{\beta^\star}
	\\
	\norm{\hbeta-\beta^\star} 
	&\leq 2(1+\frac{1}{\ql})\frac{\lambda}{\REc} (\frac{\ql}{\ql-1})^2 \varphi^2
	= (\frac{2\varphi^2}{\REc})\cdot \frac{\ql(\ql+1)}{(\ql-1)^2} \lambda 
	\\
	\norm{\hbeta-\beta^\star}_2 
	&\leq 2(1+\frac{1}{\ql})\frac{\lambda}{\REc}\frac{\ql}{\ql-1}\varphi
	= (\frac{2\varphi}{\REc})\cdot \frac{\ql+1}{\ql-1} \lambda
\end{align*}
where we now use $\REc = \REc(\cC^{(\ql)}) \geq \REc(\ErrSet^{(\ql)})$. 
\item 
Observe that we can use any $\ql\in (1, \frac{\lambda}{\theta}]$ we wish in our analysis. 
Define $\theta = \norm{\frac{1}{n}X^\sT \epsilon}^\star$. It is easy to see that among all $\ql\in (1, \frac{\lambda}{\theta}]$, largest $\ql$ minimizes all three bounds (ignoring the dependence of $\REc = \REc(\cC^{(\ql)})$ on $\ql$) and $\REc(\cC^{(\lambda/\theta)})\geq \REc(\cC^{(\ql)})$ for any $\ql\in (1, \frac{\lambda}{\theta}]$. Therefore, plugging $\ql=\frac{\lambda}{\theta}$ we get 
\begin{align*}
	\frac{1}{n}\norm{X(\beta^\star - \hbeta)}_2^2 
	&\leq 2(\lambda+\theta) \norm{\beta^\star}
	\\
	\norm{\hbeta-\beta^\star} 
	&\leq (\frac{2\varphi^2}{\REc})\cdot \frac{\lambda^2(\lambda+\theta)}{(\lambda-\theta)^2} 
	\\
	\norm{\hbeta-\beta^\star}_2 
	&\leq (\frac{2\varphi}{\REc})\cdot \frac{\lambda(\lambda+\theta)}{\lambda-\theta} 
\end{align*}
for any $\lambda > \theta$ used in \eqref{eq:estimator}, where $\REc = \REc(\cC^{(\lambda/\theta)})$.

\item an adaptation of \autoref{thm:RE-random} yields $\REc = \REc(\cC^{(\lambda/\theta)}) = \lambda_{\min}/2$ for 
\[
n \geq (C^2 k \log p) \cdot (\lambda_{\min}^{-1}\varphi^2 \cdot 6 (\frac{\ql}{\ql-1})^2 )^2
= (36C^2 k \log p) (\lambda_{\min}^{-2} \varphi^4) (\frac{\lambda}{\lambda-\theta})^4.
\]
\end{itemize}
\end{proof}

\section{Proofs: RE for Subgaussian Designs}\label{app:proof-RE}
\begin{proof}[Proof of \autoref{thm:RE-random}]
	By triangle inequality we have
	\[
	v^\sT \hSigma v \ge v^\sT \Sigma v - |v^\sT (\Sigma - \hSigma) v| \ge \lambda_{\min} \|v\|_2^2 - |v^\sT (\Sigma - \hSigma) v|\,.
	\] 
	
	Let $\Gamma\equiv \Sigma - \hSigma$. Using the above inequality, it suffices to show that 
	\begin{align}\label{claim0}
		|v^\sT \Gamma v| \le \frac{1}{2}\lambda_{\min}\|v\|_2^2\,, \quad \text{for all } v\in \cC(\varphi)\,.
	\end{align}
	
	\begin{lemma}\label{lem:quad-upper}
	Consider a closed scale-invariant set $\Sc$ that span $\mathbb{R}^p$ as well as the corresponding structure norm $\nor_\Sc$. 
		For a given matrix $\Gamma\in\mathbb{R}^{p\times p}$, and a given value $\delta>0$, 
		suppose the following holds for all $v\in (\Sc \oplus \Sc)$, 
		\begin{align}\label{eq:quad-upper-v-2s}
			\abs{v^\sT \Gamma v} \leq \delta \|v\|_2^2.
		\end{align}
		Then, we have the following for all $v\in\mathbb{R}^p$, 
		\begin{align}
			\abs{v^\sT \Gamma v} \leq 3\delta \norm{v}_\Sc^2.
		\end{align}
	\end{lemma}
	We follow a similar approach to proof of Lemma 12 in \citep{loh2012}. 
	\begin{proof}[Proof of \autoref{lem:quad-upper}]
		By definition of $\nor_\Sc$, there exists a set of $\alpha_i>0$ and $v_i\in \Sc$ for which ${v}=\sum_i \alpha_i v_i$, $\sum_i \alpha_i=1$ and $\norm{v_i}_2 \leq \|{v}\|$. Observe that, 
		\[
		v^\sT \Gamma v = \sum_{i,j} \alpha_i \alpha_j v_i^\sT \Gamma v_j. 
		\]
		Moreover, each of $v_i$, $v_j$, and $(v_i+v_j)/2$, are in $ (\Sc \oplus \Sc)$. Therefore, using~\eqref{eq:quad-upper-v-2s} we get
		\begin{align*}
			\abs{v_i^\sT \Gamma v_j} &\leq \frac{1}{2}\abs{(v_i+v_j)^\sT \Gamma (v_i+v_j)} + \frac{1}{2}\abs{v_i^\sT \Gamma v_i} + \frac{1}{2} \abs{v_j^\sT \Gamma v_j}\\
			&\leq \frac{\delta}{2} \|v_i+v_j\|_2^2 + \frac{\delta}{2} \|v_i\|_2^2 + \frac{\delta}{2} \|v_j\|_2^2\\
			&\le 3 \delta\|{v}\|^2 .
		\end{align*}
		Since $\alpha_i$'s define a convex combination, we get $\abs{{v}^\sT \Gamma {v}} \leq 3\delta\|{v}\|^2$. 
	\end{proof}
	
	By virtue of \autoref{lem:quad-upper} and definition of $\cC(\varphi)$, in order to prove Claim~\eqref{claim0} it suffices to show that~\eqref{eq:quad-upper-v-2s} holds for $\delta = {\lambda_{\min}}/{(24\varphi^2)}$.
	
	\begin{lemma}\label{lem:concentration}
		Under the assumptions of \autoref{thm:RE-random}, for any constants $c_0, c_1>0$, there exists $C = C(\lambda_{\max},\lambda_{\min},\kappa,c_0, c_1)$ such that
		\begin{align*}
		\max\left\{\|(\hSigma - \Sigma)_{A,A}\|_2:\,\; A\subseteq[p],\, |A| \le c_0 k \right\} \le C\sqrt{\frac{k\log p}{n}}\,,
		\end{align*}
		with probability at least $1-2p^{-c_1k}$.
	\end{lemma}
	Fix an arbitrary $v\in \Sc\oplus\Sc$. Then, by our assumption that $\Sc\subseteq\{\beta:~\card(\beta)\leq k\}$, $v$ is $2k$ sparse. Denote by $A$ the support of $v$. Then, by employing \autoref{lem:concentration} with $c_0 =2$, we have
	\begin{align*}
		|v^\sT \Gamma v| = |v_A^\sT \Gamma_{A,A} v_{A}|\le \|\Gamma_{A,A}\|_2 \|v_A\|_2^2 
		\le C\sqrt{\frac{k\log p}{n}} \|v\|_2^2 \,,
	\end{align*} 
	with probability at least $1-2p^{-c_1k}$. Hence, for $n\ge (24C/\lambda_{\min})^2 \varphi^4 k\log p$, we obtain \eqref{eq:quad-upper-v-2s} for $\delta = {\lambda_{\min}}/{(24\varphi^2)}$. This completes the proof.
\end{proof}

\begin{proof}[Proof of \autoref{lem:concentration}]
	Fix $A\subseteq[p]$, with $|A|\le c_0 k$ and let $X_A\in \reals^{n\times |A|}$ be the sub-matrix containing columns of $X$ that are in the set $A$. We can write $\hSigma_{A,A} = (X_A^\sT X_A)/n$ and $\Sigma_{A,A} = \mathbb{E}(X_A^\sT X_A)$. By employing the result of Remark 5.40~in~\citep{vershynin2012}, for every $t\ge 0$, with probability at least $1-2 e^{-ct^2}$ the following holds:
	\[
	\|\hSigma_{A,A}-\Sigma_{A,A}\|_{2} \le \max\{\delta , \delta^2\}\,, \quad \text{ where } \delta = C\sqrt{\frac{k}{n}} + \frac{t}{\sqrt{n}}\,,
	\]
	where $C=C(\kappa, c_0)$ and $c= c({\kappa}) > 0$ depend only on the subgaussian norms of the rows of $X$ and constant~$c_0$. Choosing $t = \sqrt{\tilde{c} k \log p}$, and using \autoref{n-condition}, we get that with probability at least $1 - 2p^{c\tilde{c}k}$, 
	\[
	\|\hSigma_{A,A}-\Sigma_{A,A}\|_2\le (C+\sqrt{\tilde{c}}) \sqrt{\frac{k\log p}{n}}\,.
	\]
	We next define $\cF \equiv \{A\subseteq[p]: \, |A|\le c_0 k\}$. Note that $|\cF|\le p^{c_0 k}$. The proof is completed by taking union bound over all sets in $\cF$ and choosing $\tilde{c} = (c_0+c_1)/c$.
\end{proof}

\section{Computing $\varphi$ for Different Families of Norms}\label{app:varphi}
In each section below, we provide a characterization for the subdifferential and for the dual norm, and compute or upper bound $\varphi$. 

\subsection{Auxiliary Lemmas}
\begin{lemma}\label{lem:dist-ei-simplex}
For $\simplex_p = \{u\geq \zero_p:~ \one^\sT u = 1\}$ we have $\dist^2(-e_i, \simplex_p ) = 1+1/\max\{p-1,1/3\}$. 
\end{lemma}
\begin{proof}[Proof of \autoref{lem:dist-ei-simplex}] 
The case of $p=1$ is easy to verify; hence assume $p\geq 2$. Since the projection is unique, we provide a candidate and verify its optimality. In fact, we claim that $\proj(-e_i;\simplex_p) = b = \frac{1}{p-1}(\one_p - e_i)$. By Kolmogorov criteria, for this to be the projection, we need $\<-e_i - b, u-b\>\leq 0$ for any $u\in \simplex_p$; which can be easily verified. This establishes the claim. 
\end{proof}

\begin{lemma}\label{lem:dist-weighted-simplex}
For $w\in\mathbb{R}_{++}^p$, $A = \{u\geq 0:~ \langle w,u\rangle = 1\}$, and any $i\in[p]$, we have 
\[
\dist^2( -\frac{1}{w_i}e_i , A ) = 
\begin{cases}
 \frac{4}{w_i^2}& p=1, \\
\frac{1}{w_i^2} + \frac{1}{\norm{w}_2^2 - w_i^2}& p > 1,~ 2w_i^2 \leq \norm{w}_2^2, \\
\frac{4}{\norm{w}_2^2}& p > 1,~ 2w_i^2 \geq \norm{w}_2^2, \\
\end{cases}
\] 
where $e_i$ is the $i$-th standard basis vector. 
\end{lemma}
\begin{proof}[Proof of \autoref{lem:dist-weighted-simplex}] 
The case of $p=1$ is easy to verify; hence assume $p\geq 2$. Since the projection is unique, we provide a candidate and verify its optimality. There are two cases (illustrated in \autoref{fig:dist-weighted-simplex}): 
\begin{itemize}
\item If $2w_i^2 \leq \norm{w}_2^2$ then we claim that 
$\proj(-\frac{1}{w_i}e_i;A) = b = \frac{1}{\norm{w}_2^2 - w_i^2} (w - w_ie_i)$. To prove the claim, consider any $u\in A$ and observe that, 
\[
\langle -\frac{1}{w_i}e_i - b, u-b \rangle
= \norm{b}_2^2 - \langle b,u\rangle - \frac{u_i}{w_i}
= \frac{1}{\norm{w}_2^2 - w_i^2} - \frac{1-w_iu_i}{\norm{w}_2^2 - w_i^2} - \frac{u_i}{w_i}
= \frac{u_i}{w_i} ( \frac{2w_i^2 - \norm{w}_2^2}{\norm{w}_2^2 - w_i^2} ) \leq 0
\]
which establishes the claim by Kolmogorov criteria. 

\item If $2w_i^2 \geq \norm{w}_2^2$ then we claim that 
$\proj(-\frac{1}{w_i}e_i;A) = b = \frac{2}{\norm{w}_2^2}w - \frac{1}{w_i}e_i \geq 0$ with $\norm{b}_2^2 = \frac{1}{w_i^2}$. To prove the claim, consider any $u\in A$ and observe that, 
\begin{align*}
\langle -\frac{1}{w_i}e_i - b, u-b \rangle
= \norm{b}_2^2 - \langle b,u\rangle - \frac{u_i}{w_i} + \frac{2}{\norm{w}_2^2} - \frac{1}{w_i^2} = 0 
\end{align*}
which establishes the claim by Kolmogorov criteria. Note that the condition was used to make sure $b\in A$. 

\end{itemize}
With the projection at hand, calculating the distances is straightforward. 
\end{proof}

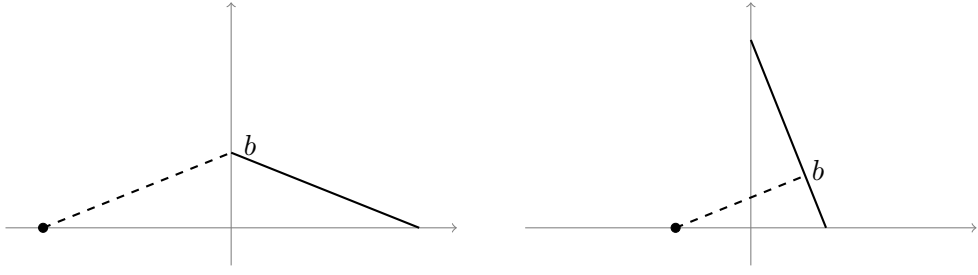
\begin{figure}[h]
\centering
\begin{tikzpicture}
	\draw [gray, ->] (-3,0)--(3,0);
	\draw [gray, ->] (0,-0.5)--(0,3);
	\draw [thick] (2.5,0)--(0,1);
	\draw [fill=black] (-2.5,0) circle [radius=.06];
	\draw [thick, dashed] (-2.5,0)--(0,1) node[pos=1.1] {$b$};
\end{tikzpicture}
\qquad
\begin{tikzpicture}
	\draw [gray, ->] (-3,0)--(3,0);
	\draw [gray, ->] (0,-0.5)--(0,3);
	\draw [thick] (0,2.5)--(1,0);
	\draw [fill=black] (-1,0) circle [radius=.06];
	\draw [thick, dashed] (-1,0)--($(1,0)!(-1,0)!(0,2.5)$) node[pos=1.1] {$b$};
	
\end{tikzpicture}
\caption{Illustrating the two possibilities in Proof of \autoref{lem:dist-weighted-simplex}. The thick line segments represent the set $A$ and the dot represents $-\frac{1}{w_i}e_i$. The projection is denoted by $b$. }
\label{fig:dist-weighted-simplex}
\end{figure}

\begin{lemma}\label{lem:min-norm-simplex}
For a given $w\in\mathbb{R}_{++}^p$, we have 
$\min\{\norm{u}_2^2:~ \langle w,u\rangle=1,~ u\geq 0\} = \frac{1}{\norm{w}_2^2}$ and 
$\argmin\{\norm{u}_2^2:~ \langle w,u\rangle=1,~ u\geq 0\} = \frac{1}{\norm{w}_2^2}w$. 
\end{lemma}
\begin{proof}[Proof of \autoref{lem:min-norm-simplex}]
Writing down the Lagrange dual of this optimization problem we get the desired result. 
\end{proof}

\begin{lemma}\label{lem:ext-inf-conv}
Consider two atomic norms and their infimal convolution. The extreme points of the ball for infimal convolution is a subset of the union of extreme points for each norm ball. 
\end{lemma}
\begin{proof}[Proof of \autoref{lem:ext-inf-conv}]
Easy from \eqref{eq:atomic_repr}. 
\end{proof}

\subsection{Weighted $\ell_1$ and $\ell_\infty$ Norms}
Given a positive vector $w\in\mathbb{R}_{++}^p$, one can define a pair of dual norms as 
\[
\sum_{i=1}^p w_i\abs{\beta_i} \qquad \text{and} \qquad \max_{i\in[p]} \frac{1}{w_i}\abs{\beta_i}
\]
which are commonly referred to as the {\em weighted $\ell_1$ norm} and the {\em weighted $\ell_\infty$ norm}, respectively. 

\begin{lemma}\label{lem:varphi-weighted-l1-exact}
Consider the weighted $\ell_1$ norm $f(\beta) = \sum_{i=1}^p w_i\abs{\beta_i}$ where $w\in\mathbb{R}_{++}^p$. Then, 
\[
\varphi^2(\beta; f) = 4 \norm{w_{\supp(\beta)}}_2^2
\]
\end{lemma}
\begin{proof}[Proof of \autoref{lem:varphi-weighted-l1-exact}]
Define $S = \{i\in[p]:~ \beta_i\neq 0\}$ and observe that 
\begin{align*}
	\partial f(\beta) 
	&= \{g:~ \langle g, \beta\rangle = f(\beta),~ \max_{i\in [p]}\frac{1}{w_i}\abs{\beta_i}<1\}\\
	&= \{g:~ g_i=w_i \sign(\beta_i) \text{ if } \beta_i\neq 0,~ \abs{g_i}\leq w_i \text{ otherwise}\}
\end{align*}
where we used the form of dual norm in the first equality. Then, \autoref{eq:varphi-max} implies 
\begin{align*}
\varphi^2(\beta; f) 
&= \max_{z} \min_{g} \;\bigl\{\norm{z-g}_2^2 :~ \abs{z_i}\leq w_i ~~i\in[p],~ g_i = w_i \sign(\beta_i)~~i\in S,~ \abs{g_i}\leq w_i ~~ i\in S^c \bigr\}\\
&= \max_{z}\; \bigl\{ \sum_{i\in S}(z_i-w_i \sign(\beta_i))_2^2 + \sum_{i\in S^c} (\abs{z_i} -w_i)_+^2:~ \abs{z_i}\leq w_i ~~i\in[p]\bigr\}\\
&= 4\norm{w_S}_2^2
\end{align*}
where $(a)_+ \equiv \max\{a,0\}$.
\end{proof}
\autoref{lem:varphi-weighted-l1-exact} recovers the earlier result $\varphi(\beta; \nor_1) = 2\sqrt{\norm{\beta}_0}$.

\begin{lemma}\label{lem:subdiff-weighted-linf}
For the weighted $\ell_\infty$ norm, namely $f(\beta) = \max_{i\in[p]} \frac{1}{w_i}\abs{\beta_i}$ with $w\in\mathbb{R}_{++}^p$, and $\beta\neq 0$, we have
\begin{align*}
\partial f(\beta) 
&= \{g:~ g\circ \beta \geq 0,~ g_{T^c}=0,~ \sum_{i\in T} w_i \abs{g_i} = 1\} 
\end{align*}
where $T\equiv \{i\in[p]:~ \frac{1}{w_i}\abs{\beta_i} = f(\beta)\}\neq \emptyset$. Moreover, $\min_{g\in\partial f(\beta)}\norm{g}_2^2 = \frac{1}{\norm{w_T}_2^2}$.
\end{lemma}
\begin{proof}[Proof of \autoref{lem:subdiff-weighted-linf}]
For $\beta\neq 0$, consider $T$ and observe that 
\begin{align*}
\partial f(\beta)
&= \bigl\{g:~ \langle g,\beta\rangle = \max_{i\in[p]}\frac{1}{w_i}\abs{\beta_i},~ \sum_{i=1}^p w_i\abs{g_i} = 1 \bigr\}.
\end{align*}
For any $g$ in the above, we have
\[
\langle g,\beta\rangle
= \sum_{i\in S} g_i\beta_i
= \sum_{i\in S} (w_ig_i)(\frac{1}{w_i} \beta_i) 
\leq \sum_{i\in S} (w_i\abs{g_i})(\frac{1}{w_i} \abs{\beta_i})
\leq (\max_{i\in[p]} \frac{1}{w_i} \abs{\beta_i}) \sum_{i\in S}w_i\abs{g_i} 
= \max_{i\in[p]} \frac{1}{w_i} \abs{\beta_i}
\]
which implies that the inequalities have to hold with equality, establishing $g_i = 0$ for $i\not\in T$, $g\circ \beta \geq 0$, as well as $\sum_{i\in T}w_i\abs{g_i} = 1$. This completes the characterization of the subdifferential (checking that each such $g$ is a subgradient is straightforward). 

The last statement follows from \autoref{lem:min-norm-simplex}.
\end{proof}

\begin{lemma}\label{lem:varphi-weighted-linf-exact}
For the weighted $\ell_\infty$ norm, namely $f(\beta) = \max_{i\in[p]} \frac{1}{w_i}\abs{\beta_i}$ with $w\in\mathbb{R}_{++}^p$, and $\beta\neq 0$, we have
	\begin{equation*}
	\varphi^2(\beta; f) =
	\begin{cases}
	\max\left\{
	\frac{1}{\omega^2}+\frac{1}{\norm{w_T}_2^2}~,~
	\frac{4}{\norm{w_T}_2^2} 
	\right\} 
	&\text {if } \abs{T}=1
	\\
	\max\left\{
	\frac{1}{\omega^2}+\frac{1}{\norm{w_T}_2^2}~,~
	\frac{1}{\tau^2}+\frac{1}{\norm{w_T}_2^2-\tau^2}
	\right\} 
	& \text{if } \abs{T}\geq 2
	\end{cases}
	\end{equation*}
where $T\equiv \{i\in[p]:~ \frac{1}{w_i}\abs{\beta_i} = f(\beta)\}\neq \emptyset$, 
$T_1\equiv \{i\in T:~ 2w_i^2 \leq \norm{w_T}_2^2\}$, 
$\omega \equiv \min_{i\not\in T} w_i$, and $\tau \equiv \min_{i\in T_1} w_i$. Note that $\abs{T}\geq 2$ if and only if $T_1$ is non-empty. Moreover, if $\abs{T}\geq 2$ then $2\tau^2\leq \norm{w_T}_2^2$ which implies $\varphi^2 \leq \max\{\frac{1}{\omega^2},\frac{1}{\tau^2}\}+\frac{1}{\norm{w_T}_2^2-\tau^2} \leq \max\{\frac{1}{\omega^2},\frac{1}{\tau^2}\}+\frac{1}{\tau^2} \leq 2\max\{\frac{1}{\omega^2},\frac{1}{\tau^2}\}\leq 2/(\min_{i\in[p]}w_i)^2$. 
\end{lemma}
\begin{proof}[Proof of \autoref{lem:varphi-weighted-linf-exact}]
Consider the characterization of $\partial f(\beta)$, for any $\beta\neq 0$, from \autoref{lem:subdiff-weighted-linf}. Moreover, note that $\ext(\Bc^\star) = \{\pm \frac{1}{w_i}e_i:~ i\in[p]\}$ where $e_i$ is the $i$-th standard basis vector. There are three cases: 
\begin{itemize}
\item If $i\not\in T$ and $z = \pm \frac{1}{w_i}e_i$, then $z_T =0$ and $\norm{z_{T^c}}_2 = \frac{1}{w_i}$. Therefore, 
\begin{align*}
\min_{g\in\partial f(\beta)} \norm{z-g}_2^2 
= \min_{g\in\partial f(\beta)} \norm{g_T}_2^2 + \norm{z_{T^c}}_2^2
= \frac{1}{\norm{w_T}_2^2} + \frac{1}{w_i^2},
\end{align*}
where we used \autoref{lem:subdiff-weighted-linf}.

\item If $i\in T$ and $z = \frac{1}{w_i}e_i$ then $z\in \partial f(\beta)$ which implies $\dist(z,\partial f(\beta)) = 0$. 

\item If $i\in T$ and $z = -\frac{1}{w_i}e_i$ then we use \autoref{lem:dist-weighted-simplex} to get 
\[
\dist^2(-\frac{1}{w_i}e_i, \partial f(\beta)) = 
\begin{cases}
 \frac{4}{w_i^2}& \abs{T}=1, \\
\frac{1}{w_i^2} + \frac{1}{\norm{w_T}_2^2 - w_i^2}& \abs{T} \geq 2,~ i\in T_1, \\
\frac{4}{\norm{w_T}_2^2}& \abs{T} \geq 2,~ i\in T\backslash T_1. \\
\end{cases}
\] 
Observe that for any $i\in T$, we have $\frac{1}{w_i^2} + \frac{1}{\norm{w_T}_2^2 - w_i^2} > \frac{4}{w_i^2}$. 
\end{itemize}
Combining all of the above, to find the maximum over all $z\in\ext(\Bc^\star)$, we get
	\begin{equation*}
	\varphi^2(\beta; f) =
	\begin{cases}
	\max\left\{
	\max_{i\not\in T} (\frac{1}{w_i^2}+\frac{1}{\norm{w_T}_2^2})~,~
	\frac{4}{\norm{w_T}_2^2} 
	\right\} 
	&\text {if } \abs{T}=1
	\\
	\max\left\{
	\max_{i\not\in T} (\frac{1}{w_i^2}+\frac{1}{\norm{w_T}_2^2})~,~
	\max_{i\in T_1} (\frac{1}{w_i^2}+\frac{1}{\norm{w_T}_2^2-w_i^2})
	\right\} 
	& \text{if } \abs{T}\geq 2
	\end{cases}
	\end{equation*}
The claim follows by defining $\omega$ and $\tau$. 
\end{proof}
\autoref{lem:varphi-weighted-linf-exact} provides an alternative proof for \autoref{lem:varphi-linf-exact}.

\subsection{Bounds for the Ordered Weighted $\ell_1$ Norm}\label{app:OWL}

Given $w_1\geq w_2 \geq \cdots \geq w_p \geq 0$, the ordered weighted $\ell_1$ norm is defined as
\begin{align}\label{eq:def-OWL}
\norm{\beta}_\owl = \sum_{i=1}^p w_i \bar\beta_i
\end{align}
where $\bar\beta$ is the sorted absolute value of $\beta$ satisfying $\bar\beta_1\geq \bar\beta_2 \geq \cdots \geq \bar\beta_p \geq 0$. The above is clearly 1-homogeneous. It is also convex due to the assumption on $w$.

\begin{lemma}[Lemma 1 in \citep{zeng2014ordered}]
The dual norm for $\nor_\owl$ is given by
\begin{align}\label{eq:OWL-dual}
\norm{z}_\owl^\star = \max_{i\in [p]} ~\frac{\sum_{j=1}^i \bar z_j}{\sum_{j=1}^i w_j}.
\end{align}
\end{lemma}

In the following, we present results on these norms, which to the best of our knowledge, are new.

\begin{remark}
Using the characterization of the norm ball for $\nor_\owl$ in \cite[Theorem 1]{zeng2014ordered}, it is easy to see that $\nor_\owl$ is a structure norm (all of the extreme points lie on the unit sphere) if and only if $w_i = \sqrt{i}-\sqrt{i-1}$ for $i\in[p]$. Observe that such $w$ satisfies $w_1\geq \cdots \geq w_p>0$ which is required in defining $\nor_\owl$. In the approach of \citep{obozinski2016unified}, for such norm $\nor_\owl$ we have $\norm{\beta}_\owl^\star = \max\{ \frac{1}{\sqrt{\abs{A}}} \norm{\beta_A}_1:~ A\subseteq [p]\}$. 
\end{remark}

\begin{lemma}\label{lem:extBst-owl}
Consider $w\in\mathbb{R}^p$ with $w_1\geq w_2\geq \cdots \geq w_p> 0$ and $\Bc_{\nor_\owl}^\star =\Bc_{\nor_\owl^\star} = \{z:~\norm{z}_\owl^\star \leq 1\}$. Then, $\ext(\Bc_{\nor_\owl^\star}) = \{Qw:~ Q\in \mathcal{P}_\pm \}$. This implies that $\norm{w}_2\nor_\owl^\star$ is a structure norm in the sense of \autoref{sec:Snorms-def}. 
\end{lemma}
\begin{proof}[Proof of \autoref{lem:extBst-owl}]
First, the support function for the right-hand side is equal to $\nor_\owl$. Therefore, the convex hull of the right-hand side is $\Bc^\star$. On the other hand, without loss of generality consider $w=I\cdot w$ and assume $w=\alpha x+ (1-\alpha) y$ for some $\alpha\in[0,1]$ and $x,y\in\Bc^\star$. Then, for any $i\in[p]$, 
\[
1 
= \frac{\sum_{j=1}^i w_j}{\sum_{j=1}^i w_j}
= \frac{\alpha \sum_{j=1}^i x_j + (1-\alpha) \sum_{j=1}^i y_j}{\sum_{j=1}^i w_j}
\leq \frac{\alpha \sum_{j=1}^i \bar x_j + (1-\alpha) \sum_{j=1}^i \bar y_j }{\sum_{j=1}^i w_j}
\leq 1
\]
which implies that $x=y=w$. Therefore, $w$ is an extreme point of the dual norm ball. 
\end{proof}

\begin{definition}
Given $\beta$, sort $\abs{\beta}$ in descending order to get $\bar\beta$. Moreover, consider $d = \abs{\{ \abs{\beta_i}\neq 0:~ i\in[p]\}}$. Then, define $\cG = (\cG_1, \cdots, \cG_d)$ as a partition of $\supp(\bar\beta)$ into $d$ {\em intervals} where for any $i,j\in\supp(\bar\beta)$ and any $t\in[d]$: $i,j\in \cG_t$ if and only if $\bar\beta_i = \bar\beta_j$. Moreover, define $\cG_0 \equiv [p]\backslash \supp(\bar\beta)$.
\end{definition}

\begin{lemma}\label{lem:OWL-subdiff}
Given $w_1\geq w_2 \geq \cdots \geq w_p \geq 0$, the ordered weighted $\ell_1$ norm defined by \eqref{eq:def-OWL}. Then, the subdifferential at $\beta\in\mathbb{R}^p$ is given by 
\begin{align}\label{eq:OWL-subdiff}
\partial \norm{\beta}_\owl = 
\Bigl\{ g:~ 
&g\circ \beta \geq 0,~ 
\text{$\abs{g}$ and $\abs{\beta}$ are similarly sorted},~
\sum_{j=1}^i \bar g_j \leq \sum_{j=1}^i w_j ~~ \forall i\in [p],
\nonumber\\
&\sum_{j\in \cG_t} \bar g_j = \sum_{j\in \cG_t} w_j ~~ \forall\, t\in[d],~
\sum_{j\in \cG_0} \bar g_j \leq \sum_{j\in \cG_0} w_j
\Bigr\}. 
\end{align} 
\end{lemma}
\begin{proof}[Proof of \autoref{lem:OWL-subdiff}]
Consider from \citep{watson1992characterization} the characterization of the subdifferential for a norm as 
\[
\partial \norm{\beta}_\owl = \{g:~ \langle g,\beta\rangle = \norm{\beta}_\owl ,~ \norm{g}_\owl^\star =1 \}. 
\]
For any $g\in \partial \norm{\beta}_\owl $, we have 
\begin{align}\label{eq:owl-dummy1}
\< w, \bar\beta\> = \norm{\beta}_\owl = \<g, \beta\> \leq \<\bar{g}, \bar\beta\>
\end{align}
where the last inequality holds by the rearrangement inequality. Therefore, with the convention $\bar\beta_{p+1}=0$ we have, 
\begin{align*}
	\sum_{i=1}^p w_i \bar\beta_i
&\stackrel{(a)}{=}	
	\sum_{i=1}^p \left( (\bar\beta_i-\bar\beta_{i+1}) \sum_{j=1}^i w_j \right) \\
&\stackrel{(b)}{\geq}
	\sum_{i=1}^p \left( (\bar\beta_i-\bar\beta_{i+1}) \sum_{j=1}^i \bar g_j \right) \\
&\stackrel{(a)}{=}	
	\sum_{i=1}^p \bar g_i \bar\beta_i\\
&\stackrel{(c)}{\geq}	
	\sum_{i=1}^p w_i \bar\beta_i
\end{align*}
where $(a)$ is a trick we use, $(b)$ is by $\norm{g}_\owl^\star \leq 1$ and \eqref{eq:OWL-dual}, and $(c)$ is by \eqref{eq:owl-dummy1}. Therefore, all of the inequalities we have used must hold with equality: From \eqref{eq:owl-dummy1} we get that $g$ and $\beta$ are similarly signed, and, $\abs{g}$ and $\abs{\beta}$ are similarly sorted. Moreover, equality in $(b)$ implies
\begin{align}\label{eq:OWL-dummy2}
\sum_{j=1}^i w_j=\sum_{j=1}^i \bar g_j 
~~\text{whenever}~~
\bar\beta_i>\bar\beta_{i+1} 
\end{align}
with the previous convention $\bar\beta_{p+1}=0$. Recall the definitions $d = \abs{\{ \abs{\beta_i} \neq 0:~ i\in[p]\}}$ and $\cG = (\cG_1, \cdots, \cG_d)$ for~$\beta$, from right before the statement of \autoref{lem:OWL-subdiff}. Then, we get \eqref{eq:OWL-subdiff} where the last two conditions have been derived from \eqref{eq:OWL-dummy2}. 
\end{proof}
For example, consider $w=e_1$ which gives $\norm{\cdot}_\owl = \norm{\cdot}_\infty$ whose subdifferential is given in \eqref{eq:subdiff-linf}.

\begin{proof}[Proof of \autoref{lem:varphi-OWL}]
We use the min-max inequality to get 
\begin{align*}
\varphi^2(\beta) 
&= \max_{z\in \Bc^\star}\, \min_{g\in \partial\|\beta\|} ~\norm{g-z}_2^2 \nonumber \\
&\leq \min_{g\in \partial\|\beta\|} \,\max_{z\in \Bc^\star} ~\norm{g-z}_2^2 \nonumber \\
&= \min_{g\in \partial\|\beta\|} \,\max_{z\in \Bc^\star} ~\norm{g+z}_2^2 \nonumber 
\end{align*}
where we used the symmetry of $\Bc^\star$. 
We now focus on the inner optimization problem. Fix $g\in \partial \norm{\beta}_\owl$ and consider
\[
\max_z \bigl\{ \norm{z}_2^2 + 2\<z,g\>:~ 
{\sum_{j=1}^i \bar z_j} \leq {\sum_{j=1}^i w_j}
~~\forall\,i\in[p] \bigr\}.
\]
Observe that 1) the optimal $z$ will have the same sign pattern as $g$, 2) $\abs{z}$ and $\abs{g}$ are similarly ordered. Furthermore, we claim that the optimal $z$ satisfies $\bar z = w$, hence providing the optimal $z$ completely (one can use \autoref{lem:extBst-owl} to establish this claim). For this, we show that such choice of $z$ maximize each of the two terms in the objective subject to the constraint. Take any $z$ on the boundary of the dual norm ball. First, observe that 
\begin{align}
\<z,g\> = \<\bar z,\bar g\> 
= \sum_{i=1}^p \left( (\bar g_i-\bar g_{i+1}) \sum_{j=1}^i \bar z_j \right)
\leq \sum_{i=1}^p \left( (\bar g_i-\bar g_{i+1}) \sum_{j=1}^i w_j \right)
= \<w, \bar g\>.
\end{align}
Secondly, 
\begin{align}
\norm{z}_2^2 = \<\bar z,\bar z\> 
&= \sum_{i=1}^p \left( (\bar z_i-\bar z_{i+1}) \sum_{j=1}^i \bar z_j \right) \nonumber\\ 
&\leq \sum_{i=1}^p \left( (\bar z_i-\bar z_{i+1}) \sum_{j=1}^i w_j \right) \nonumber\\
&= \<w, \bar z\> \nonumber\\
&= \sum_{i=1}^p \left( (w_i-w_{i+1}) \sum_{j=1}^i \bar z_j \right)\nonumber\\
&\leq \sum_{i=1}^p \left( (w_i-w_{i+1}) \sum_{j=1}^i w_j \right) \nonumber\\
&= \<w,w\> = \norm{w}_2^2
\end{align}
Finally, note that $w$ (and any signed permuted version of it) is feasible in the above optimization program; i.e., $w\in\Bc^\star$. Therefore, the optimal value for the original inner optimization program is given by 
\[
\norm{g}_2^2 + \norm{w}_2^2 + 2 \<w,\bar g\>. 
\]
Next, we are interested in minimizing the above for all $g\in \partial \norm{\beta}_\owl$ where the subdifferential is characterized in~\eqref{eq:OWL-subdiff}. After a slight change in variable $g$, we would like to solve
\begin{align}\label{eq:OWL-minmax-min}
\min_g \Bigl\{ \norm{g+w}_2^2 :~
	 g_1 \geq \cdots \geq g_k \geq 0 ~,
	 \sum_{j\in \cG_t} \bar g_j = \sum_{j\in \cG_t} w_j ~~ \forall\, t\in[d] ~,
	 \sum_{j\in \cG_0} \bar g_j \leq \sum_{j\in \cG_0} w_j \Bigr\}
\end{align}
where $k = \norm{\beta}_0$. 
We upper bound the above by plugging in 
\[
g = \left[ 
\frac{\sum_{j\in \cG_1} w_j}{\abs{\cG_1}} \one_{\abs{\cG_1}}^\sT 
~,~ \cdots ~,~
\frac{\sum_{j\in \cG_d} w_j}{\abs{\cG_d}} \one_{\abs{\cG_d}}^\sT
~,~ -w_{\cG_0}^\sT
\right]^\sT 
\]
which gives
\begin{align*}
\varphi^2(\beta) 
\leq \norm{w_{\cG}}_2^2 + 3\sum_{t=1}^d \frac{(\sum_{j\in \cG_t} w_j)^2}{\abs{\cG_t}} 
\end{align*}
where we abuse the notation to denote $\cG = \cup_{t=1}^d \cG_t = \supp(\beta)$, and where $d = \abs{\{ \abs{\beta_i}\neq 0:~ i\in[p]\}}$, and the partition $\cG = (\cG_1, \cdots, \cG_d)$ is according to equal absolute values in $\beta$. This finishes proof. 
Moreover, 
\begin{align*}
\sum_{t=1}^d \frac{(\sum_{j\in \cG_t} w_j)^2}{\abs{\cG_t}} 
= \norm{w_\cG}_2^2 - \dist^2(w; \Sc_\cG(\beta))
\end{align*}
where $\Sc_\cG(\beta) = \{u:~ \bar\beta_i=\bar\beta_j \implies \bar u_i = \bar u_j\}$.
\end{proof}

\begin{proof}[Proof of \autoref{lem:varphi-linf}]
We use the min-max inequality to get 
\begin{align*}
\varphi^2(\beta^\star) 
&= \max_{z\in \Bc^\star}\, \min_{g\in \partial\|\beta^\star\|} ~\norm{g-z}_2^2 \nonumber \\
&\leq \min_{g\in \partial\|\beta^\star\|} \,\max_{z\in \Bc^\star} ~\norm{g-z}_2^2 \nonumber \\
&= \min_{g\in \partial\|\beta^\star\|} \,\max_{z\in \Bc^\star} ~\norm{g+z}_2^2 \nonumber \\
&\leq \min_{g\in \partial\|\beta^\star\|} \, \bigl\{ \norm{g}_2^2 + \max_{z\in \Bc^\star} ~\norm{z}_2^2 + 2\langle g,z \rangle \bigr\} 
\end{align*}
where we used the symmetry of $\Bc^\star$. We now focus on the inner optimization problem. 

It is easy to see that vertices of the (scaled) $\ell_1$ norm ball maximize both $\norm{z}_2^2$ and $\langle g,z \rangle$. Therefore, the optimal value of the original inner problem is given by 
\[
\norm{g}_2^2 + 1 + 2\norm{g}_\infty .
\]
Now, we would like to minimize the above over all $g\in \partial \norm{\beta^\star}_\infty$ where 
\begin{align}
\partial \norm{\beta^\star}_\infty
&= \bigl\{ g:~ \langle g, \beta^\star\rangle = \norm{\beta^\star}_\infty,~ \norm{g}_1\leq 1 \bigr\} \nonumber\\
&= \bigl\{ g:~ g_i = 0 \text{ if } \abs{\beta^\star_i}<\norm{\beta^\star}_\infty,~ 
\norm{g}_1=1,~ g \circ \beta \geq 0 \bigr\}.\label{eq:subdiff-linf}
\end{align}
This time, note that a vector with all equal values minimizes both the $\ell_2$ and the $\ell_\infty$ norm subject to $\ell_1$ constraints. Therefore, for $t = \abs{\{i\in[p]:~ \abs{\beta^\star_i} = \norm{\beta^\star}_\infty\}}$, the optimal $g$ has~$t$ nonzero entries with absolute values equal to $1/t$, which yields 
\begin{align*}
\varphi^2(\beta^\star) 
\leq \frac{1}{t} + 1 + 2 \cdot\frac{1}{t} 
= 1 + \frac{3}{t} 
\leq 4
\end{align*}
and finishes the proof. 
\end{proof}

\begin{proof}[Proof of \autoref{lem:varphi-linf-exact}]
Consider \eqref{eq:varphi-max-ext} and observe that $\ext(\Bc^\star) = \{\pm e_i:~ i\in[p]\}$ where $e_i$ is the $i$-th standard basis vector. 
Define $S = \{i\in[p]:~ \beta^\star_i = \norm{\beta^\star}_\infty \} $ and $t=\abs{S}$. 

\begin{itemize}
\item Case 1: For $i\not\in S$ and $z=\pm e_i$ we have $\dist^2(z, \partial \norm{\beta^\star}_\infty) = \min_{g\in \partial \norm{\beta^\star}_\infty} 1+\norm{g}_2^2=1 + \frac{1}{t}$. If $S=[p]$, we ignore this case in the maximum over $z\in \ext(\Bc^\star)$ in \eqref{eq:varphi-max-ext}. 

\item Case 2: For $i\in S$ and $z = \sign(\beta^\star_i)e_i$ we have $\dist(z, \partial \norm{\beta^\star}_\infty) =0$. 

\item Case 3: For $i\in S$ and $z = -\sign(\beta^\star_i)e_i$, the distance is equal to the distance of $-\abs{z}$ to a $t$-dimensional simplex whose square, by \autoref{lem:dist-ei-simplex}, is equal to $1+\frac{1}{t-1}$ when $t\geq 2$ and is equal to $4$ when $t=1$. 
\end{itemize}
Gathering all of the above into the maximum over $z\in \ext(\Bc^\star)$ in \eqref{eq:varphi-max-ext} yields the desired result.
\end{proof}

\subsection{Doubly-sparse Norms: \texorpdfstring{$k\square 1$}{$(k,1)$}}

\begin{lemma}\label{lem:all-k-1}
We have 
\begin{enumerate}
\item \label{lem:all-k-1-normk1}
$\norm{\beta}_{k\square 1} = \max\{\frac{1}{\sqrt{k}}\norm{\beta}_1 , \sqrt{k}\norm{\beta}_\infty \}$

\item \label{lem:all-k-1-normk1-dual}
$\norm{\beta}_{k\square 1}^\star = \frac{1}{\sqrt{k}}\sum_{i=1}^k \bar\beta_i = \inf_{u,v} \bigl\{ \frac{1}{\sqrt{k}}\norm{u}_1 + \sqrt{k}\norm{v}_\infty:~ \beta = u+v \bigr\}$ which leads to a representation as an ordered weighted $\ell_1$ norm, $\norm{\cdot}^\star_{k\square 1} = \norm{\cdot}_w$, with $w = \frac{1}{\sqrt{k}}[\one_k^\sT~,~ \zero_{p-k}^\sT]^\sT$.

\item \label{lem:all-k-1-normk1-ext}
$\ext(\Bc_{k\square 1}) = \Sc_{k,1}\cap \mathbb{S}^{p-1} = \{Q\theta:~ Q\in \mathcal{P}_\pm,~ \theta = \frac{1}{\sqrt{k}}[\one_k^\sT ~,~ \zero_{p-k}^\sT]^\sT\}$.

\item \label{lem:all-k-1-normk1-dual-ext}
$\ext(\Bc_{k\square 1}^\star) = \{Q\theta:~\theta\in A,~Q\in\mathcal{P}_\pm\}$ where $A = \{\sqrt{k}e_1, \frac{1}{\sqrt{k}}\one_p\}$.

\end{enumerate}
\end{lemma}
\begin{proof}[Proof of \autoref{lem:all-k-1}]
The duality of $\frac{1}{\sqrt{k}}\sum_{i=1}^k \bar\beta_i$ and $\max\{\frac{1}{\sqrt{k}}\norm{\beta}_1 , \sqrt{k}\norm{\beta}_\infty \}$ is well-known; e.g., see Exercise IV.1.18~in \citep{bhatia1997matrix}. The representation of $\sum_{i=1}^k \bar\beta_i$ as an infimal convolution can be found in \cite[Proposition IV.1.5]{bhatia1997matrix}. 

Recall the definition of $\Sc_{k,d}$ from \eqref{eq:def-Skd} which gives
\begin{align*}
\Sc_{k,1} 
&= \bigl\{\beta:~ \card(\beta) \leq k \,,~ \abs{\{\bar{\beta}_1,\ldots,\bar{\beta}_k\}} \leq 1 \bigr\} \\
&= \bigl\{\beta:~ \card(\beta) = k \,,~ \abs{\{\bar{\beta}_1,\ldots,\bar{\beta}_k\}} = 1 \bigr\} \cup\{0\}\\
&= \bigl\{\eta Q\theta:~ \theta = [\one_k^\sT, \zero_{p-k}^\sT],~ Q\in\mathcal{P}_\pm ,~ \eta \in\mathbb{R}\bigr\}.
\end{align*}
Therefore, 
\[
\norm{\beta}_{k\square 1}^\star 
= \sup\{\langle \theta,\beta\rangle: \theta\in\Sc_{k,1},~\norm{\theta}_2 =1 \}
= \sup\{\frac{1}{\sqrt{k}}\langle Q\theta,\beta\rangle: \theta= [\one_k^\sT, \zero_{p-k}^\sT],~ Q\in\mathcal{P}_\pm \}
= \frac{1}{\sqrt{k}} \sum_{i=1}^k \bar\beta_i.
\]
These establish \autoref{lem:all-k-1-normk1} and \autoref{lem:all-k-1-normk1-dual}. 
To prove \autoref{lem:all-k-1-normk1-ext}, observe that by the representation of $\Sc_{k,1}$ above, and by the definitions in \eqref{eq:Snorm-gauge} and \eqref{eq:gauge}, we have 
\[
\ext(\Bc_{k\square 1}) \subseteq \Sc_{k,1}\cap \mathbb{S}^{p-1} = \{Q\theta:~ Q\in \mathcal{P}_\pm,~ \theta = \frac{1}{\sqrt{k}}[\one_k^\sT ~,~ \zero_{p-k}^\sT]^\sT\}.
\] 
Then, since each element on the right-hand side has $\ell_2$ norm equal to $1$, no one can be in the convex hull of others. Therefore, we get equality which establishes the claim. 
Alternatively, assuming \autoref{lem:all-k-1-normk1-dual}, then $\nor_{k\square 1}^\star$ is an ordered weighted $\ell_1$ norm with $w = [\frac{1}{\sqrt{k}}\one_k^\sT ~,~ \zero_{p-k}^\sT]^\sT$. Therefore, \autoref{lem:all-k-1-normk1-ext} can also be seen from \autoref{lem:extBst-owl}. 
\autoref{lem:all-k-1-normk1-dual-ext} follows from \autoref{lem:ext-inf-conv} and \autoref{lem:all-k-1-normk1-dual}. 
\end{proof}

\begin{remark}
The representation of $\nor_{k\square 1}^\star$ as an ordered weighted $\ell_1$ norm in \autoref{lem:all-k-1-normk1-dual} and the atomic representation for this family of norms in \cite[Theorem 1]{zeng2014ordered}, provide
\[
\ext(\Bc_{k\square 1}^\star) \subseteq \{Q\theta:~ Q\in\mathcal{P}_\pm ,~ \theta = \frac{\sqrt{k}}{\min\{r,k\}} [\one_r^\sT, \zero_{p-r}^\sT]^\sT ~~~ r\in[p] \}.
\]
However, as evident from \autoref{lem:all-k-1} (\autoref{lem:all-k-1-normk1-dual-ext}), many of the points on the right-hand side are redundant (lie in the convex hull of others). 
\end{remark}

\begin{lemma}\label{lem:varphi-k-1}
For a given $\beta\neq 0$, define $k^\star = \norm{\beta}_0$ and $t^\star = \abs{\{i\in[p]:~ \abs{\beta_i} = \norm{\beta}_\infty\}}$. Then, 
\begin{itemize}
\item If $\norm{\beta}_1 > k \norm{\beta}_\infty$ then $\varphi^2(\beta; \nor_{k\square 1}) = \max\bigl\{\frac{4k^\star}{k}, 2+\frac{k^\star}{k}+k \bigr\}$. 
\item If $\norm{\beta}_1 < k \norm{\beta}_\infty$ then $\varphi^2(\beta; \nor_{k\square 1}) = \max\bigl\{ 
k(1+\frac{1}{\max\{t^\star-1,1/3\}}),
2+\frac{k}{t^\star}+\frac{p}{k}
\bigr\}$
\item If $\norm{\beta}_1 = k \norm{\beta}_\infty$ then $\varphi^2(\beta; \nor_{k\square 1})$ is bounded from above by the {\em minimum} of the two above values. 
\end{itemize}
\end{lemma}
As an example, consider $k=p$ and assume $\beta$ is not a multiple of $\one_p$, hence $t<p$. Then, using the second item above, we recover the result of \autoref{lem:varphi-linf-exact}.
 
\begin{proof}[Proof of \autoref{lem:varphi-k-1}]
Recall from \autoref{lem:all-k-1} (\autoref{lem:all-k-1-normk1}) that $\norm{\beta}_{k\square 1} = \max\{\frac{1}{\sqrt{k}}\norm{\beta}_1 , \sqrt{k}\norm{\beta}_\infty \}$. Therefore, 
\begin{align*}
\partial \norm{\beta}_{k \square 1} = \begin{cases}
	\frac{1}{\sqrt{k}}\partial \norm{\beta}_1 & \norm{\beta}_1 > k \norm{\beta}_\infty, \\
	\sqrt{k}\partial \norm{\beta}_\infty & \norm{\beta}_1 < k \norm{\beta}_\infty, \\
	\conv(\frac{1}{\sqrt{k}}\partial \norm{\beta}_1 \cup \sqrt{k}\partial \norm{\beta}_\infty) & \norm{\beta}_1 = k \norm{\beta}_\infty.
\end{cases}
\end{align*}
Consider $S = \supp(\beta)$. 
In the following, we first compute the distance to the subdifferential in each case. 
\begin{itemize}
\item If $\norm{\beta}_1 > k \norm{\beta}_\infty$, we have $k^\star = \norm{\beta}_0 \geq \norm{\beta}_1/ \norm{\beta}_\infty > k$. Fix $z\in \Bc^\star$ and observe that 
\[
\dist^2(z, \frac{1}{\sqrt{k}}\partial \norm{\beta}_1)
= \sum_{i\in S} (z_i - \frac{1}{\sqrt{k}} \sign(\beta_i) )^2 
+ \sum_{i\not\in S} (\abs{z_i} - \frac{1}{\sqrt{k}})_+^2.
\] 
In maximizing the above over all $z\in \Bc^\star$, we use the sign-invariance property to arrive at 
\[
\varphi^2 = \max_z\, \bigl\{ 
	\sum_{i\in S} (\abs{z_i} + \frac{1}{\sqrt{k}} )^2 
	+ \sum_{i\not\in S} (\abs{z_i} - \frac{1}{\sqrt{k}})_+^2
	:~ \sum_{i=1}^k \bar z_i = \sqrt{k}
\bigr\}.
\]
Denote by $z^\star$ an optimal solution to the above. Using the permutation-invariance property, it is easy to show that if $i\in S$ and $j\in S^c$, then $\abs{z_i^\star} \geq \abs{z_j^\star}$. This allows for replacing $S$ with $[k^\star]$ as well as for adding a constraint $\abs{z_1}\geq \abs{z_2} \geq \cdots \geq \abs{z_p}$ (or $\abs{z}=\bar z$) to the above optimization without changing the optimal solution. 

Therefore, as the constraint is insensitive to the lowest $p-k$ values, we can set $\bar z_{k} = \bar z_{k+1} = \cdots = \bar z_{p} = \theta $. Then, we get 
\begin{align*}
\varphi^2 = \max_{0\leq \theta \leq 1/\sqrt{k}} ~ \max_{h} ~ \Bigl\{ 
\sum_{i=1}^{k-1} (h_i + \theta + \frac{1}{\sqrt{k}})^2 + (k^\star-k+1)(\theta + \frac{1}{\sqrt{k}})^2 + (p-k^\star)(\theta - \frac{1}{\sqrt{k}})_+^2 \\
\sum_{i=1}^{k-1} h_i = \sqrt{k}-k\theta ,~ h_1,\ldots,h_{k-1} \geq 0
\Bigr\}
\end{align*}
where we used the assumption $k^\star>k$ to break $[k^\star]$ into $[k-1]$ and $[k^\star]\backslash [k-1]$. The optimization problem over $h$ is a continuous-convex maximization over a compact convex domain. Hence, by Bauer's Maximum Principle (e.g., see \citet[Proposition~1.7.8]{schirotzek2007nonsmooth}), the maximum is attained by one of the extreme points of the feasible set, which due to symmetry in variables can be taken to be $h_{\rm opt}=[\sqrt{k}-k\theta, \zero_{k-2}]^\sT$. Plugging this into the above gives 
\begin{align*}
\varphi^2 = \max_{0\leq \theta \leq 1/\sqrt{k}} ~ \Bigl\{ 
 (\sqrt{k}-k\theta + \theta + \frac{1}{\sqrt{k}})^2 + (k^\star-1)(\theta + \frac{1}{\sqrt{k}})^2 
\Bigr\}.
\end{align*}
Again, we are dealing with a convex maximization problem which will attain its maximum at the boundary. Therefore, plugging $\theta=0$ and $\theta=\frac{1}{\sqrt{k}}$ in the objective yields
\begin{align}\label{eq:varphi-k-1-gtrk}
\varphi^2 = \max\bigl\{ k+2+\frac{k^\star}{k} , \frac{4k^\star}{k}\bigr\}.
\end{align}

\item If $\norm{\beta}_1 < k \norm{\beta}_\infty$, then $t=t^\star = \abs{T} \leq \norm{\beta}_1/\norm{\beta}_\infty < k$ where $T=\{i\in[p]:~ \abs{\beta_i} = \norm{\beta}_\infty\}$. Fix $z\in\Bc^\star$ and observe that 
\[
\dist^2(z, \sqrt{k}\partial \norm{\beta}_\infty)
= \min_h \bigl\{ \sum_{i\in T} (z_i - \sqrt{k} h_i \sign(\beta_i) )^2 
+ \sum_{i\not\in T} z_i^2:~ h\geq \zero_{t},~ \one^\sT h=1 \bigr\}.
\] 
In maximizing the above over all $z\in \ext(\Bc^\star)$, we use the sign-invariance property to arrive at
\[
\varphi^2 = \max_z\min_h \bigl\{ 
k\cdot \sum_{i\in T} ( \frac{1}{\sqrt{k}}\abs{z_i} + h_i )^2 
+ \sum_{i\not\in T} z_i^2
:~
h\geq \zero_{t},~ \one^\sT h=1,~
z\in\ext(\Bc^\star)
\bigr\}.
\]
Denote by $z^\star$ an optimal solution to the above. Using the permutation-invariance property, it is easy to show that if $i\in T$ and $j\in T^c$, then $\abs{z_i^\star} \geq \abs{z_j^\star}$. This allows for replacing $T$ with $[t]$ as well as for adding a constraint $\abs{z_1}\geq \abs{z_2} \geq \cdots \geq \abs{z_p}$ (or $\abs{z}=\bar z$) to the above optimization without changing the optimal solution. 
Hence, we get 
\begin{align}\label{eq:varphi-k-1-dumm}
\varphi^2 = \max_z \min_h \bigl\{ 
k\cdot \sum_{i=1}^t ( \frac{1}{\sqrt{k}}z_i + h_i )^2 
+ \sum_{i=t+1}^p z_i^2
:~
h\geq \zero_{t},~ \one^\sT h=1,~ 
z\in\{\sqrt{k}e_1, \frac{1}{\sqrt{k}}\one_p\}
\bigr\}
\end{align}
where we used \autoref{lem:all-k-1}. We now divide the maximization in two parts depending on the choice of $z$:
\begin{itemize}
\item If $z=\sqrt{k}e_1$, then $\dist^2(-\frac{1}{\sqrt{k}}z_{1:t}; \simplex_t) = 1+1/\max\{t-1,1/3\}$ leading to a corresponding value for objective in~\eqref{eq:varphi-k-1-dumm} of $k(1+1/\max\{t-1,1/3\})$.

\item If $z=\frac{1}{\sqrt{k}}\one_p$, then $\frac{1}{\sqrt{k}}z_{1:t} = \frac{1}{k}\one_t$. Since $z_{1:t}$ is a multiple of $\one_t$, it is easy to see that $\dist^2(-\frac{1}{\sqrt{k}}z_{1:t}; \simplex_t) = t(\frac{1}{t}+\frac{1}{k})^2$ leading to a corresponding value for objective in~\eqref{eq:varphi-k-1-dumm} of 
\[
kt(\frac{1}{t}+\frac{1}{k})^2 + \frac{p-t}{k} = 2+\frac{k}{t}+\frac{p}{k}.
\]

\end{itemize}
Taking the maximum over the above three cases, we get 
\begin{align}\label{eq:varphi-k-1-lessk}
\varphi^2 = \max\bigl\{ 
k(1+\frac{1}{\max\{t-1,1/3\}}),
2+\frac{k}{t}+\frac{p}{k}
\bigr\}
\end{align}
for when $\norm{\beta}_1 < k \norm{\beta}_\infty$.

\item If $\norm{\beta}_1 = k \norm{\beta}_\infty$, we proceed with upper bounding $\varphi$ using \eqref{eq:varphi-k-1-gtrk} and \eqref{eq:varphi-k-1-lessk}. Observe that in this case, $\partial \norm{\beta}_{k\square 1}$ contains both $\frac{1}{\sqrt{k}}\partial\norm{\beta}_1$ and $\sqrt{k}\partial\norm{\beta}_\infty$. Therefore, for any fixed vector, the distance to $\partial \norm{\beta}_{k\square 1}$ is smaller than the distance to either of the other two subdifferentials. Therefore, 
\begin{align*}
\varphi^2 \leq \min\{\eqref{eq:varphi-k-1-gtrk} , \eqref{eq:varphi-k-1-lessk} \}
\end{align*}
for when $\norm{\beta}_1 = k \norm{\beta}_\infty$. For such condition to hold, it is necessary that $t\leq k \leq k^\star$. 
\end{itemize}
\end{proof}

\begin{lemma}\label{lem:varphi-bnd-norm-k-1-dual}
$\varphi^2(\beta; \nor^\star_{k\square 1}) \leq 4 \min\{1 ,\frac{\norm{\beta}_0}{k}\}$.
\end{lemma}
\begin{proof}[Proof of \autoref{lem:varphi-bnd-norm-k-1-dual}]
Recall the representation of $\nor_{k\square 1}^\star$ as an ordered weighted $\ell_1$ norm in \autoref{lem:all-k-1} (\autoref{lem:all-k-1-normk1-dual}) with $w = \frac{1}{\sqrt{k}}[\one_k^\sT~,~ \zero_{p-k}^\sT]^\sT$. Moreover, from \autoref{lem:varphi-OWL} we have $\varphi^2(\beta; \nor_\owl) \leq 4\norm{w_\cG}_2^2$ where $\cG = \supp(\bar\beta)$. 
These establish the result. 
\end{proof}

\clearpage 

\bibliography{JJF19}
\bibliographystyle{plainnat}

\end{document}